\documentclass{article}

\usepackage{microtype}
\usepackage{graphicx}
\usepackage{subfigure}
\usepackage{booktabs} 

\usepackage{hyperref}

\usepackage[accepted]{icml2024}

\usepackage{amsmath}
\usepackage{amssymb}
\usepackage{mathtools}
\usepackage{amsthm}
\usepackage{enumitem}

\usepackage[capitalize,noabbrev]{cleveref}

\usepackage{multirow}
\usepackage{subcaption}

\theoremstyle{plain}
\newtheorem{theorem}{Theorem}[section]
\newtheorem{proposition}[theorem]{Proposition}
\newtheorem{lemma}[theorem]{Lemma}
\newtheorem{corollary}[theorem]{Corollary}
\theoremstyle{definition}

\theoremstyle{remark}

\usepackage[textsize=tiny]{todonotes}

\newcommand{\ie}{i.e.}
\newcommand{\eg}{e.g.}
\newcommand{\method}{Decoupled Hypershperical Energy Loss}
\newcommand{\acronym}{DHEL}

\usepackage{amsthm}
\usepackage{xr}
\makeatletter
\newcommand*{\addFileDependency}[1]{
  \typeout{(#1)}
  \@addtofilelist{#1}
  \IfFileExists{#1}{}{\typeout{No file #1.}}
}
\makeatother

\newcommand{\Linear}{K^{\text{lin}}}
\newcommand{\linear}{\kappa^{\text{lin}}}
\newcommand{\Gauss}{K^{\text{gauss}}}
\newcommand{\gauss}{\kappa^{\text{gauss}}}
\newcommand{\Riesz}{K^{\text{riesz}}}
\newcommand{\riesz}{\kappa^{\text{riesz}}}
\newcommand{\Logar}{K^{\text{log}}}
\newcommand{\logar}{\kappa^{\text{log}}}

\newcommand{\Lgentotal}{L_{\textnormal{CL-sym}}}
\newcommand{\Lgen}{L_{\textnormal{CL}}}
\newcommand{\Lgeninfoncea}{L_{\textnormal{a}}}
\newcommand{\Lgeninfonceb}{L_{\textnormal{c}}}
\newcommand{\Lgensimclr}{L_{\textnormal{b}}}

\newcommand{\Linfoncea}{L_{\textnormal{InfoNCE}}}

\newcommand{\Lsimclr}{L_{\textnormal{SimCLR}}}
\newcommand{\Ldcl}{L_{\textnormal{DCL}}}
\newcommand{\Lours}{L_{\textnormal{DHEL}}}

\newcommand{\Lgenkertotal}{L_{\textnormal{KCL-sym}}}
\newcommand{\Lgenker}{L_{\textnormal{KCL}}}

\newcommand{\funsimple}{f}
\newcommand{\fun}{f_{\boldsymbol{\theta}}}
\newcommand{\funstar}{f_{\boldsymbol{\theta}^*}}
\newcommand{\ppos}{p_{+}}
\newcommand{\pdata}{p}
\newcommand{\pushfun}{f_{\boldsymbol{\theta\#}}}
\newcommand{\psim}{p_{\text{sim}}}
\newcommand{\qsim}{q_{\text{sim}}}

\DeclareMathOperator{\E}{\mathbb{E}}

\DeclareMathOperator{\argmin}{argmin} 
\DeclareMathOperator{\argmax}{argmax}

\usepackage{mathtools}

\renewcommand{\b}[1]{\mathbf{#1}}

\newcommand{\bU}{\b{U}}
\newcommand{\bV}{\b{V}}
\newcommand{\bX}{\b{X}}
\newcommand{\bY}{\b{Y}}

\newcommand{\bp}{\b{p}}

\newcommand{\bu}{\b{u}}
\newcommand{\bv}{\b{v}}
\newcommand{\bx}{\b{x}}
\newcommand{\by}{\b{y}}

\newcommand{\cD}{\mathcal{D}}

\newcommand{\cS}{\mathcal{S}}
\newcommand{\cT}{\mathcal{T}}

\newcommand{\cX}{\mathcal{X}}

\newcommand{\cZ}{\mathcal{Z}}

\usepackage{mathtools}

\icmltitlerunning{
Bridging Mini-Batch and Asymptotic Analysis in Contrastive Learning: From InfoNCE to Kernel-Based Losses}

\begin{document}

\twocolumn[
\icmltitle{Bridging Mini-Batch and Asymptotic Analysis in Contrastive Learning:\\ From InfoNCE to Kernel-Based Losses}



\icmlsetsymbol{equal}{*}

\begin{icmlauthorlist}
\icmlauthor{Panagiotis Koromilas}{equal,di}
\icmlauthor{Giorgos Bouritsas}{equal,di,arch}
\icmlauthor{Theodoros Giannakopoulos}{ncsrd}
\icmlauthor{Mihalis A. Nicolaou}{cyi}
\icmlauthor{Yannis Panagakis}{di,arch}
\end{icmlauthorlist}

\icmlaffiliation{di}{Department of Informatics and Telecommunications, National and Kapodistrian University of Athens}
\icmlaffiliation{arch}{Archimedes AI/Athena Research Center}
\icmlaffiliation{cyi}{The Cyprus Institute}
\icmlaffiliation{ncsrd}{NCSR "Demokritos"}
\icmlcorrespondingauthor{Panagiotis Koromilas}{pakoromilas@di.uoa.gr}

\icmlkeywords{Self-supervised learning, contrastive learning, hyperspherical energy minimisation, kernel contrastive learning}

\vskip 0.3in
]
\printAffiliationsAndNotice{\icmlEqualContribution}




\begin{abstract}
What do different contrastive learning (CL) losses actually optimize for? Although multiple CL methods have demonstrated remarkable representation learning capabilities, the differences in their inner workings remain largely opaque. In this work, we analyse several CL families and prove that, under certain conditions, they admit the same minimisers when optimizing either their batch-level objectives or their expectations asymptotically. In both cases, an intimate connection with the hyperspherical energy minimisation (HEM) problem resurfaces. Drawing inspiration from this, we introduce a novel CL objective, coined Decoupled Hyperspherical Energy Loss (DHEL). DHEL simplifies the problem by decoupling the target hyperspherical energy from the alignment of positive examples while preserving the same theoretical guarantees. Going one step further, we show the same results hold for another relevant CL family, namely kernel contrastive learning (KCL), with the additional advantage of the expected loss being independent of batch size, thus identifying the minimisers in the non-asymptotic regime. Empirical results demonstrate improved downstream performance and robustness across combinations of different batch sizes and hyperparameters and reduced dimensionality collapse, on several computer vision datasets.\\
Code: \href{https://github.com/pakoromilas/DHEL-KCL.git}{github.com/pakoromilas/DHEL-KCL.git}
\end{abstract}

\section{Introduction}

\looseness-1 Contrastive learning has revolutionised self-supervised learning of representations from unlabelled data.
Nonetheless, optimising contrastive losses exhibits significant challenges in practice. These include the need for \textit{large batches of negative samples} leading to memory issues \cite{He0WXG20, TianKI20, chen2020simple}, \textit{high sensitivity to the temperature hyperparameter} affecting model performance \cite{wang2021understanding, zhang2021does}, the propensity for \textit{dimensionality collapse} in learned representations \cite{hua2021feature, JingVLT22}, and a reliance on  \textit{sophisticated hard-negative sampling strategies} \cite{RobinsonCSJ21}.

Although there are approaches to understand and address some of the challenges above in isolation, theoretical analyses of CL often use loss functions and assumptions that diverge from those effective in practice (e.g., SimCLR) or depend on conditions often unrealistic in real-world settings, \eg\ infinite batch sizes, conditional independence, or simplified network architectures \cite{SaunshiPAKK19, wang2020understanding, JingVLT22, balestriero2022contrastive, ji2023power}.

This work poses a first step towards bridging the gap between different variants of the classical InfoNCE loss \cite{oord2018representation}. That is, we examine their optimal solutions within two regimes: \textit{the finite regime} concerning losses evaluated on a sampled mini-batch, and \textit{the asymptotic regime of their expectation} (i.e. in the limit of infinite batch size). In the finite case, under a batch size condition, we show multiple InfoNCE variants share the same unique optimal solution attained when (i) positive pairs align perfectly and (ii) representations form a regular simplex inscribed in the sphere, with all pairwise distances equal. Additionally, we show that they have the same asymptotic behaviour, and in turn, the same minimisers: those identified in \cite{wang2020understanding}, i.e. (i) perfect alignment and (ii) uniform distribution on the unit sphere. Interestingly, in both cases, outcome (ii) coincides with the notion of  \textit{minimal hyperspherical energy}.

However, despite commonalities in optima, many variants exhibit notable performance discrepancies. This suggests that optima are difficult to attain in practice. To facilitate optimisation, we introduce a new variant coined as \method\ that fully decouples the two terms reflecting desired properties to optimise. Specifically, we propose simply replacing the classical InfoNCE denominator—see Table \ref{tab:cl_variants}, left—with a denominator involving only negative samples, eliminating dependence on positive counterparts per Eq. \eqref{eq:dhel_loss}. \textit{The resulting alignment and uniformity terms are independent and can in principle be optimised separately}, contrary to existing variants where it is unclear if possible since the terms are coupled.

Moving one step further in the direction of Hyperspherical Energy Minimisation (HEM) in CL, we examine the optima of another family of CL losses, \ie\ Kernel Contrastive Learning (KCL). Kernels first appeared in the CL literature in \cite{li2021self}, where the authors 
introduce a new CL objective based on 
kernel dependence maximisation and
establish a connection with InfoNCE minimisation. In this work, we investigate a general family similar to the one of \cite{li2021self}, and discover that under certain conditions, \textit{mini-batch KCL loss, as well as its expectation, have the same optima as all the analysed InfoNCE variants}. Importantly, KCL enjoys several interesting properties: (1) the expected loss is \textit{independent of the number of negative samples}, 
and (2) \textit{we can identify its minima non-asymptotically}.

We conducted empirical tests on \acronym\ and KCL using different kernel functions meeting necessary conditions. Results show both methods (i) \textit{maintain superior performance across various and small batch sizes}, (ii) \textit{are robust to temperature hyperparameter changes}, and (iii) \textit{utilise more dimensions effectively, addressing th}e dimensionality collapse issue.
 
\looseness-1Our contributions can be summarized as follows:
\begin{itemize}[noitemsep, topsep=0pt]
    \item We prove that different general CL loss families share the same unique optimal solution 
    in the single mini-batch regime when the batch size is no larger than the ambient dimension + 1, as well as in the asymptotic expected case.
    \item We introduce a novel CL loss family that decouples positive from negative samples in the uniformity term, preserves the desired properties and achieves considerable empirical improvements across various metrics.
    \item We establish a connection between Kernel Contrastive Learning and Hyperspherical Energy Minimisation, highlight its theoretical advantages and empirically validate that KCL can be used in place of InfoNCE variants. 
\end{itemize}

\section{Related Work}

\noindent{\textbf{Contrastive Learning.} }Contrastive learning was formally introduced by \cite{chopra2005learning} and was later generalised to the (N+1) tuple loss \cite{sohn2016improved} before the popular InfoNCE loss was introduced in contrastive predictive coding \cite{oord2018representation}. InfoNCE combined with a range of engineering tricks (sampled augmentations, large batch sizes, etc) is the workhorse of modern CL methods \cite{chen2020simple, dwibedi2021little, yeh2022decoupled}. 

However, a range of limitations have been identified.
Downstream performance is sensitive to the temperature hyperparameter, necessitating extensive tuning \cite{wang2021understanding, zhang2021does}. Empirical evidence shows that performance improves with an increased number of negative samples, leading to the requirement for large batch sizes and the incorporation of hard-negative sampling \cite{chen2020simple, tian2020contrastive, he2020momentum, RobinsonCSJ21}. Additionally, there is a tendency for learned representations to use only a fraction of dimensions, not fully exploiting the capacity of the representation space \cite{hua2021feature, JingVLT22}. 

\noindent\textbf{Kernels in CL.} Kernels have been used in a CL for different purposes, including incorporating prior knowledge, conditional sampling of positives and analysing
the induced representation space kernels when optimising SLL objectives 
\cite{pmlr-v202-dufumier23a, kiani2022joint, TsaiLM00MS22, Johnson2023Contrastive, waida2023towards}. Most relevant to our work is the loss of \cite{li2021self} (a regularised version of the Hilbert-Schmidt Independence Criterion - HSIC \cite{gretton2005measuring} - which reducess to a two-term KCL loss under certain conditions) which motivated the theoretical study of the CL generalisation error on downstream tasks through the lens of kernels \cite{waida2023towards}. 

\looseness-1\noindent{\textbf{Optima of CL Objectives}.} It is well known that the CL objective is asymptotically minimised for encoders that produce perfectly aligned and uniformly distributed representations \cite{wang2020understanding}. This is in line with continuous HEM which is also achieved by the uniform distribution \cite{liu2022generalizing}.  \citet{sreenivasan2023minibatch}  showed that in the mini-batch regime, the optimal solution of InfoNCE is achieved when positive representations are perfectly aligned and negatives are placed on
a regular simplex (equivalent to an equiangular tight frame - ETF \cite{benedetto2003finite}), which connects the solution to discrete HEM and is a special case of our Theorem \ref{thrm:mb_main_theorem}. \citet{graf2021dissecting} show that the Supervised CL loss is also minimised when each class embeddings collapse to the vertices of an ETF. Projections, a concept that is included in the contrastive learning pipeline, is also shown to help better minimise the energy \cite{lin2020regularizing}.

\looseness-1\noindent{\textbf{Neural Collapse \& Hyperspherical Energy Minimisation.}}
Neural Collapse, where intra-class embeddings have zero variability and class means align with classifier weights in a simplex ETF during overtraining, was first identified in \cite{papyan2020prevalence} and explored under various training conditions (e.g., MSE, Cross Entropy, data imbalance) in \cite{han2021neural, thrampoulidis2022imbalance, zhu2021geometric, LU2022224, zhou2022optimization}. \citet{LiuYWS23} generalised the notion of Neural Collapse by showing that the class means converge to the uniform distribution in the asymptotic case, thus further enhancing the connection between optimization of supervised deep learning methods and the energy minimisation problem. Moreover, energy minimisation has been effective in NN regularisation, promoting neuron diversity on a hypersphere to avoid correlated neurons \cite{liu2018learning, liu2021learning, lin2020regularizing}.

The above are strikingly similar to the minima of mini-batch CL objectives with the unifying umbrella being the minimisation of hyperspherical energy, a problem deeply studied (see \cite{borodachov2019discrete}).
Notably, solutions have been identified for specific scenarios, in the discrete context, such as when the number of points $M$ is less than or equal to the ambient dimension $d$+1, and for $M = 2d$, as well as continuous \cite{liu2022generalizing}. We will discuss these cases several times throughout the paper.

\section{Preliminaries and Notation}\label{sec:prelims}

\paragraph{Contrastive Learning setup.}

 \textit{Self-supervised contrastive learning (SSCL)} is a paradigm aiming to learn data representations without having access to labels, but based solely on prior knowledge about similarities between inputs, or more strictly speaking, about \textit{downstream task invariances}. 
 
 Formally, let $\cX$ be a (measurable) space, \ie\ the \textit{input space} where our data reside and another (measurable) space $\cZ$, the \textit{embedding space}.
Let $f_{\boldsymbol{\theta}}:\cX \to \cZ$ be an encoder (\eg\ a neural network) parametrised by a set of parameters ${\boldsymbol{\theta}} \in {\Theta}$ mapping datapoints to representations. In our setup $\cZ = \mathbb{S}^{d-1} = \{|\bu| \in \mathbb{R}^d \mid \|\bu\| = 1\}$ the unit sphere. We will be using the symbols $\bx, \by$ for input datapoints and $\bu, \bv$ for representations. 

Additionally, denote the (unknown) underlying data distribution with $\pdata$ (on $\cX$). Further, consider a distribution of \textit{positive pairs} with $\ppos$ (on $\cX \times \cX$ and marginals equal to $\pdata$), which incorporates all the data symmetries, \ie\ its support are all pairs of data that are considered equivalent w.r.t. downstream tasks. We will denote the pushforward measures induced by $f$ with $f_\#\pdata$ and (with slight abuse of notation) $f_\#\ppos$ where $f$ is applied element-wise to $\bx$ and $\by$.

Denote with $\bX = [\bx_1; \dots; \bx_M] \in \cX^M $ a collection of $M$ input datapoints and with $\bU = \left[\bu_1; \dots; \bu_M\right] \in \mathbb{R}^{d \times M}$, 
a collection of $M$ representations. We will be also using the following shorthand $f_{\boldsymbol{\theta}}(\bX) = [f_{\boldsymbol{\theta}}(\bx_1); \dots; f_{\boldsymbol{\theta}}(\bx_M)] = \bU$. Also, when we sample $(\bX, \bY)\sim \ppos^M$ we will occasionally write $\hat{\bX}\sim \ppos^M$ instead, with $\hat{\bx}_i = \bx_i$ iff $i \in \{1,\dots,M\}$ and $\hat{\bx}_i = \by_{i-M}$ iff $i \in \{M+1, 2M\}$ (smilarly for $(\bU, \bV)\sim f_\#\ppos$ and $\hat{\bU}$). 
In SSCL, the encoder is trained by optimising an objective that encourages the representations of positive pairs to be close in $\cZ$ and those of negatives to be further. 
In practice, this is performed by iteratively obtaining a sample $(\bX, \bY)$ of $M$ positives from $\ppos$ and computing a \textit{mini-batch loss} denoted with $\Lgen(f_{\boldsymbol{\theta}}(\bX), f_{\boldsymbol{\theta}}(\bY))$, the gradients of which are used to update the parameters $\boldsymbol{\theta}$. This common process can be perceived as aiming to optimise an \textit{expected loss} $\underset{(\bX, \bY)\overset{\text{i.i.d}}{\sim} p^M_{\text{pos}}}{\E}\left[\Lgen(f_{\boldsymbol{\theta}}(\bX), f_{\boldsymbol{\theta}}(\bY))\right]$ with gradient descent by estimating the gradients with Monte Carlo (in this case a single sample is used). For reasons that will become clear later, our interest will revolve around these two viewpoints of the loss, along with a third one, \ie\ the \textit{asymptotic expected loss} $\underset{M \to \infty}{\lim}\underset{(\bX, \bY)\overset{\text{i.i.d}}{\sim} p^M_{\text{pos}}}{\E}\left[\Lgen(f_{\boldsymbol{\theta}}(\bX), f_{\boldsymbol{\theta}}(\bY))\right]$.

\section{Reconciling Contrastive Loss Variants} 

\noindent\textbf{Mini-batch optimisation.} We will start our investigation of the minima of different CL variants from mini-batch losses, whose gradients are used to update the parameters of the encoder. 
We initially focus on the following two formulas of single sample (mini-batch) contrastive losses:

\begin{equation}\label{eq:general_losses}
    \begin{split}
      \Lgeninfoncea(\bU, \bV; \phi, \psi) &= \frac{1}{M} \sum_{i=1}^M \psi \left(\sum_{j=1, j \neq i}^M \phi\left(\left(\bv_j-\bv_i\right)^{\top} \bu_i\right)\right),\\
     \Lgensimclr(\bU, \bV; \phi, \psi) &= \frac{1}{M} \sum_{i=1}^M \psi \left(\sum_{j=1, \atop j \neq i, i+M}^{2M} \phi\left(\left(\hat{\bu}_j-\bv_j\right)^{\top} \bu_i\right)\right), 
    \end{split}
\end{equation}
 where $\psi, \phi : \mathbb{R} \to \mathbb{R}$. These generalise many practical variants such as the original InfoNCE \cite{gutmann2010noise, oord2018representation, wu2018unsupervised, sohn2016improved, chen2020simple} ($\Lgeninfoncea$), SimCLR (or NT-Xent loss) \cite{chen2020simple} and DCL \cite{yeh2022decoupled} ($\Lgensimclr$). The two variants differ in the datapoints that are used to compute the denominator that normalises the similarity between the pair of positive datapoints ($\bu_i, \bv_i$); the former considers half of the datapoints in the batch, while the latter all of them, except $\bu_i$ itself. Its positive counterpart $\bv_i$ may or may not be considered. The exact formulas for each particular method can be found in Table \ref{tab:cl_variants} and Appendix \ref{sec:mb_minima}, Eq. \eqref{eq:special_losses}.

Frequently, a symmetric version of the losses in Eq. \eqref{eq:general_losses} is used, defined as 
${\Lgentotal(\bU, \bV) = \frac{1}{2} \left(\Lgen(\bU, \bV) + \Lgen(\bV, \bU) \right)}$, where we omitted $\phi$ and $\psi$ for brevity. In a very recent work \citet{sreenivasan2023minibatch} studied the optima of InfoNCE for the $M\leq d+1$ case. Here, we generalise their results for all the losses of Eq. \eqref{eq:general_losses} - proof in Appendix \ref{sec:mb_minima}. Formally:

\begin{theorem}\label{thrm:mb_main_theorem}
Consider the following optimisation problem:
\begin{equation}\label{eq: mb_optimisation}
    \underset{\bU, \bV \in (\mathbb{S}^{d-1})^M}{\argmin} \Lgentotal(\bU, \bV),
\end{equation}
where $\bU, \bV$ are tuples of $M$ vectors on the unit $d-1$-sphere and $\Lgentotal$ is the symmetric version of any of the loss functions $\Lgeninfoncea(\cdot, \cdot; \phi, \psi),
\Lgensimclr(\cdot, \cdot; \phi, \psi)$
as defined in  Eq. \eqref{eq:general_losses}. Further, suppose the following conditions: (1) $\phi: \mathbb{R} \to \mathbb{R}$ is \textbf{increasing \& convex}, (2) $\psi: \mathbb{R} \to \mathbb{R}$ is \textbf{increasing \& $\tilde{\psi}(x; \alpha) = \psi\left(\alpha \phi\left(x\right)\right)$ is convex for $\alpha>0$} and (3) ${1<M \leq d + 1}$. Then, the  problem of Eq. \eqref{eq: mb_optimisation} obtains its optimal value when $(\bU, \bV) = (\bU^*, \bV^*)$ with:
\begin{equation}\label{eq:mb_minima}
    \bU^* = \bV^* \ \text{ and } \ \bU^*:
    \  \text{regular $M-1$ simplex}.
\end{equation}
Additionally, (4) if $\psi, \phi$ are \textbf{strictly increasing} and $\tilde{\psi}$ is  \textbf{strictly convex} then all the $(\bU^*, \bV^*)$ that satisfy Eq. \eqref{eq:mb_minima} are the \textbf{unique} optima.
\end{theorem}

\begin{corollary}  The mini-batch CL loss functions $\Linfoncea$, 
 $\Lsimclr$, $\Ldcl$ have the \textbf{same unique minima } on the unit sphere when $1<M\leq d+1$, \ie\ all the optimal solutions of Eq. \eqref{eq: mb_optimisation} will satisfy the properties of Eq. \eqref{eq:mb_minima}.
\end{corollary}

Despite the multitude of variations of InfoNCE that have been proposed, Theorem \ref{thrm:mb_main_theorem} asserts that having the same optimal solution is conditioned solely on the monotonicity and convexity of the functions $\phi$ and $\tilde{\psi}$. This result might be counterintuitive given that $\Lgeninfoncea$ and $\Lgensimclr$ allow for different couplings of the representations $\bu_i$ and $\bv_i$. Additionally, it provides a first step towards clarifying the CL landscape and gives a general strategy for designing losses without compromising the optimality of the above solutions.

The discovered minima are themselves typically considered desirable in the representation learning literature \cite{papyan2020prevalence, kothapalli2023neural, wang2020understanding}
since the L.H.S. of Eq. \eqref{eq:mb_minima} implies \textit{perfect alignment}, \ie\ pairs of equivalent points according to $\ppos$ are mapped to the same representation, and the R.H.S. implies \textit{perfect uniformity}, \ie\ maximum spreading of the points in the unit sphere, a property that usually simplifies the downstream function to-be-learned. Finally, it is well known \cite{borodachov2019discrete, liu2022generalizing} that the regular $M-1$ simplex is a (unique minimiser) of the \textit{Hypershperical Energy} for a wide variety of kernels, which
illustrates the connection between mini-batch CL, neural collapse and hyperspherical energy minimisation. Note that such a connection has been previously pinpointed by \citet{wang2020understanding}, but only for the asymptotic behaviour of the expected mini-batch CL loss, as we discuss below.

\begin{table*}
\caption{Comparison of InfoNCE variants.}
\centering
\label{tab:cl_variants}
\resizebox{\textwidth}{!}{%
\begin{tabular}{|ll|ccc|}
\hline
\multicolumn{2}{|l|}{Loss name} &
  \multicolumn{1}{l}{InfoNCE} &
  SimCLR &
  DCL \\ \hline
 MB &  &
  $\frac{1}{M} \sum\limits_{i=1}^M \log \left(1 + {\sum\limits_{j=1 \atop j \neq i }^{M} e^{(\bv_j- \bv_i)^{\top} \bu_i/\tau}}\right)$ &
  $\frac{1}{M} \sum\limits_{i=1}^M \log \left(1 + \sum\limits_{j=1 \atop j\neq i, M+i}^{2M} e^{(\hat{\bu}_j - \bv_i)^{\top} \bu_i/\tau}\right)$ &
  $\frac{1}{M} \sum\limits_{i=1}^M \log \left(\sum\limits_{j=1 \atop j\neq i, M+i}^{2M} e^{(\hat{\bu}_j - \bv_i)^{\top} \bu_i/\tau}\right)$ \\ \hline
\multicolumn{1}{|l|}{EMB} &
  $\underset{(\bu, \bv)\sim \funsimple_\#\ppos}{\E}\left[-\bv^\top \bu\right] +$ &
  $\underset{(\bu, \bv)\sim \funsimple_\#\ppos \atop \bV^{\prime}\overset{\text{i.i.d}}{\sim} \funsimple_\#\pdata^{M-1}}{\E}\left[\log \left(e^{\bv^{\top} \bu} + \sum_{j=1}^{M-1} e^{\bu^{\top} \bv^{\prime}_j}\right)\right]$ &
  $\underset{(\bu, \bv) \sim \funsimple_\#\ppos \atop \hat{\bU}\overset{\text{i.i.d}}{\sim} \funsimple_\#\ppos^{M-1}}{\E}\left[\log \left(e^{\bv^{\top} \bu} + \sum\limits_{j=1}^{2M-2} e^{\hat{\bu}_j^{\top} \bu}\right)\right]$ &
  $\underset{\bu \sim \funsimple_\#\pdata \atop \hat{\bU}\overset{\text{i.i.d}}{\sim} \funsimple_\#\ppos^{M-1}}{\E}\left[\log \left( \sum\limits_{j=1}^{2M-2} e^{\hat{\bu}_j^{\top} \bu}\right)\right]$ \\
  \hline
\multicolumn{1}{|l|}{\multirow{2}{*}{Asymptotic}} &
  \multicolumn{1}{|l|}{Normalising constant} &
  $-\log(M-1)$ &
  $-\log(2M-2)$ &
    $-\log(2M-2)$ \\ \cline{2-5} 
\multicolumn{1}{|l|}{} &
  \multicolumn{1}{|l|}{Limit} &
  \multicolumn{3}{c|}{$\underset{(\bu,\bv)\sim \funsimple_\#\ppos}{\E}\left[\bu^\top \bv\right] + \underset{\bu \sim \funsimple_\#\pdata}{\E}\left[\log \underset{\bu' \sim \funsimple_\#\pdata}{\E}\left[e^{\bu^\top \bu'}\right]\right]$} \\ \hline
\multicolumn{2}{|l|}{argmin MB ($M\leq d+1$)} &
  \multicolumn{3}{c|}{$M-1$ regular simplex} \\ \hline
\multicolumn{2}{|l|}{argmin EMB / Asymptotic} &
  \multicolumn{3}{c|}{Unknown / $U(\mathbb{S}^{d-1})$}
  \\ \hline
\end{tabular}
}
\end{table*}

\begin{table*}
\centering
\caption{Comparison of \acronym\ and KCL variants.}
\label{tab:our_variants}
\resizebox{\textwidth}{!}{
\begin{tabular}{|ll|cc|}
\hline
\multicolumn{2}{|c|}{Loss name} &
  \multicolumn{1}{c}{\acronym} &
  KCL 
  \\ \hline
\multicolumn{2}{|l|}{MB} &
    $\frac{1}{M} \sum\limits_{i=1}^M \log \left( {\sum\limits_{j=1 \atop j \neq i }^{M} e^{(\bu_j- \bv_i)^{\top} \bu_i}}\right)$
  & $ 
- \frac{1}{M} \sum\limits_{i=1}^{M} K_A\left(\bu_i, \bv_i\right)
+ 
\frac{\gamma}{M(M-1)} 
\sum\limits_{i, j=1 \atop j\neq i}^{M} K_U\left(\bu_i, \bu_j\right)$
  \\ \hline
\multicolumn{2}{|l|}{EMB} &
  $\underset{(\bu, \bv)\sim \funsimple_\#\ppos}{\E}\left[-\bv^\top \bu\right] +
    \underset{\bu \sim \funsimple_\#\pdata \atop \bU^{\prime}\overset{\text{i.i.d}}{\sim} \funsimple_\#\pdata^{M-1}}{\E}\left[\log \left(\sum_{j=1}^{M-1} e^{\bu^{\top} \bu^{\prime}_j}\right)\right]$ &
$\underset{(\bu, \bv)\sim \funsimple_\#\ppos}{\E}\left[-K_A\left(\bu, \bv\right)\right] + \gamma
    \underset{\bu \sim \funsimple_\#\pdata \atop \bu^{\prime}{\sim} \funsimple_\#\pdata}{\E}\left[K_U(\bu, \bu^{\prime})\right]$
  \\ \hline
{\multirow{2}{*}{Asymptotic}} &
  \multicolumn{1}{|l|}{Normalising constant} &
  $-\log(M-1)$ &
 $0$ 
  \\ \cline{2-4} 
\multicolumn{1}{|c|}{} &
  \multicolumn{1}{|l|}{Limit} &
$\underset{(\bu, \bv)\sim \funsimple_\#\ppos}{\E}\left[-\bv^\top \bu\right]+ \underset{\bu \sim \funsimple_\#\pdata}{\E}\left[\log \underset{\bu^\prime \sim \funsimple_\#\pdata}{\E}\left[e^{\bu^\top \bu'}\right]\right]$ &
$\underset{(\bu, \bv)\sim \funsimple_\#\ppos}{\E}\left[-K_A\left(\bu, \bv\right)\right] +
    \gamma \underset{\bu \sim \funsimple_\#\pdata \atop \bu^{\prime}{\sim} \funsimple_\#\pdata}{\E}\left[K_U(\bu, \bu^{\prime})\right]$
  \\ \hline
\multicolumn{2}{|l|}{argmin MB ($M\leq d+1$)} &
  \multicolumn{2}{c|}{$M-1$ regular simplex} \\ \hline
\multicolumn{2}{|l|}{argmin EMB / Asymptotic} &
  \multicolumn{1}{c}{Unknown / $U(\mathbb{S}^{d-1})$}
  &
  \multicolumn{1}{c}{$U(\mathbb{S}^{d-1})$ / $U(\mathbb{S}^{d-1})$}
  \\ \hline
\end{tabular}%
}
\end{table*}

\noindent{\textbf{Asymptotic behaviour of the expectation.}}  Theorem \ref{thrm:mb_main_theorem} gives us a qualitative understanding of the similarities among CL variants and of the direction of the gradient at each training iteration, but it does not reveal the bigger picture, \ie\ the objective that we are actually trying to optimise. For example, an obvious limitation of the optimum of Eq. \eqref{eq:mb_minima} is that we can have at most two perfectly aligned points since adding one extra would compromise uniformity. 

To understand the true objective, observe that the gradient of the mini-batch loss is an \textit{unbiased estimate} of the gradient of the \textit{expected loss} ${\E}\left[\Lgen(f_{\boldsymbol{\theta}}(\bX), f_{\boldsymbol{\theta}}(\bY))\right]$ (due to linearity of expectation and gradient) using a single sample. In other words, the expected loss is the \textit{true loss} that we are optimising using gradient estimates. 
It is, therefore, more appropriate to analyse the optima of the latter. However, as we  see in Table \ref{tab:cl_variants}, the expected loss for three common CL losses: \textit{InfoNCE, SimCLR and DCL}, depends on the batch size $M$ even after normalising with an appropriate (missing from the original objective) constant (proof in Lemma \ref{lemma:exp_cl_losses} in the Appendix using simple derivations). Thus, we resort to examining the asymptotic behaviour similarly to \cite{wang2020understanding}. Using similar arguments, it is straightforward to see that the asymptotic behaviour of the above variants is the same (proof in Appendix \ref{sec:expectation_proofs}). Formally:
\begin{proposition}\label{prop:asymptotic}
The expectations of the following batch-level contrastive loss functions: $\Linfoncea(\cdot, \cdot)$, 
 $\Lsimclr(\cdot, \cdot)$, $\Ldcl(\cdot, \cdot)$ have the \textbf{same asymptotic behaviour} when subtracting appropriate normalising constants ($\log(M-1)$ for the first and $\log(2M-2)$ for two latter), \ie\ when $M \to \infty$ they converge to the asymptotic formula of InfoNCE \cite{wang2020understanding}:
\begin{equation}\label{eq:asymptotic}
\underset{(\bu, \bv)\sim f_\#\ppos}{\E}
\left[-\bv^\top \bu/\tau\right] + \underset{\bu{\sim} f_\#\pdata}{\E}\left[\log \underset{ \bu^{\prime}{\sim} f_\#\pdata}{\E}\left[e^{\bu^\top \bu^{\prime}/\tau}\right]\right].
\end{equation}
\end{proposition}
Therefore, the conclusions of Theorem 1 in \cite{wang2020understanding} hold for all three variants; the first term is minimised if there exists $f$ such that all positive pairs are perfectly aligned and the second is minimised if there exists $f$ such that $f_\#\pdata$ is the uniform distribution on the sphere $U(\mathbb{S}^{d-1})$.

\section{\method}
\paragraph{Expected loss: What happens when the batch size is finite?} Closely examining the equations in Table \ref{tab:cl_variants}, we can see that the true objectives are a sum of a common alignment term and a uniformity one that varies. As previously discussed we aim to achieve perfect alignment and perfect uniformity. \textit{The latter does not depend on $\ppos$ (supposing that $\ppos$ is such that perfectly optimising the alignment term does not prohibit the ability to achieve perfect uniformity)}. However, this does not straightforwardly seem to be the case for InfoNCE, SimCLR and DCL, since in all cases we observe a dependence of the uniformity term on $\ppos$ that vanishes only asymptotically. This imposes a coupling between the two terms that can potentially hinder optimisation.

\paragraph{Decoupling uniformity from alignment} Motivated by this observation, we make a simple modification on InfoNCE and propose a new CL objective that allows for an expected uniformity term that is only dependent on $\pdata$:
\begin{equation}\label{eq:dhel_loss}
     \Lours(\bU, \bV)  =\frac{1}{M} \sum_{i=1}^M-\log \left(\frac{e^{{\bu_i^{\top} \bv_i}{/\tau}}}{\sum_{j=1 \atop i \neq j}^M e^{\bu_i^{\top} \bu_j{/\tau}}}\right),
\end{equation}
which is a special case of the generalised \method:

\begin{equation}\label{eq:general_dhel_loss}
      \Lgeninfonceb(\bU, \bV)= \frac{1}{M} \sum_{i=1}^M \psi \left(\sum_{j=1, \atop j \neq i}^M \phi\left(\left(\bu_j-\bv_i\right)^{\top} \bu_i\right)\right)
\end{equation}

\noindent\textbf{Key advantage of DHEL.} The dependence of our uniformity only on $\pdata$ can be also understood intuitively: DHEL is based on the observation that for perfect uniformity, it suffices to contrast a datapoint $\bx_i$ against a \textit{single} positive view of a negative $\bx_j$. Adding more views, as in \cite{chen2020simple, yeh2022decoupled}, does not only seem unnecessary but might also have undesired repercussions since such a uniformity term would aim to uniformly distribute \textit{all} points on the sphere, ignoring that half of them are positives of the other half.\footnote{A formal analysis of the uniformity of the symmetric SimCLR loss, using a lower bound obtained with Jensen's inequality, reveals the hyperspherical energy of a linear kernel, which is minimised when \textit{all} points are uniformly distributed when $1<M\leq d+1$.} Therefore, even though in theory the minima do not seem to be affected, previous InfoNCE variants have two \textit{competing terms}, an issue that we overcome with DHEL.

\looseness-1\paragraph{Theoretical properties of \acronym.} First off, we also analysed \acronym\ w.r.t. its optima in the mini-batch and the asymptotic expectation case. The following theorem shows that under the same conditions, Theorem \ref{thrm:mb_main_theorem} and Proposition \ref{prop:asymptotic} continue to hold (proofs in Appendix \ref{sec:mb_minima},  \ref{sec:expectation_proofs}):
\begin{theorem}\label{thrm:mb_dhel_main_theorem}
Consider the optimisation problem of Eq. \eqref{eq: mb_optimisation} with $\Lgentotal(\cdot, \cdot)$ being the symmetric version of the loss function $\Lgeninfonceb$. Further, suppose that conditions (1) and (2) of Theorem \ref{thrm:mb_main_theorem} hold, \eg\ as for our loss \acronym. Then, when $1<M\leq d+1$, the  mini-batch CL optimisation of Eq. \eqref{eq: mb_optimisation} obtains its optimal value $(\bU^*, \bV^*)$ as in Eq. \eqref{eq:mb_minima}. Moreover, the expectation of
 $\Lours(\cdot, \cdot)$ asymptotically converges to Eq. \eqref{eq:asymptotic} when subtracting a normalising constant equal to $\log(M-1)$. Therefore, the asymptotic expectation of \acronym\ is minimised by any encoder $f$ that is perfectly aligned and distributes representations uniformly on the sphere, \ie\ $f_\#\pdata = U(\mathbb{S}^{d-1})$, if such an encoder exists.
\end{theorem}

\section{Minima of Kernel Contrastive Learning}
As discussed, for all CL variants considered, the uniform distribution on the unit sphere is known to be a minimiser of the true loss only asymptotically. Motivated by this, we seek an alternative loss whose expectation will admit the same minimiser in the non-asymptotic regime. To achieve this, we first observe that the logarithm makes the characterisation of the optima difficult.\footnote{Using Jensen's inequality we can attempt to minimise a lower bound as in Theorem \ref{thrm:mb_main_theorem}, but equality can hold only for DCL and \acronym\ and only if for all $\bu$ and any $M$ negatives $\bu_j'$, the inner products $\bu^\top \bu'_j$ are equal  $\forall j$. If the minimiser of the bound is the uniform distribution this can only happen for $d=2$ \cite{cho2009inner}.} Removing the logarithm and dividing by an appropriate normalisation constant indeed provides us with a batch-level loss whose expectation is independent of the batch size:

\begin{equation}
    \underset{(\bu, \bv)\sim \funsimple_\#\ppos}{\E}\left[-\bv^\top \bu/\tau\right] +
    \underset{\bu \sim \funsimple_\#\pdata \atop \bu^{\prime}{\sim} \funsimple_\#\pdata}{\E}\left[e^{\bu^{\top} \bu^{\prime}_j}/\tau\right].
\end{equation}
But are the desired minima preserved by this objective? To answer this question, we observe that the second term is equivalent to minimising the energy potential of the \textit{gaussian kernel}. This is a well-studied problem \cite{borodachov2019discrete}, discussed also in \cite{wang2020understanding} and is known that the minimiser is once again the uniform distribution on the sphere. Drawing inspiration from this, we examine a more general case, that of \textit{Kernel Contrastive Learning} \cite{li2021self,waida2023towards}; the mini-batch objective is as follows:
\begin{equation}\label{eq:kernels_mb_loss}
\Lgenker(\bU, \bV) = - \frac{\sum\limits_{i=1}^{M} K_A\left(\bu_i, \bv_i\right)}{M} 
+ 
\gamma \frac{\sum\limits_{i, j=1 \atop j\neq i}^{M} K_U\left(\bu_i, \bu_j\right)}{M(M-1)}, 
\end{equation}
where both kernels are of the form ${K(\bx, \by) = \kappa(\|\bx - \by\|^2)}$, with $\kappa: (0,4] \to \mathbb{R}$ and the limit $\underset{x \to 0^+}{\lim}\kappa(x)$ exists and is bounded, and $\gamma>0$ is a weighting coefficient. Using known results for the hyperspherical energy minimisation problem, in the following theorem we provide the conditions that guarantee the preservation of the already discussed minima:

\begin{theorem}\label{thrm:mb_kernels_theorem}
Consider the optimisation problem of Eq. \eqref{eq: mb_optimisation} with $\Lgentotal(\cdot, \cdot)$ being the symmetric version of the loss function $\Lgenker$. Further, suppose the following conditions: (1) the function $k_A$ corresponding to kernel $K_A$ is \textbf{decreasing}, (2) $k_U$, the function corresponding to $K_U$ is \textbf{decreasing and convex} and (3) $1<M \leq d + 1$. Then, the  problem of Eq. \eqref{eq: mb_optimisation} obtains its optimal value when $(\bU, \bV) = (\bU^*, \bV^*)$ as in Eq. \eqref{eq:mb_minima}. Additionally, (4) if $\kappa_A$ is \textbf{strictly decreasing} and $\kappa_U$ is  \textbf{strictly decreasing and strictly convex} then all the $(\bU^*, \bV^*)$ that satisfy Eq. \eqref{eq:mb_minima} are the \textbf{unique} optima. 
\end{theorem}
\looseness-1In appendix \ref{sec:mb_kcl_proof} we extend the above theorem for the case $M=2d$, where using known results from the HEM literature, we show that a minimiser of $\Lgenkertotal$ is the \textit{cross-polytope}.
Moreover, the following proposition states that the expectation is always independent of the batch size, thus we do not have to resort to asymptotic analyses and provides the necessary conditions for the minimiser to be the uniform distribution on the sphere.
\begin{proposition}\label{prop:kernel_expected}
\looseness-1The expectation of the batch-level kernel contrastive loss functions $\Lgenker(\cdot, \cdot)$ is \textbf{independent of the size of the batch}. Therefore, the batch-level loss is an \textbf{unbiased estimator} of the (asymptotic) expected loss:
\begin{equation}\label{eq:kernel_expected_loss}
\underset{(\bu, \bv) \sim f_\#\ppos}{\E} \left[-K_A\left(\bu, \bv\right)\right] + \gamma\underset{\bu,  \sim f_\#\pdata \atop \bu' \sim f_\#\pdata}{\E} \left[K_U\left(\bu, \bu'\right)\right].
\end{equation}
If (1) $\kappa_A$ is (strictly) decreasing and if (2) $\exists \, \boldsymbol{\theta}^*$ such that $\mathbb{P}_{(\bx, \by)\sim\ppos}\left[\fun(\bx) = \fun(\by)\right] = 1$, then the set of $\boldsymbol{\theta}^*$ for which (2) holds are (unique) minimisers of the first term of Eq. \eqref{eq:kernel_expected_loss}. Additionally, if (3) -$\kappa^{\prime}_U$ (first derivative) is \textbf{strictly completely monotone} in $(0,4]$, (4) the expectation defined in the l.h.s. of Eq. \eqref{eq:kernel_expected_loss} is finite and (5) $\exists \, \boldsymbol{\theta}^*$ such that the pushforward measure $f_\#\pdata = U(\mathbb{S}^{d-1})$, then $\boldsymbol{\theta}^*$ is a unique minimiser of the second term of Eq. \eqref{eq:kernel_expected_loss}. Finally, if (6) $\exists \, \boldsymbol{\theta}^*$ such that conditions (2) and (3) can be satisfied simultaneously, then $\boldsymbol{\theta}^*$ is a unique minimiser of Eq. \eqref{eq:kernel_expected_loss}.
\end{proposition}

\looseness-1\noindent\textbf{Remark.}  In \cite{li2021self} it is shown that in certain cases (\ie\ for a discrete distibution) a two-term kernel contrastive loss as the one in Eq. \eqref{eq:kernel_expected_loss} arises as proportional to a kernel dependence measure they aim to maximise (HSIC). 
However, in their paper a different loss is used in practice; first, they use a biased estimator different from Eq. \eqref{eq:kernels_mb_loss} and second they add a regulariser. Additionally, they identify a connection only with they asymptotic version of InfoNCE and they do not study the minima of HSIC as we do in Theorem \ref{thrm:mb_kernels_theorem} and Proposition \ref{prop:kernel_expected}.

\begin{figure*}[t]
    \centering
   \begin{minipage}{0.24\textwidth}
    \resizebox{\textwidth}{!}{\includegraphics{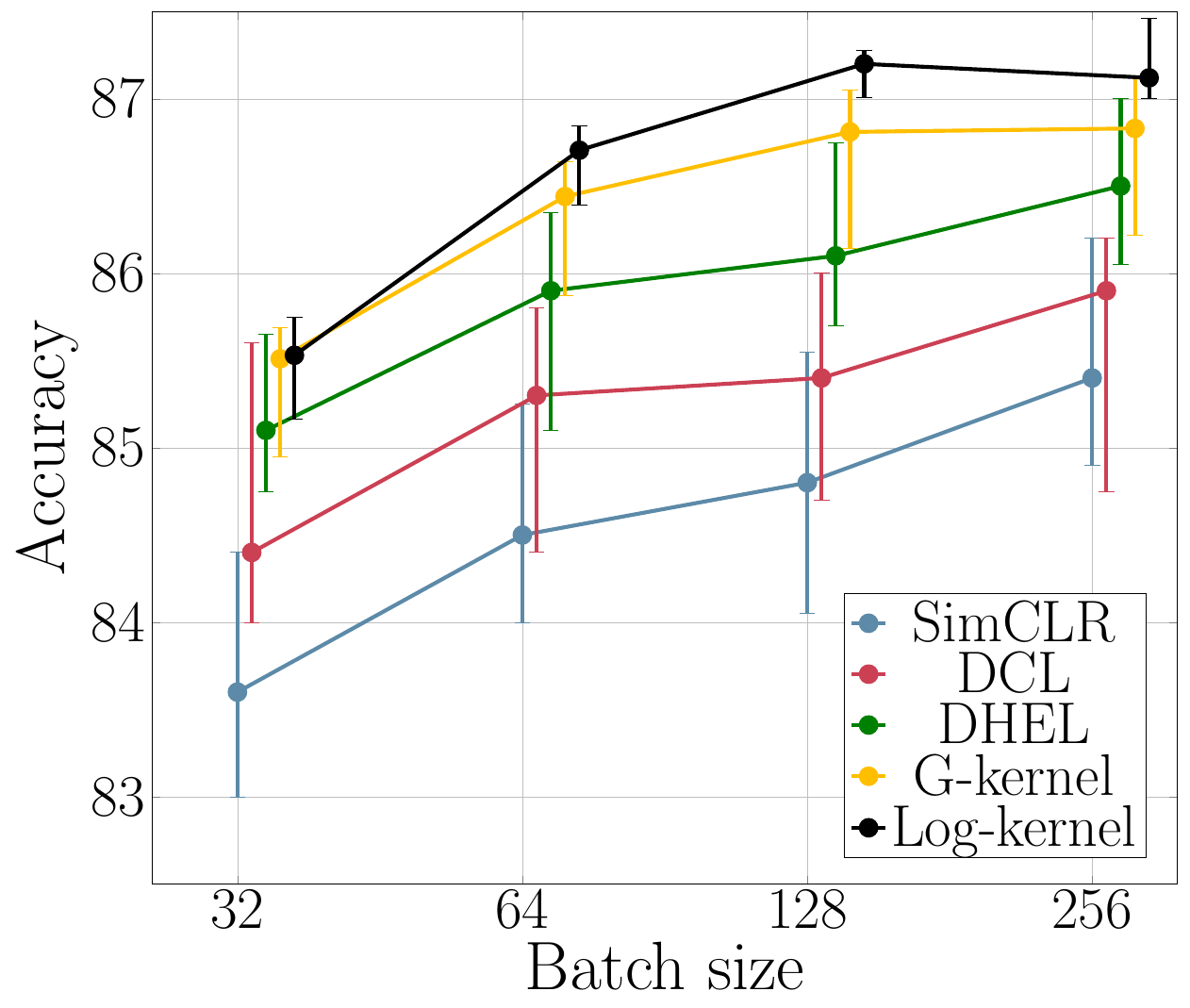}}
    \caption*{(a) CIFAR10}
    \end{minipage}
    \begin{minipage}{0.24\textwidth}
    \resizebox{\textwidth}{!}{\includegraphics{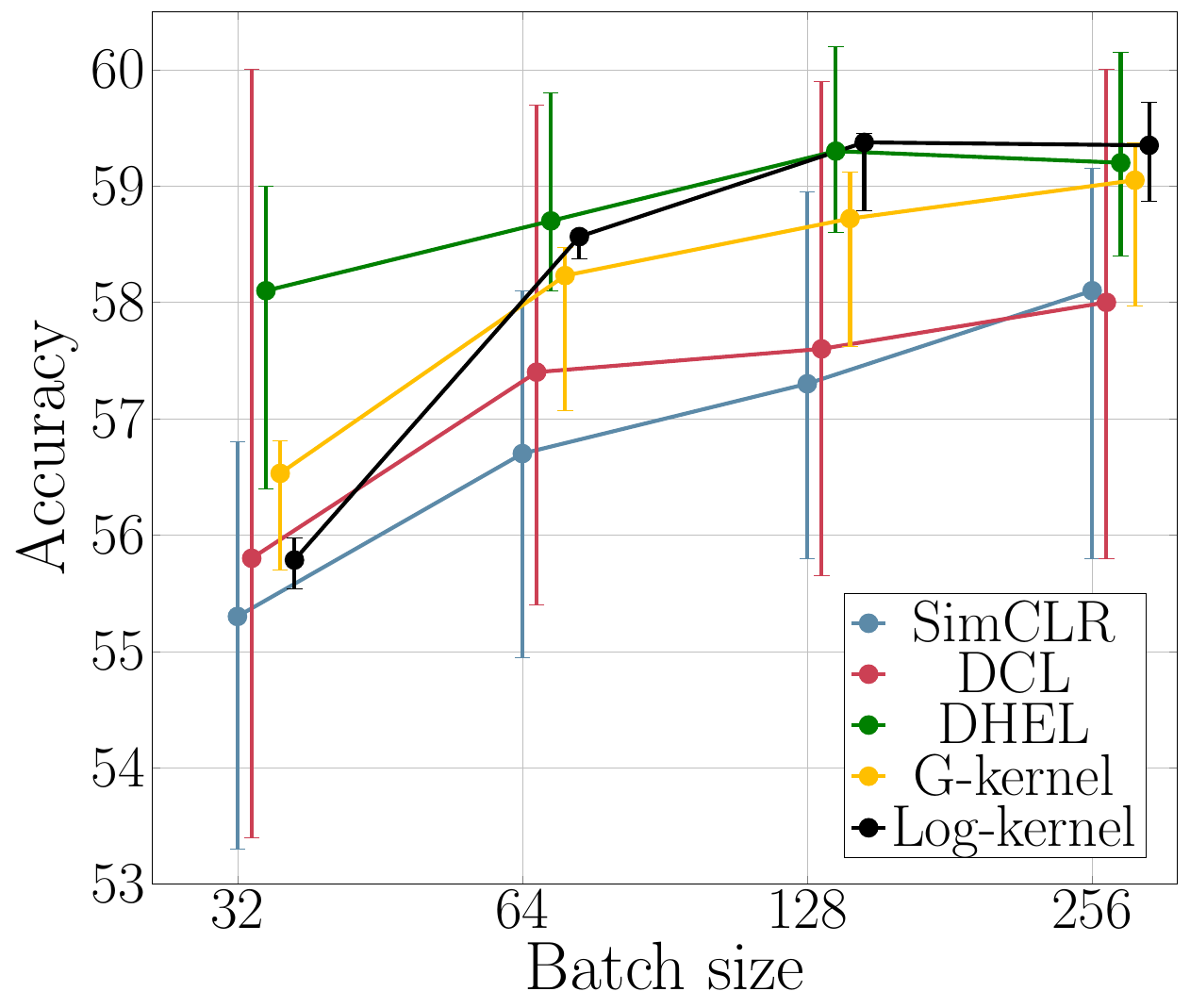}}
    \caption*{(b) CIFAR100}
    \end{minipage}
    \begin{minipage}{0.24\textwidth}
    \resizebox{\textwidth}{!}{\includegraphics{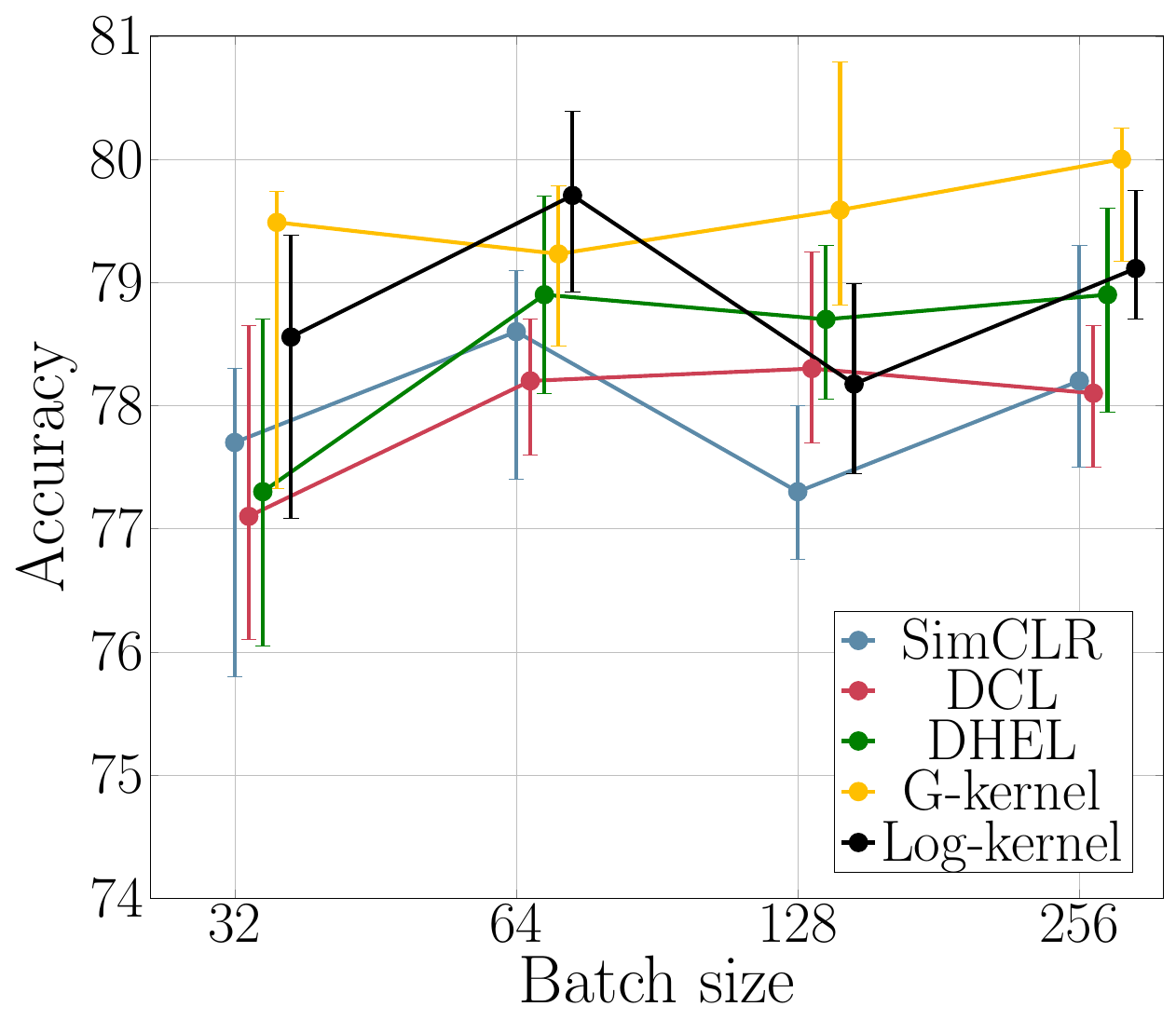}}
    \caption*{(c) STL10}
    \end{minipage}
    \begin{minipage}{0.24\textwidth}
    \resizebox{\textwidth}{!}{\includegraphics{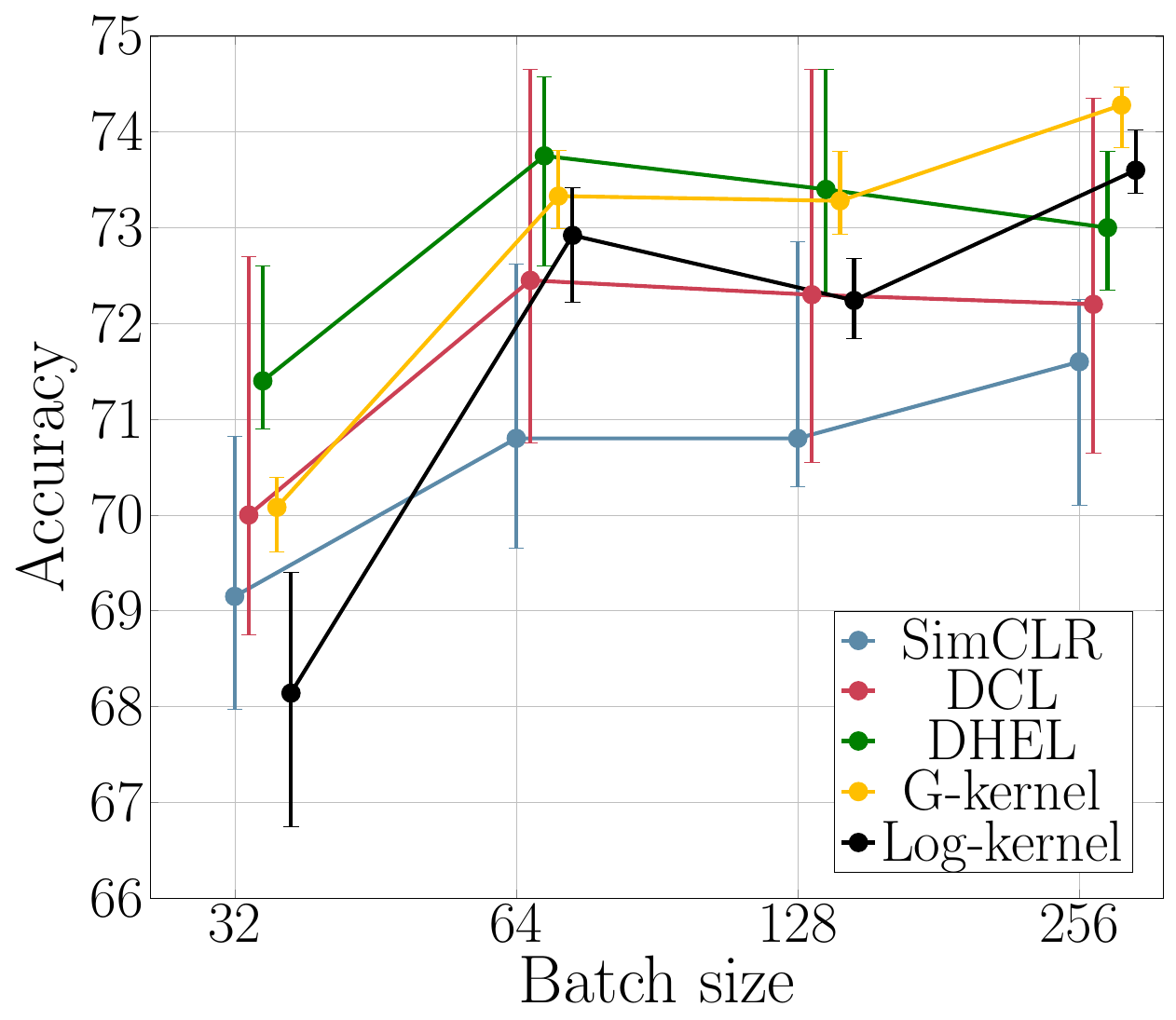}} 
    \caption*{(d) ImageNet-100}
    \end{minipage}
    \caption{Median performance for different batch sizes. Errors against each methods hyperparameters are calculated using the 25\% and 75\% quantiles. DHEL and KCL showcase improved both performance and robust against hyperparameters.}\label{fig:performance}
\end{figure*}

\section{Experimental Evaluation}
In this section, we empirically verify our theoretical results. Our methods are compared to two popular techniques in the literature: (i) \textbf{SimCLR} \cite{chen2020simple} the most used implementation of contrastive pretaining that also demonstrates consistency in terms of performance and (ii) \textbf{DCL} \cite{yeh2022decoupled} the only method in the literature that demonstrates robust performance for various and small batch sizes. We implement (iii) \textbf{DHEL} and (iv) \textbf{two KCL losses for the Gaussian and Logarithmic kernel} (see Appendix \ref{sec:addtional_prelims}).

Following common practices \cite{wang2021unsupervised, yeh2022decoupled, zhang2022dual, wang2020understanding}, we conduct experiments on four popular image classification datasets, namely \textit{CIFAR10, CIFAR100, STL-10, and ImageNet-100}. To illustrate robustness, we validate the performance for a range of each method's hyperparameters and different batch sizes. In addition, to understand the quality of the learned representations, we demonstrate the behaviour of several desired properties. 

We choose ResNet50 as the encoder architecture for the ImageNet-100 dataset and ResNet18 for the other datasets. We train our models for 200 epochs on four batch sizes (32, 64, 128, 256) and optimise them using SGD. Downstream performance is measured using the classical linear evaluation benchmarking technique: we train a linear layer on the learned representations for 200 epochs. Following \cite{wang2021unsupervised}, 11 temperatures (regarding the methods that use temperature as a hyperparameter) are tested, while for kernel methods, along with their hyperparameter, we additionally run experiments for different weighting coefficients $\gamma$. Further details on learning rates, schedulers, augmentations etc. are provided in Appendix \ref{sec:impl_details}.

\begin{figure*}
    \centering
   \begin{minipage}{0.24\textwidth}
    \resizebox{\textwidth}{!}{\includegraphics{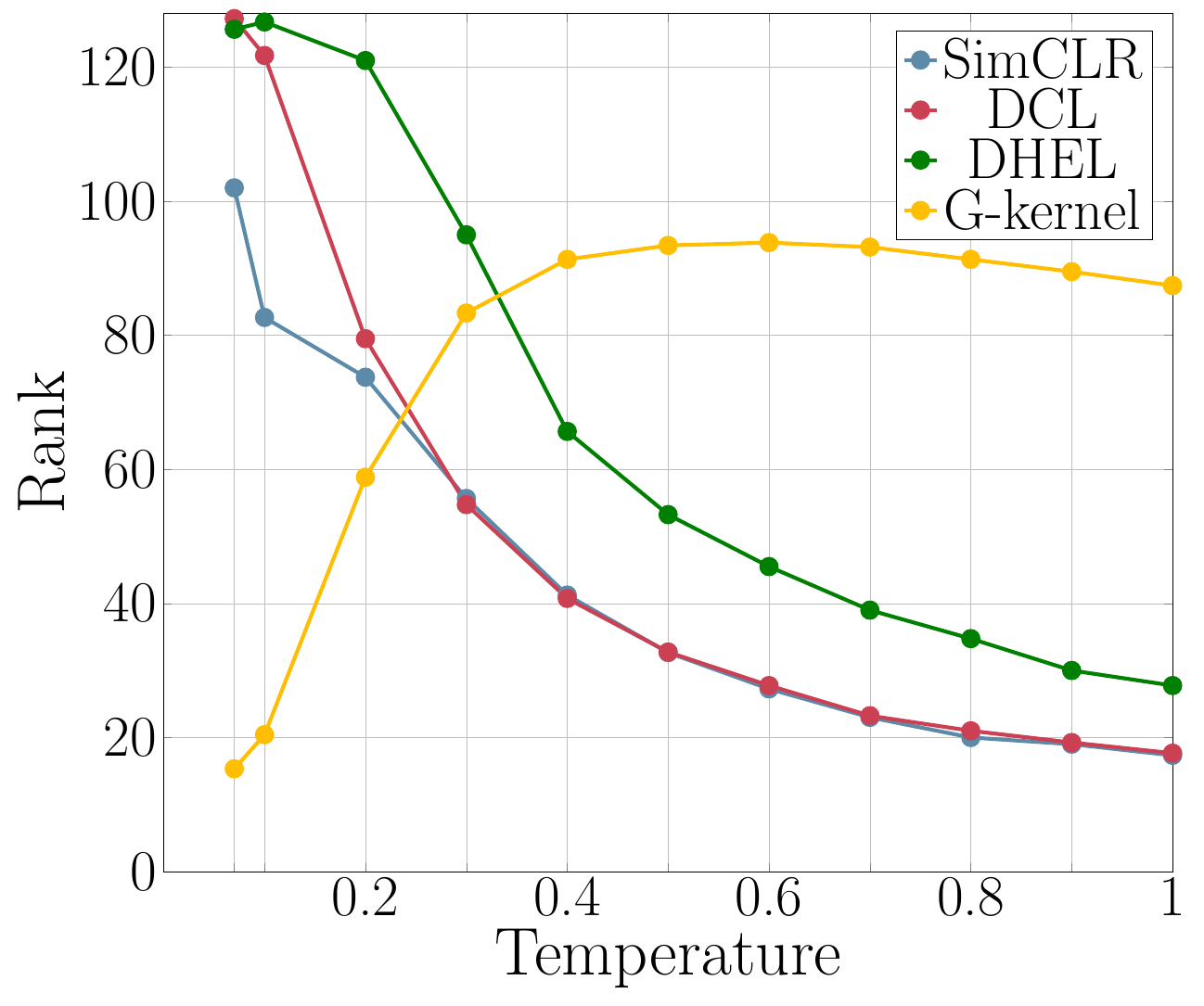}}
    \end{minipage}
    \begin{minipage}{0.24\textwidth}
    \resizebox{\textwidth}{!}{\includegraphics{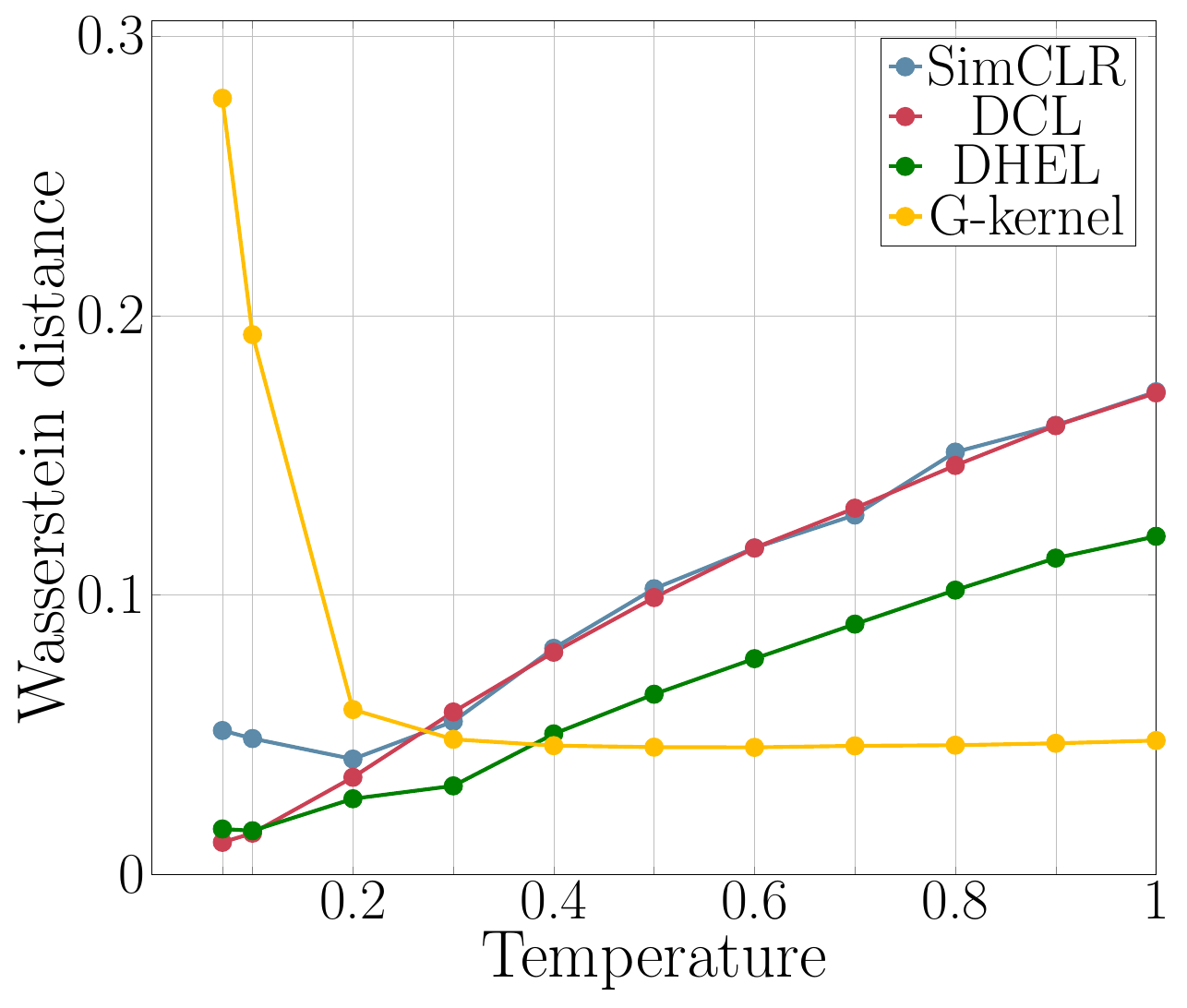}}
    \end{minipage}
    \begin{minipage}{0.24\textwidth}
    \resizebox{\textwidth}{!}{\includegraphics{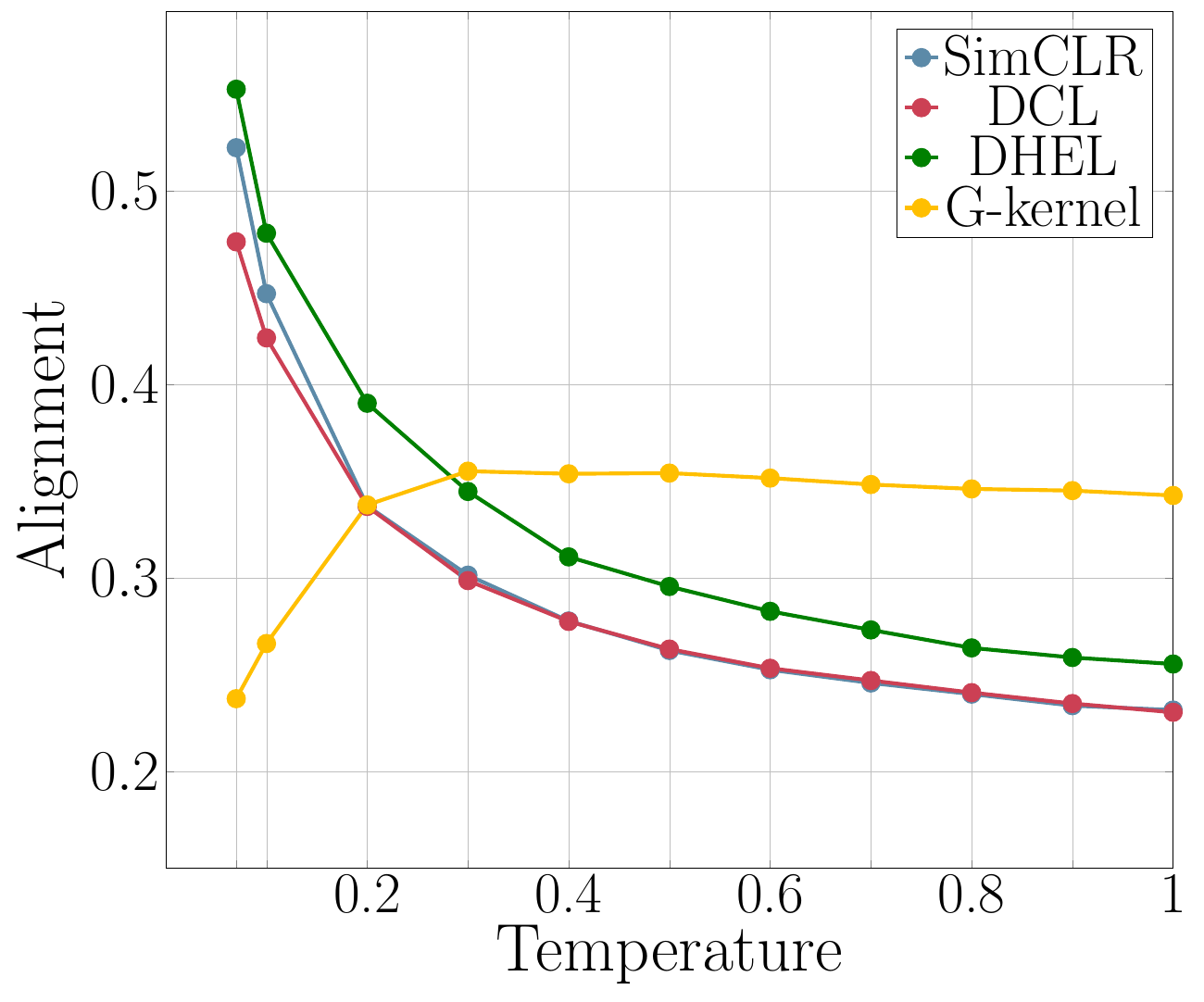}}
    \end{minipage}
    \begin{minipage}{0.24\textwidth}
    \resizebox{\textwidth}{!}{\includegraphics{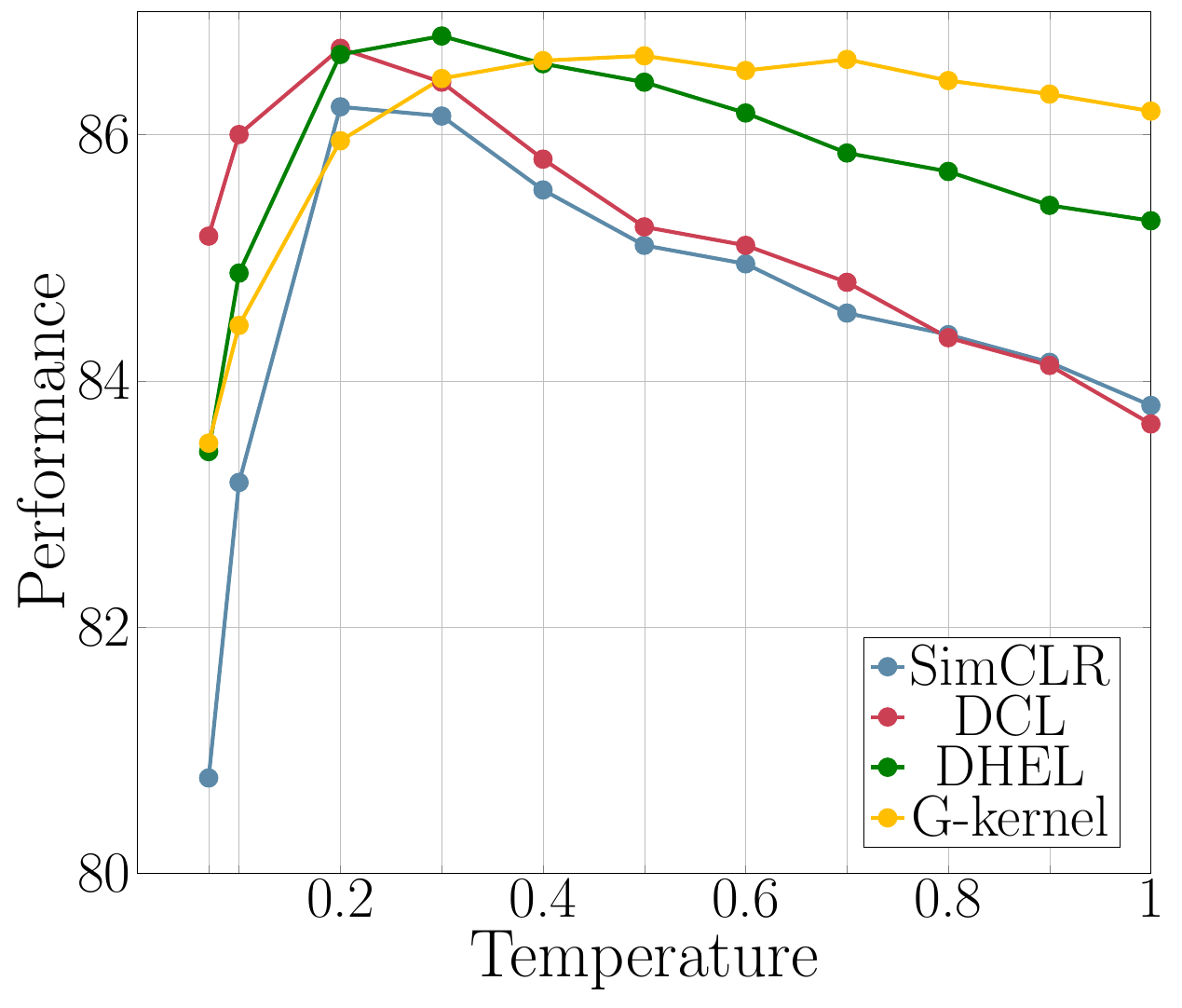}} 
    \end{minipage}
    \begin{minipage}{0.24\textwidth}
    \resizebox{\textwidth}{!}{\includegraphics{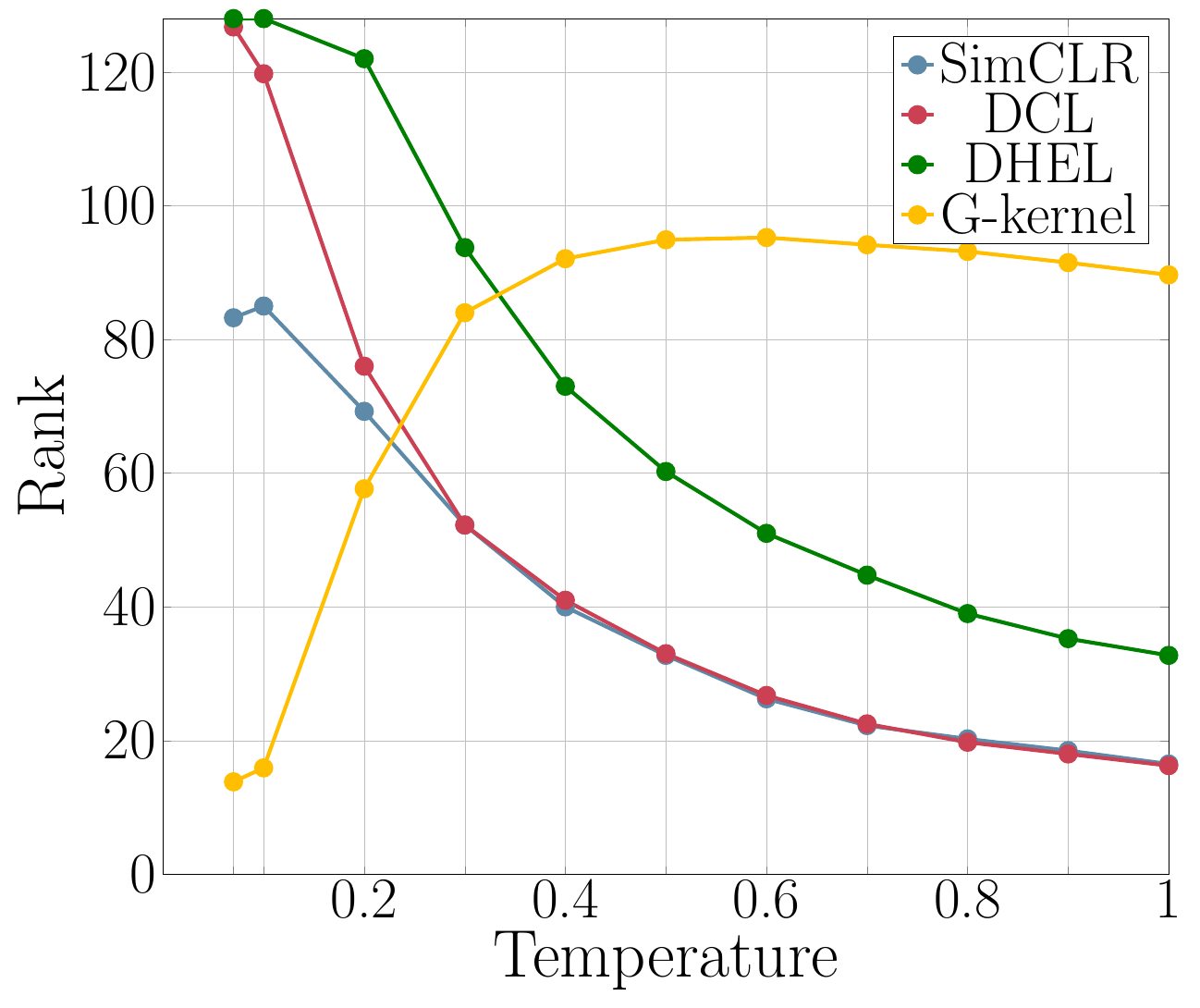}}
    \caption*{(a) Rank}
    \end{minipage}
    \begin{minipage}{0.24\textwidth}
    \resizebox{\textwidth}{!}{\includegraphics{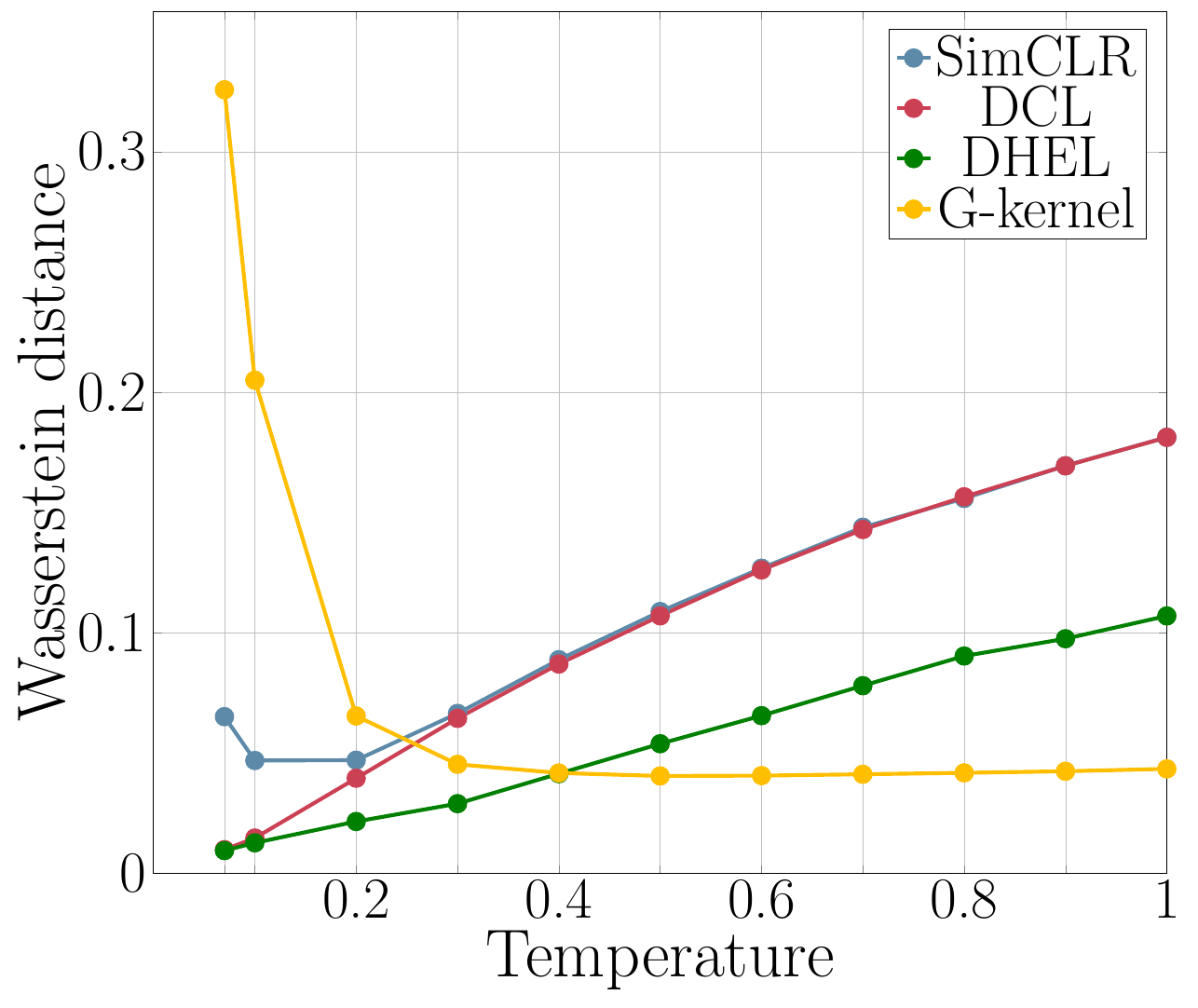}}
    \caption*{(b) Wasserstein distance}
    \end{minipage}
    \begin{minipage}{0.24\textwidth}
    \resizebox{\textwidth}{!}{\includegraphics{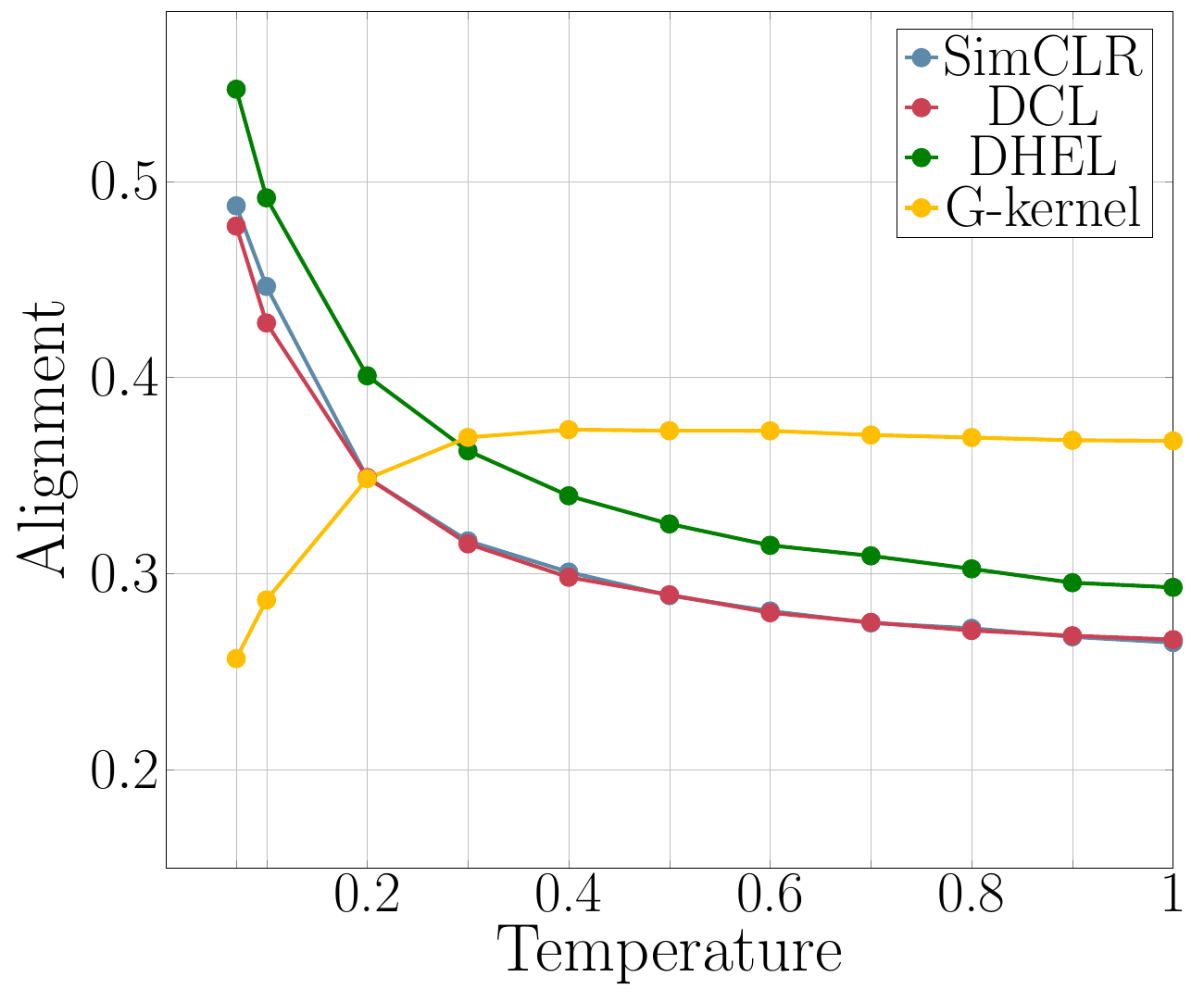}}
    \caption*{(c) Alignment}
    \end{minipage}
    \begin{minipage}{0.24\textwidth}
    \resizebox{\textwidth}{!}{\includegraphics{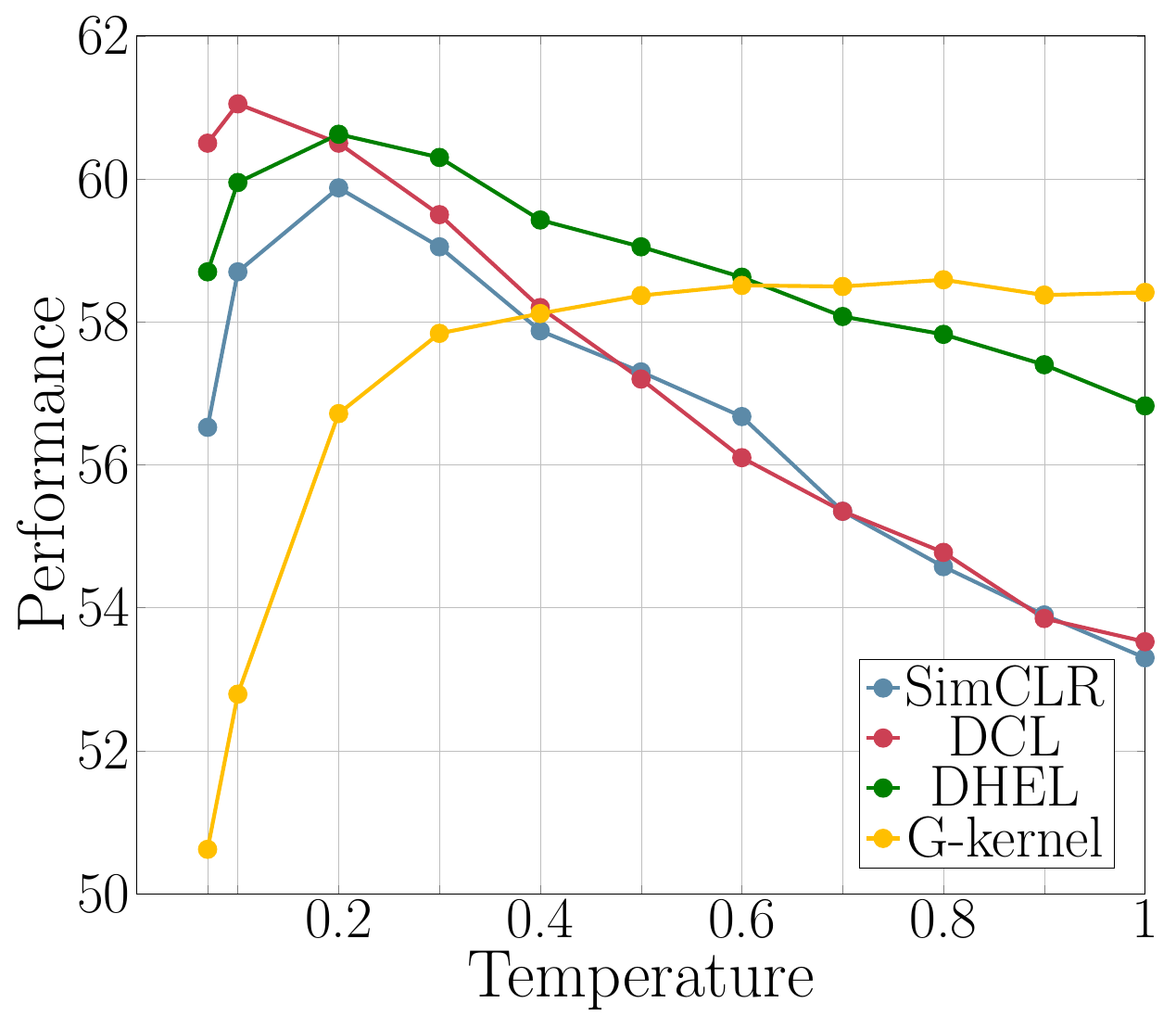}}
    \caption*{(d) Performance}
    \end{minipage}
    \caption{Mean value of properties vs temperature calculated on CIFAR10 (top) \& CIFAR100 (bottom) dataset}\label{fig:properties_temperature}
\end{figure*}

\subsection{Downstream performance and robustness}

The error diagram in \Cref{fig:performance} illustrates both performance and robustness; for each method and batch size, we include the median and 25\% (lower error) and  75\% (upper error) quantiles calculated on the accuracy for different hyperparameters (the median was preferred to the mean due to the presence of a few outliers across all methods). 

\noindent\textbf{Performance.} First off, \textit{\acronym\ significantly outperforms SimCLR across all datasets and batch sizes}, with the upper performance of the latter being smaller than that of the former. Second,\textit{ the median of \acronym\ is always higher than that of DCL}, while their upper performance are comparative (in CIFAR10 and STL10 DHEL's upper performance is always higher than DCL). Additionally, \textit{kernel methods outperform both SimCLR and DCL} competitors, while in several cases, \textit{kernels improve further upon DHEL}.
Overall, \textit{both DHEL and KCL methods showcase significant improvements in median performance, with their upper quantiles being, in most cases, comparable or better than DCL}.

\noindent\textbf{Performance w.r.t. batch size.} 
DHEL and KCL \textit{largely outperform competitors for small batch sizes} in terms of median performance. It is inferred that \textit{our methods enable high downstream performance Contrastive Learning pretraining for a small number of negative samples}. Note that, typically in the literature much larger batch sizes are used, \eg\ \textit{SimCLR needs batch sizes greater than 512 \cite{chen2020simple, yeh2022decoupled}} and \citet{he2020momentum} use batch sizes as large as 64K.

\noindent\textbf{Robustness w.r.t. hyperparameters.}  In addition to the fact that median accuracy of both DHEL and kernel methods consistently outperform SimCLR and DCL, \textit{their performance deviates in a much smaller range}, thus empirically proving robustness w.r.t temperature and $\gamma$ for KCL. Importantly, observe that G-kernel and Log-kernel, in most cases demonstrate small spread around the median, a property that hints that are easier to optimise, in accordance to our result that kernel mini-batch losses are unbiased estimators of an objective minimised by perfect alignment \& perfect HE.

\subsection{Ablation studies}\label{sec:ablations}
\looseness-1In the following section, we ablate the aforementioned methods w.r.t. various metrics:
\textbf{(1) Alignment:} An estimate of the expected distance between the representations of a positive pair.
\textbf{(2) Uniformity}: The logarithm of an estimate of the expected pairwise Gaussian energy potential of the distribution of the learned representations
\textbf{(3) Rank}: The rank of a matrix of representations sampled from $\pdata$; reflects the number of dimensions utilised and thus the ability to linearly separate our data \cite{cover1965geometrical, garrido2023rankme}.
\textbf{(4) Effective rank}: A smooth rank approximation \cite{roy2007effective, garrido2023rankme}, that is less prone to numerical errors; has been found in practice to correlate well with downstream performance. Please refer to Appendix \ref{sec:app_experiments} for more details.
 
\paragraph{Novel metric: Wasserstein distance between similarity distributions. }
\looseness-1We introduce \textbf{(5)} a novel metric. The motivation is the fact that, although the uniformity metric is minimised when the representations are uniformly distributed on the unit sphere, it relies on a specific kernel (gaussian) and requires selecting a parameter $t$. Here we propose instead a metric that measures the distance between the ideal inner product (or equivalently L2 distance) distribution and the one that our algorithm yields. In particular, we estimate the \textit{1-Wasserstein distance} $W_1(\qsim, \psim)$  where $\psim$ is the p.d.f of the inner products when $\bu, \bu' \sim U(\mathbb{S}^{d-1})$ and $\qsim$ is the corresponding one when  $\bu, \bu' \sim f_\#\pdata$. According to  \cite{cho2009inner} the former has a closed-form expression that we can use to obtain samples and estimate the distance from data. Importantly, in Appendix \ref{sec:app_experiments}, we show that the \textit{uniformity metric underestimates the closeness of $\qsim$ to the ideal distribution of similarities}, and therefore our metric paints a more complete picture of the learned distribution of representations. In \Cref{fig:properties_temperature} we demonstrate the mean across batch sizes (and $\gamma$ for g-kernel) for  3 representative properties (alignment, Wasserstein distance \& rank) and performance for different temperatures (including all methods that use this hyperparameter and are comparable) in the CIFAR10 and CIFAR100 datasets. The traditional uniformity metric as well as the effective rank are presented in the Appendix \ref{sec:further_res}.

\looseness-1\noindent\textbf{Dimensionality collapse.} With the maximum number of available dimensions being 128, \textit{DHEL consistently utilises a greater number of dimensions}, \eg  in CIFAR100 uses more than double the dimensions as compared to competitors. Additionally, \textit{gaussian-KCL demonstrates once again its stability}, without compromising the rank, albeit not reaching the highest values of DHEL. \textbf{Uniformity.} \textit{DHEL manages to learn representations that are consistently more uniformly distributed across temperature values}. It is also verified that in the low-temperature regime, all methods learn uniformly distributed features, a behaviour that is known as the uniformity-tolerance dilemma \cite{wang2021understanding}.
\textbf{Alignment.} SimCLR and DCL learn more aligned representations, which along with the aforementioned findings seem to imply that uniformity is preferential for DHEL (see also \Cref{sec:discussion}). It is still not clear why this happens, but we may speculate that our modification indeed facilitates optimisation of the second term, and therefore a weighting coefficient might alleviate this modest imbalance. \textit{ Nevertheless, the current balancing seems to benefit downstream performance more.} 

\noindent\textbf{Performance.}
Downstream performance is decreased with respect to temperature for all methods, except G-kernel, but with \textit{DHEL enjoying a greater range of effective temperatures and a smaller rate of decrease in accuracy}. Once again, observe the \textit{remarkable stability of the G-kernel across different hyperparameter values}.

\section{Discussion}
\label{sec:discussion}

\noindent\textbf{Non-asymptotic optima.}
As demnonstrated in \Cref{tab:cl_variants} and \Cref{tab:our_variants} all examined loss variants share the same minimisers in both the mini batch and the asymptotic scenarios. However, there's a practical discrepancy: the optimal solutions for the former case are attainable by optimizing each batch separately, while the latter scenario is not feasible due to the finite size of the dataset.

In the only practical scenario, where the non-asymptotic expected loss is optimised, only the optimal solution of the kernel based methods is known. This means that in practice contrastive methods may or may not have the same optima. This, along the difference in the bias of the estimator (\Cref{prop:kernel_expected}), may explain the inconsistency in both performance and properties between DHEL and KCL methods. Of course, assuming that the target of Contrastive Learning is indeed perfect alignment and uniformity, then \textit{the fact that KCL optimises for it in the non-asymptotic regime is favourable}.

\noindent\textbf{Batch size dependence.} \Cref{prop:kernel_expected} suggests that the expectation of the mini-batch KCL loss is independent of the batch size. In other words, \textit{different batch sizes yield the same expected loss}, contrary to the InfoNCE methods, where essentially, when one changes the batch size, the loss that is optimised changes as well. Does this mean that KCL should have stable performance across batch sizes? In practice one does actually compute an estimation of the expected loss which is affected by the batch size, while the same holds for the gradient of the expected loss. That is the batch size affects the InfoNCE losses by changing both the actual value of the expected loss and its estimate, while the optimisation of kernel losses is affected only by the second.

Too small batch sizes might lead to suboptimal solutions or slow optimisation, due to high-variance gradients. This holds for all methods, regardless of the loss they are optimising for, which might in part explain why all methods tend to improve when increasing the batch size. However, except for the gradient variance, one should also consider the gradient bias. This is zero for KCL, but not zero for the other methods \cite{chen2022we}.

Overall, when increasing the batch size over a threshold below which the gradients are too noisy, KCL obtains better gradient estimates than its counterparts, due to the zero gradient bias. This probably explains why KCL achieves better performance in absolute numbers. Further increasing the batch size, improves the gradient estimates even more for KCL, but also for the other methods since both their gradient variance and their gradient bias are reduced, which might explain why performance keeps improving across all methods. 

\noindent\textbf{Optimisation vs downstream performance.}  Both InfoNCE and kernel-based losses seek to optimise for uniformity and alignment. Our methods do not in all cases achieve to better optimise for both these desired properties. Instead they achieve a balance that better reflects on downstream performance. They also tend to favour uniformity more (\Cref{fig:properties_temperature}) which is probably desired since recent works \cite{gupta2023structuring, xie2022should} have argued that perfect alignment might not be ideal for downstream performance, since several downstream tasks might not actually be invariant to the augmentations from which we obtain the positive samples. 

\noindent\textbf{Connection to supervised learning.} When performing supervised training beyond zero error the class means either form a simplex ETF in the non-asymptotic case or follow a uniform distribution asymptotically, with zero in-class variability \cite{LiuYWS23}. Our analysis shows that the same results hold for contrastive learning. However, in this case, the results apply to individual data points rather than classes. By considering contrastive learning as instance discrimination \cite{wu2018unsupervised} —where each data point represents a unique class— we can identify connections for both optimisation and the representation spaces learned through self-supervised and supervised methods.

\noindent\textbf{Limitations.}
\looseness-1 Our analysis of InfoNCE loss variants focuses on the mini-batch and asymptotic optima. However, it does not address the non-asymptotic optima, which is the scenario typically encountered in practice. Adding that InfoNCE loss variants are fundamentally different from most Machine Learning objectives, where there is no influence of the batch size on the expected loss, the practical behavior of such loss functions requires further research in order to enhance our understanding of contrastive learning optimisation. In contrast, our work does provide the non-asymptotic optima of kernel methods. While these methods serve as unbiased estimators of their expected loss, examining the variance of these estimators can offer valuable insights that can guide the design of methods that are even more robust across different batch sizes.

The experiments in this work aimed to provide empirical results comparing InfoNCE and kernel-based methods across various properties, including robustness to batch sizes and hyperparameters, rank, uniformity, and alignment. However, a more comprehensive evaluation is necessary to better understand the applicability of these methods. To thoroughly assess the superiority of kernel methods, they need to be tested on large-scale datasets and examined under practical conditions such as large batch sizes, memory banks, and momentum contrast \cite{He0WXG20}.

\section{Conclusion}
\looseness-1In this paper, we made a step towards bridging  theory and practice in CL by proving InfoNCE variants share fundamental finite sample and asymptotic optimal solutions. To better attain these optima exhibiting alignment and uniformity, we proposed Decoupled Hyperspherical Energy Loss. Furthermore, establishing kernel CL as equivalent to hyperspherical energy minimization provides optimization advantages. Both new methods empirically demonstrate consistent improvements in downstream performance across different hyperparameters and
small batch sizes, as well as mitigation of dimensionality collapse. 

\section*{Impact Statement}

\looseness-1This paper advances the understanding of contrastive learning (CL) optimisation goals, aiming not just to boost model performance but to clarify the underpinnings of CL losses and their relation to hyperspherical energy minimization (HEM). While our focus is on theoretical insights and introducing the novel Decoupled Hyperspherical Energy Loss (DHEL), this work also lays the groundwork for developing state-of-the-art models with improved robustness and reduced dimensionality collapse. We acknowledge the potential dual-use of our findings and advocate for responsible application and the development of safeguards against misuse. To facilitate further research, we make our code plublicly available at \url{https://github.com/pakoromilas/DHEL-KCL.git}.

\section*{Acknowledgements}
Panagiotis Koromilas was supported by the Hellenic Foundation for Research and Innovation (HFRI) under the 4th Call for HFRI PhD Fellowships (Fellowship Number:  10816). Giorgos Bouritsas and Yannis Panagakis were supported by the project MIS 5154714 of the National Recovery and Resilience Plan Greece 2.0 funded by the European Union under the NextGenerationEU Program. This research was partially supported by a grant from The Cyprus Institute on Cyclone.

\bibliography{references}
\bibliographystyle{icml2024}

\newpage
\appendix
\onecolumn

\section{Additional Preliminaries and Notations}\label{sec:addtional_prelims}

\paragraph{Detailed Formulas for InfoNCE Variants.}
It is not hard to see that all the losses considered in the main text are special cases of the above two losses:
\begin{equation*}
\begin{split}
    \Linfoncea(\bU, \bV)  &=\frac{1}{M} \sum_{i=1}^M-\log \left(\frac{e^{\bu_i^{\top} \bv_i{/\tau}}}{\sum_{j=1}^M e^{{\bu_i^{\top} \bv_j}{/\tau}}}\right)
    =\frac{1}{M} \sum_{i=1}^M \log \left(1+\sum_{j=1, j \neq i}^M e^{{\left(\bv_j-\bv_i\right)^{\top} \bu_i}{/\tau}}\right)\\
    \Lours(\bU, \bV)  &=\frac{1}{M} \sum_{i=1}^M-\log \left(\frac{e^{{\bu_i^{\top} \bv_i}{/\tau}}}{\sum_{j=1 \atop i \neq j}^M e^{\bu_i^{\top} \bu_j{/\tau}}}\right) 
    =\frac{1}{M} \sum_{i=1}^M \log \left( 
    \sum_{j=1, j \neq i}^M e^{\left(\bu_j-\bv_i\right)^\top \bu_i{/\tau}}\right)\\
    \Lsimclr(\bU, \bV)  &=\frac{-1}{M} \sum_{i=1}^M\log \left(\frac{e^{{\bu_i^{\top} \bv_i}{/\tau}}}{\sum_{j=1}^M e^{{\bu_i^{\top} \bv_j}{/\tau}} + \sum_{j=1 \atop i \neq j}^M e^{{\bu_i^{\top} \bu_j}{/\tau}}}\right)\\ 
    &=\frac{1}{M} \sum_{i=1}^M \log \left(1+\sum_{j=1, j \neq i}^M \left(e^{{\left(\bv_j-\bv_i \right)^{\top} \bu_i}{/\tau}} + e^{{\left(\bu_j-\bv_i\right)^\top \bu_i}{/\tau}}\right)\right)\\
    \Ldcl(\bU, \bV)  &=\frac{-1}{M} \sum_{i=1}^M\log \left(\frac{e^{{\bu_i^{\top} \bv_i}{/\tau}}}{\sum_{j=1 \atop j \neq i}^M e^{{\bu_i^{\top} \bv_j}{/\tau}} + \sum_{j=1 \atop i \neq j}^M e^{{\bu_i^{\top} \bu_j}{/\tau}}}\right)\\ 
    &=\frac{1}{M} \sum_{i=1}^M \log \left(\sum_{j=1, j \neq i}^M \left(e^{{\left(\bv_j-\bv_i \right)^{\top} \bu_i}{/\tau}} + 
     e^{{\left(\bu_j-\bv_i\right)^\top \bu_i}{/\tau}}\right)\right)
\end{split}
\end{equation*}

Therefore,
\begin{equation}\label{eq:special_losses}
\begin{split}
    \Linfoncea(\bU, \bV) &= \Lgeninfoncea\left(\bU, \bV; \exp(x/\tau) ; \log(1 + x) \right) \\
    \Lours(\bU, \bV) &= \Lgeninfonceb\left(\bU, \bV; \exp(x/\tau) ; \log(x) \right) \\
    \Lsimclr(\bU, \bV)  & = \Lgensimclr\left(\bU, \bV; \exp(x/\tau) ; \log(1 + x) \right) \\
    \Ldcl(\bU, \bV)   & = \Lgensimclr\left(\bU, \bV; \exp(x/\tau) ; \log(x) \right) \\
\end{split}
\end{equation}

\paragraph{Kernels.}
Notable examples of kernels that obey the conditions that we encounter in this paper are the following: 
\begin{itemize}
    \item \textit{Linear}: ${\Linear_t(\bx, \by) = -t \|\bx - \by \|^2}$$= \linear_t (\|\bx-\by\|^2)$, where $\linear_t(x) = -t x$.
    \item \textit{Gaussian}: ${\Gauss_t(\bx, \by) = e^{-t \|\bx - \by \|^2}}$ $= \gauss_t (\|\bx-\by\|^2)$, where $\gauss_t(x;t) = e^{-t x}$.
    \item \textit{Riesz}: ${\Riesz_s(\bx, \by) = \text{sign}(s) \|\bx - \by \|^{-s}}$$= \riesz_s (\|\bx-\by\|^2)$, where $\riesz_s(x) = \text{sign}(s)x^{-s/2}$.
    \item \textit{Inverse Multiquadric (IMQ)}:
    ${K^{imq}_{c}(\bx, \by) = -\frac{1}{2}\log\left(s \|\bx - \by \|^{2} + \beta\right)}$$=\logar_{s, \beta} (\|\bx-\by\|^2)$, where $\logar_{s, \beta}(x) = -\frac{1}{2}\log \left( sx + \beta\right)$.
    \item \textit{Logarithmic}:
    ${\Logar_{s, \beta}(\bx, \by) = -\frac{1}{2}\log\left(s \|\bx - \by \|^{2} + \beta\right)}$$=\logar_{s, \beta} (\|\bx-\by\|^2)$, where $\logar_{s, \beta}(x) = -\frac{1}{2}\log \left( sx + \beta\right)$.
\end{itemize}

The properties of kernel functions that arise in the theoretical results are the below: (1) (strict) \textit{monotonicity}, (2) (strict) \textit{convexity}, (3) (strict) \textit{absolute monotonicity}, \ie\ derivatives of all orders $f^{(n)}$ exist and are non-negative (positive in the strict case) everywhere)and (4) \textit{complete monotonicity}, \ie\ derivatives of all orders exist and $(-1)^n f^{(n)}\geq 0$ everywhere ($>0$ for the strict case).

With elementary derivations, it is easy to see the following for $t > 0$, $\linear_t$ and $\gauss_t$ are strictly decreasing and convex, while the latter is also strictly convex. For $s>-2$, $\riesz_s$ is strictly decreasing and strictly convex, while the same holds for $\logar_{s, \beta}$ when $s,\beta>0$. Additionally, for $t>0$, $\linear_t$ is completely monotone, $\gauss_t$ is strictly completely monotone, while the same holds for their negative first derivatives $-(\linear_t)^{(1)}, -(\gauss_t)^{(1)}$. $\riesz_s$ is strictly completely monotone for $s>0$, while its negative first derivative is strictly completely monotone for $s>-2$. $\logar_{s, \beta}$ is strictly completely monotone for $s,\beta>0$, while the same holds for its negative first derivative.

\section{Deferred Proofs}
\subsection{Minima of Mini-Batch Contrastive Losses}\label{sec:mb_minima}

Now we can proceed in proving our first theorem which encapsulates Theorems \ref{thrm:mb_main_theorem} and \ref{thrm:mb_dhel_main_theorem} of the main paper.

\begin{theorem}\label{thrm:mb_main_theorem_proof}
Consider the following optimisation problem:
\begin{equation}\label{eq: mb_optimisation_proof}
    \underset{\bU, \bV \in (\mathbb{S}^{d-1})^M}{\argmin} \Lgentotal(\bU, \bV),
\end{equation}
with $\Lgentotal(\bU, \bV) = \frac{1}{2} \left(\Lgen(\bU, \bV) + \Lgen(\bV, \bU) \right)$,
where $\mathbb{S}^{d-1} = \{\bu \in \mathbb{R}^{d}\  | \ \|\bu\|_2=1\}$ a unit sphere of $d$ dimensions, $\bU, \bV$ are tuples of $M$ vectors on the unit sphere and $\Lgen(\cdot, \cdot)$ is any of the loss functions $\{\Lgeninfoncea(\cdot, \cdot; \phi, \psi),  \Lgensimclr(\cdot, \cdot; \phi, \psi), \Lgeninfonceb(\cdot, \cdot; \phi, \psi),\}$ as defined in  Eq. \eqref{eq:general_losses} and \eqref{eq:general_dhel_loss}. Further, suppose the following conditions: (1) $\phi: \mathbb{R} \to \mathbb{R}$ is \textbf{increasing \& convex}, (2) $\psi: \mathbb{R} \to \mathbb{R}$ is \textbf{increasing \& $\tilde{\psi}(x; \alpha) = \psi\left(\alpha \phi\left(x\right)\right)$ is convex for $\alpha>0$} and (3) $1<M \leq d + 1$. Then, the optimisation problem of Eq. \eqref{eq: mb_optimisation} obtains its optimal value $(\bU^*, \bV^*)$ when:
\begin{equation}\label{eq:mb_minima_proof}
    \bU^* = \bV^* \quad \text{ and } \quad \bU^*= [\bu^*_1, \dots, \bu^*_M] \  \text{ form a regular $M-1$ simplex centered at the origin}.
\end{equation}
Additionally, (4) if $\psi, \phi$ are \textbf{strictly increasing} and $\tilde{\psi}$ is  \textbf{strictly convex} then all the $(\bU^*, \bV^*)$ that satisfy Eq. \eqref{eq:mb_minima_proof} are the \textbf{unique} optima.
\end{theorem}

\begin{proof}

Let us start from $\Lgeninfoncea$. Our proof will follow similar steps as in \cite{sreenivasan2023minibatch} but in a more general fashion. \\ \\
\noindent\textbf{Part I: $\Lgeninfoncea$.}

First, we will use the convexity of $\phi$ to lower bound the inner sum.
\begin{align*} \sum_{j=1, j \neq i}^M \phi\big({\left(\bv_j-\bv_i\right)^{\top} \bu_i}\big) 
& \stackrel{(a)}{\geq}(M-1)
\phi \left(\frac{1}{M-1}\sum_{j=1, j \neq i}^M \left(\bv_j^{\top} \bu_i  - \bv_i^{\top} \bu_i \right)\right) \\ 
& =(M-1)
\phi \left(\frac{\bv^{\top} \bu_i - \bv_i^{\top} \bu_i - (M-1)( \bv_i^{\top }\bu_i)}{M-1} \right)  =(M-1)
\phi \left(\frac{\bv^{\top} \bu_i- M\bv_i^{\top} \bu_i}{M-1}   \right),
\end{align*}
where $(a)$ follows from Jensen's inequality (Condition (1)) and $\bv = \sum_{j=1}^M \bv_j$. Then, we will use the convexity of $\tilde{\psi}$ (Condition (2)) to bound the outer sum.

\begin{align*} \Lgeninfoncea(\bU, \bV; \phi, \psi) &  \stackrel{(b)}{\geq}\frac{1}{M} \sum_{i=1}^M \psi \left((M-1) \phi \left(\frac{\bv^{\top} \bu_i - M\bv_i^{\top} \bu_i}{M-1}\right)\right)\\
&= \frac{1}{M} \sum_{i=1}^M \tilde{\psi}  \left(\frac{\bv^{\top} \bu_i - M\bv_i^{\top} \bu_i}{M-1} ; M-1\right) \\
& \stackrel{(c)}{\geq} \tilde{\psi} \left(\frac{1}{M} \sum_{i=1}^M\left(\frac{\bv^{\top} \bu_i - M\bv_i^{\top} \bu_i}{M-1}\right);M-1 \right) \\
& =\psi \left((M-1) \phi \left(\frac{1}{M} \sum_{i=1}^M\left(\frac{\bv^{\top} \bu_i - M\bv_i^{\top} \bu_i}{M-1}\right)\right)\right) \\ 
& = \psi \left((M-1) \phi \left(\frac{1}{M(M-1)}\left({\bv^{\top} \bu}-{M} \sum_{i=1}^M\bv_i^{\top} \bu_i\right)\right)\right) \\
& =\psi \left((M-1) \phi \left(\frac{1}{M} \sum_{i=1}^M\left(\frac{\bv^{\top} \bu_i - M\bv_i^{\top} \bu_i}{M-1}\right)\right)\right) \\ 
& \stackrel{(d)}{\geq} \psi \left((M-1) \phi \left(\frac{1}{M(M-1)}\left({\bv^{*\top} \bu^*}-{M} \sum_{i=1}^M\bv_i^{*\top} \bu^*_i\right)\right)\right),
\end{align*}
where $\bu = \sum_{j=1}^M \bu_j$ and $\bu_i^*, \bv_i^*$ are the minima of the inner argument, $(b)$ follows from the fact that $\psi$ is increasing (Condition (2)), $(c)$ follows from the fact that $M-1>0$ and Jensen's inequality (Condition (2)) and (d) from the fact that $\psi$ and $\phi$ are increasing (Conditions (1) and (2)) and $M-1 > 0$. Thus, it suffices to solve the following optimisation problem:
\begin{equation*}
    \underset{\bU, \bV \in (\mathbb{S}^{d-1})^M}{\argmin}  \frac{1}{M(M-1)}\left({\bv^{\top} \bu}-{M} \sum_{i=1}^M\bv_i^{\top} \bu_i\right)\Leftrightarrow \underset{\bU, \bV \in (\mathbb{S}^{d-1})^M}{\argmax}   M \sum_{i=1}^M \bv_i^{\top} \bu_i-\left(\sum_{i=1}^M \bv_i\right)^{\top}\left(\sum_{i=1}^M \bu_i\right) 
\end{equation*}
or equivalently, as shown in Eq. (12) in the Appendix of \cite{sreenivasan2023minibatch},
\begin{equation}\label{eq:mb_optim_1st}
\underset{\bU, \bV \in (\mathbb{S}^{d-1})^M}{\argmax}  \quad \bv_{\text {stack }}^{\top}\left(\left(M \mathbb{I}_M-\mathbf{1}_M \mathbf{1}_M^{\top}\right) \otimes \mathbb{I}_d\right) \bu_{\text {stack }},
\end{equation}
where $\bv_{\text {stack}} = \left[\bv_1^\top, \dots, \bv_M^\top\right]^\top \in \mathbb{R}^{Md\times 1}$, $\bu_{\text {stack}} = \left[\bu_1^\top, \dots, \bu_M^\top\right]^\top \in \mathbb{R}^{Md\times 1}$, $\mathbb{I}_M \in \mathbb{R}^{M \times M}$ the identity matrix, $\mathbf{1}_M \in \mathbb{R}^{M \times 1}$ a vector whose elements are all equal to 1 and $\otimes$ denotes the Kronecker product. The optimisation problem of Eq. \eqref{eq:mb_optim_1st} has been studied in \cite{LU2022224} and \cite{sreenivasan2023minibatch}. We repeat the result here for completeness. 

It is shown that the eigenvalues of $\left(M \mathbb{I}_M-\mathbf{1}_M \mathbf{1}_M^{\top}\right)$ are equal to $M$ with multiplicity $M-1$ and 0 with multiplicity 1, with corresponding eigenvectors $\bp \in \mathbb{R}^M$ such that $\bp^\top \mathbf{1}_M = 0$ and $\bp = k \mathbf{1}_M$ respectively. Therefore, since the set of eigenvalues of the Kronecker product of two matrices contains all the possible pair-wise products of the eigenvalues of the individual matrices, the eigenvalues of $\left(M \mathbb{I}_M-\mathbf{1}_M \mathbf{1}_M^{\top}\right) \otimes \mathbb{I}_d$ are  $M$ with multiplicity $M(M-1)$ and $0$ with multiplicity $M$. Since the matrix of interest is symmetric, its singular values coincide with the absolute of its eigenvalues, and therefore $\|\left(M \mathbb{I}_M-\mathbf{1}_M \mathbf{1}_M^{\top}\right) \otimes \mathbb{I}_d\|_2 = M$. Concluding:

\begin{align}
\bv_{\text {stack }}^{\top}\left(\left( M \mathbb{I}_M - \mathbf{1}_M \mathbf{1}_M^{\top}\right) \otimes \mathbb{I}_d\right) \bu_{\text {stack }}
&\stackrel{(e)}{\leq} \|\bv_{\text {stack }}\|_2 \|\left(\left(M \mathbb{I}_M-\mathbf{1}_M \mathbf{1}_M^{\top}\right) \otimes \mathbb{I}_d\right)\bu_{\text {stack }}\|_2\notag\\
&\stackrel{(f)}{\leq} \|\bv_{\text {stack }}\|_2 \|\left(M \mathbb{I}_M-\mathbf{1}_M \mathbf{1}_M^{\top}\right) \otimes \mathbb{I}_d \|_2 \|\bu_{\text {stack }}\|_2\notag\\
&= (\sqrt{M}) M \sqrt{M} = M^2 \label{spectral_norm},
\end{align}
where (e) follows from the Cauchy–Schwarz inequality, (f) from the definition of the spectral norm and the last equality from the fact that $\|\bu_i\| = \|\bv_i\| = 1, \forall i \in [M]$. Now, moving backwards for every inequality (a)-(e) we have that:

\begin{itemize}
\item (f) 
holds with equality \textbf{iff} $\bu_{\text{stack}}$ is an eigenvector corresponding to the maximum eigenvalue, \ie\ $\left(\left(M \mathbb{I}_M-\mathbf{1}_M \mathbf{1}_M^{\top}\right) \otimes \mathbb{I}_d \right)\bu_{\text{stack}} = M \bu_{\text{stack}}$. But using a common property of the Kronecker product and the fact that $\text{vec}(\bU) = \bu_{\text{stack}}$, it follows that 
$\left(\left(M \mathbb{I}_M-\mathbf{1}_M \mathbf{1}_M^{\top}\right) \otimes \mathbb{I}_d \right)\bu_{\text{stack}} = \text{vec}\left(\mathbb{I}_d \bU \left(M \mathbb{I}_M-\mathbf{1}_M \mathbf{1}_M^{\top}\right)^\top \right)
= \text{vec}\left([\bu_1 \dots \bu_M] \left(M \mathbb{I}_M-\mathbf{1}_M \mathbf{1}_M^{\top}\right)^\top \right)$ and thus:
\begin{align}
\text{vec}\left([\bu_1 \dots \bu_M] \left(M \mathbb{I}_M-\mathbf{1}_M \mathbf{1}_M^{\top}\right)^\top \right) = M \bu_{\text{stack}}
&\Leftrightarrow
[\bu_1^{(\ell)} \dots \bu_M^{(\ell)}] \left(M \mathbb{I}_M-\mathbf{1}_M \mathbf{1}_M^{\top}\right)^\top = M [\bu_1^{(\ell)} \dots \bu_M^{(\ell)}]\notag\\
&\Leftrightarrow 
\left(M \mathbb{I}_M-\mathbf{1}_M \mathbf{1}_M^{\top}\right)[\bu_1^{(\ell)}, \dots.\bu_M^{\ell}]^\top = M [\bu_1^{(\ell)}, \dots, \bu_M^{(\ell)}]^\top\notag\\
&\Leftrightarrow [\bu_1^{(\ell)}, \dots, \bu_M^{(\ell)}] \mathbf{1}_M = 0, \forall \ell \in [d]\notag\\
&\Leftrightarrow \sum_{i=1}^M \bu^*_i =\bu^* = \mathbf{0}.\label{eq:zero_mean_optimum}
\end{align}
\item (e) 
holds with equality \textbf{iff} $\bv_{\text{stack}} = k\bu_{\text{stack}}$ for any $k>0$. But since $\|\bv_{\text{stack}}\| = \|\bu_{\text{stack}}\| = \sqrt{M}$, then it must be that $k=1$ and therefore:
\begin{equation}\label{eq:alignment_optimum}
    \bv^*_i = \bu^*_i, \ \forall i \in [M].
\end{equation}
\item (d)
holds with equality only if $\bu_i = \bu_i^*$ and $\bv_i = \bv_i^*$ when \textbf{$\phi, \psi$ are strictly increasing} (Condition (4)). If the latter doesn't hold, then we might obtain the optimum for other input values besides the ones that were found in the previous two conditions.
\item (c) holds with equality \textbf{iff} $\frac{\bv^{\top} \bu_i - M\bv_i^{\top} \bu_i}{M-1} = c_1$ constant $\forall i \in [M]$ or if $\tilde{\psi}$ is linear. The former is already satisfied by the  
conditions \eqref{eq:zero_mean_optimum}, \eqref{eq:alignment_optimum} and $c_1 =-\frac{M}{M-1}$. Thus, no extra condition arises here.
\item (b) holds with equality only if $\bu_i, \bv_i$ minimise the arguments of $\psi$ (which happens when the conditions of (a) are satisfied) when \textbf{$\psi$ is strictly increasing} (Condition (4)). Again, if the latter doesn't hold, we might obtain the optimum for other input values besides those found by the rest of the conditions.
\item (a) holds with equality if $(\bv_j - \bv_i)^{\top} \bu_i = \tilde{c}_2$ constant $\forall i,j \in [M]$. If $\phi$ is \textbf{strictly convex} (Condition (4)), then this will be the only condition allowing equality to hold. Given condition \eqref{eq:alignment_optimum}, the former implies that $\bu_j^\top \bu_i = 1+ \tilde{c}_2 = c_2$ and since from condition \eqref{eq:zero_mean_optimum} we have that $\bu = 0$, then:
\begin{align}\label{eq:ETF}
    0 &= \bu^\top \bu = \left(\sum_{i=1}^M\bu_i^\top\right)  \left(\sum_{i=1}^M \bu_i\right) = \sum_{i=1}^M \bu_i^\top \bu_i + \sum_{i,j=1 \atop j\neq i}^M \bu_i^\top \bu_j \Leftrightarrow 0 = M + M(M-1) c_2\nonumber\\
    &\Leftrightarrow \bu_i \bu_j = -\frac{1}{M-1} \forall i \neq j \in [M].
\end{align}

\end{itemize}

Conditions \eqref{eq:zero_mean_optimum}, \eqref{eq:alignment_optimum}, \eqref{eq:ETF} prove that any point configuration that satisfies Eq. \eqref{eq:mb_minima} is an optimum of Eq. \eqref{eq: mb_optimisation} for the $\Lgeninfoncea$ loss and with the additional conditions of strict monotonicity of $\phi, \psi$ and strict convexity of $\phi$ we also obtain that these point configurations are the unique optima. 

It remains to show that these optima can be indeed attained. In particular,
Eq. \eqref{eq:zero_mean_optimum} and Eq. \eqref{eq:alignment_optimum} are easy to attain for any twin (identical) point configurations that are centred at the origin.
Eq. \eqref{eq:ETF}, or more precisely the fact that all angles are equal, is exactly the definition of a \textit{regular $M-1$-dimensional simplex inscribed in the sphere}. However, for the regular $M-1$-dimensional simplex to exist the ambient dimension must be at least as large as $M-1$, \ie $d\geq M-1$, which justifies Condition (3).
\\

\noindent\textbf{Part II: $\Lgensimclr$.} 
Using a similar rationale for $\Lgensimclr$ we obtain:
\begin{equation*}
\sum_{j=1, j \neq i}^M \left(\phi\left(\left(\bv_j-\bv_i\right)^{\top} \bu_i\right) + \phi\left(\left(\bu_j-\bv_i\right)^{\top} \bu_i\right)\right)
 \stackrel{(a')}{\geq} 
2(M-1)\phi \left(\frac{\bv^{\top} \bu_i- (2M-1)\bv_i^{\top} \bu_i + \bu^{\top} \bu_i - 1}{2(M-1)}   \right), 
\end{equation*}
and subsequently:
\begin{align*}    
\Lgensimclr(\bU, \bV) 
&  \stackrel{(b')}{\geq} \frac{1}{M} \sum_{i=1}^M \psi \left(2(M-1) \phi \left(\frac{\bv^{\top} \bu_i- (2M-1)\bv_i^{\top} \bu_i + \bu^{\top} \bu_i - 1}{2(M-1)}\right) \right) \\
& = \frac{1}{M} \sum_{i=1}^M \tilde{\psi}  \left(\frac{\bv^{\top} \bu_i- (2M-1)\bv_i^{\top} \bu_i + \bu^{\top} \bu_i- 1}{2(M-1)} ;2(M-1)\right)\\
& \stackrel{(c')}{\geq} \psi \left(2(M-1) \phi \left(\frac{1}{M} \sum_{i=1}^M\left(\frac{\bv^{\top} \bu_i- (2M-1)\bv_i^{\top} \bu_i + \bu^{\top} \bu_i- 1}{2(M-1)}\right)\right)\right) \\ 
& = \psi \left(2(M-1) \phi \left(\frac{1}{2M(M-1)}\left(\bv^{\top} \bu-(2M-1)\sum_{i=1}^M\left(\bv_i^{\top} \bu_i\right)  + \bu^{\top} \bu- M\right)\right)\right)\\
& \stackrel{(d')}{\geq}\psi \left(2(M-1) \phi \left(\frac{1}{2M(M-1)}\left(\bv^{*\top} \bu^*-(2M-1)\sum_{i=1}^M\left(\bv_i^{*\top} \bu^*_i\right)  + \bu^{*\top} \bu^*- M\right)\right)\right),
\end{align*}
where once again $(a')$ follows from Jensen's inequality, $(b')$ follows from the fact that $\psi$ is increasing, $(c')$ follows from Jensen's inequality and the fact that $M-1>0, 2(M-1)>0$ and (d') from the fact that $\psi$ and $\phi$ are increasing and $M-1 > 0, 2M(M-1)>0$. Again it suffices to solve the following optimisation problem.
\begin{align}
&\underset{\bU, \bV \in (\mathbb{S}^{d-1})^M}{\argmin}  \frac{1}{2M(M-1)}\left({\bv^{\top} \bu}-
(2M -1) \sum_{i=1}^M\bv_i^{\top} \bu_i + {\bu^{\top} \bu} - M\right)\notag\\
&= \underset{\bU, \bV \in (\mathbb{S}^{d-1})^M}{\argmax}   (2M-1) \sum_{i=1}^M \bv_i^{\top} \bu_i-(\bv^{\top} + \bu^{\top})\bu\notag\\
&=\underset{\bU, \bV \in (\mathbb{S}^{d-1})^M}{\argmax}  (2M-1) \sum_{i=1}^M \bv_i^{\top} \bu_i-(\bv^{\top} + \bu^{\top})\bu+ (2M-1)\sum_{i=1}^M \bu_i^{\top} \bu_i - M(2M-1)\notag\\
&= \underset{\bU, \bV \in (\mathbb{S}^{d-1})^M}{\argmax}
(2M-1) \sum_{i=1}^M (\bv_i^\top+ \bu_i^{\top})\bu_i-(\bv^{\top} + \bu^{\top})\bu\notag\\
&= \underset{\bU, \bV \in (\mathbb{S}^{d-1})^M}{\argmax} (\bv_{\text {stack }} + \bu_{\text {stack }})^{\top}\left(\left((2M-1) \mathbb{I}_M-\mathbf{1}_M \mathbf{1}_M^{\top}\right) \otimes \mathbb{I}_d\right) \bu_{\text {stack }}.
\end{align}

Similarly with \cite{LU2022224}, we will find the eigendecomposition of $(2M-1) \mathbb{I}_M-\mathbf{1}_M \mathbf{1}_M^{\top}$.
Select $\bp \in \mathbb{R}^M$ with $\bp^\top \mathbf{1}_M = 0$. Then, $\left((2M-1) \mathbb{I}_M-\mathbf{1}_M \mathbf{1}_M^{\top}\right) \bp = (2M-1)\bp -\mathbf{1}_M \mathbf{1}_M^{\top}\bp = (2M-1)\bp$ and so we proved that $(2M-1)$ is an eigenvalue for all the vectors with $\bp^\top \mathbf{1}_M = 0$, \ie\ with multiplicity $M-1$. Additionally, select $\bp \in \mathbb{R}^M$ with $\bp = k\mathbf{1}_M$. Then, $\left((2M-1) \mathbb{I}_M-\mathbf{1}_M \mathbf{1}_M^{\top}\right) \bp = (2M-1)\bp -\mathbf{1}_M \mathbf{1}_M^{\top}\bp = k(2M-1)\mathbf{1}_M - kM \mathbf{1}_M = k(M-1)\mathbf{1}_M = (M-1)\bp$ and so we proved that $(M-1)$ is an eigenvalue for all the vectors with $\bp = k\mathbf{1}_M$, \ie\ with multiplicity $1$. Thus, as in Part I, 
$\|\left((2M-1) \mathbb{I}_M-\mathbf{1}_M \mathbf{1}_M^{\top}\right) \otimes \mathbb{I}_d\|_2 = 2M-1$. Concluding:

\begin{align}
(\bu_{\text {stack }}^{\top} + \bv_{\text {stack }}^{\top})\left(\left( (2M-1) \mathbb{I}_M - \mathbf{1}_M \mathbf{1}_M^{\top}\right) \otimes \mathbb{I}_d\right) \bu_{\text {stack }}
&\stackrel{(e')}{\leq} \|\bu_{\text {stack }} + \bv_{\text {stack }}\|_2 \|\left(\left((2M-1) \mathbb{I}_M-\mathbf{1}_M \mathbf{1}_M^{\top}\right) \otimes \mathbb{I}_d\right)\bu_{\text {stack }}\|_2\notag\\
&\stackrel{(f')}{\leq} \|\bu_{\text {stack }} + \bv_{\text {stack }}\|_2 
\|\left((2M-1) \mathbb{I}_M-\mathbf{1}_M \mathbf{1}_M^{\top}\right) \otimes \mathbb{I}_d \|_2 \|\bu_{\text {stack }}\|_2\notag\\
&\stackrel{(g')}{\leq} \left(\|\bu_{\text {stack }}\|_2 + \|\bv_{\text {stack }}\|_2\right) 
\|\left((2M-1) \mathbb{I}_M-\mathbf{1}_M \mathbf{1}_M^{\top}\right) \notag\\
& \qquad \qquad \qquad \qquad \qquad \qquad \qquad \qquad \otimes \mathbb{I}_d \|_2 \|\bu_{\text {stack }}\|_2\notag\\
&= (2\sqrt{M}) (2M-1) \sqrt{M} = 2(2M-1)M^2 \label{eq:spectral_norm_2},
\end{align}
where again (e') follows from the Cauchy–Schwarz inequality, (f') from the definition of the spectral norm and (g') from the triangle inequality. Now, moving backwards for every inequality (a')-(g') as in Part I\footnote{(g') holds with equality for the same conditions as with (f')} we will obtain the same conditions as in Part I which will prove the desideratum for $\Lgensimclr$.

\noindent\textbf{Part III: $\Lgeninfonceb$.} 
Following the exact same steps as before we obtain: First for the inner summands, 
\begin{align*}
&\sum_{j=1, j \neq i}^M \phi\left(\left(\bu_j-\bv_i\right)^{\top} \bu_i\right)
 \stackrel{(a'')}{\geq}(M-1)
\phi \left(\frac{\sum_{j=1, j \neq i}^M \bu_j^{\top} \bu_i  - (M-1)\bv_i^{\top} \bu_i}{M-1} \right)\\ 
&= (M-1)\phi \left(\frac{\bu^{\top} \bu_i- (M-1)\bv_i^{\top} \bu_i - 1}{M-1}   \right),
\end{align*}
then for the total loss:
\begin{align*}    
\Lgeninfonceb(\bU, \bV) 
&  \stackrel{(b'')}{\geq} \frac{1}{M} \sum_{i=1}^M \psi \left((M-1) \phi \left(\frac{\bu^{\top} \bu_i- (M-1)\bv_i^{\top} \bu_i  - 1}{M-1}\right) \right) \\
& = \frac{1}{M} \sum_{i=1}^M \tilde{\psi}  \left(\frac{\bu^{\top} \bu_i- (M-1)\bv_i^{\top} \bu_i  - 1}{M-1} ;M-1\right)\\
& \stackrel{(c'')}{\geq} \psi \left((M-1) \phi \left(\frac{1}{M} \sum_{i=1}^M\left(\frac{\bu^{\top} \bu_i- (M-1)\bv_i^{\top} \bu_i- 1}{M-1}\right)\right)\right) \\ 
& = \psi \left((M-1) \phi \left(\frac{1}{M(M-1)}\left(\bu^{\top} \bu-(M-1)\sum_{i=1}^M\left(\bv_i^{\top} \bu_i\right)  - M\right)\right)\right)\\
& \stackrel{(d'')}{\geq} \psi \left((M-1) \phi \left(\frac{1}{M(M-1)}\left(\bu^{*\top} \bu^*-(M-1)\sum_{i=1}^M\left(\bv_i^{*\top} \bu^*_i\right)  - M\right)\right)\right),
\end{align*}
and finally for the inner argument optimisation problem:
\begin{align}
\underset{\bU, \bV \in (\mathbb{S}^{d-1})^M}{\argmin}  \frac{1}{M(M-1)}\left({\bu^{\top} \bu}-
(M -1) \sum_{i=1}^M\bv_i^{\top} \bu_i  - M\right)
&= \underset{\bU, \bV \in (\mathbb{S}^{d-1})^M}{\argmax}   (M-1) \sum_{i=1}^M \bv_i^{\top} \bu_i-\|\bu\|^2.
\end{align}
The last part of the proof here is slightly different. In particular, the two terms in the above equation can be maximised independently. For the first term, we have that each summand in the sum can be optimised independently and that:
\begin{align}
     \bv_i^{\top} \bu_i \stackrel{(e'')}{\leq} \|\bv_i\| \|\bu_i\| = 1,
\end{align}
where again we used Cauchy–Schwarz. Now (e'') holds with equality, as before, iff $\bv_i = \bu_i$. For the second term, we need to minimise $\|\bu\|^2$, which evidently happens iff $\bu = \mathbf{0}$. Therefore, we arrived at the same conditions as in Part I and Part II, while the rest of the equalities in (a'')-(d'') while be satisfied as before. This concludes the proof.

\end{proof}

The below corollary follows directly from Theorem \ref{thrm:mb_main_theorem_proof}.
\begin{corollary} 
Consider the following optimisation problem:
\begin{equation}\label{eq: mb_optimisation_theta}
    \underset{\boldsymbol{\theta} \in \Theta}{\argmin} \Lgentotal\left(\fun(\bX),\fun(\bY)\right),
\end{equation}
where $f_{\boldsymbol{\theta}}: \cX \to \mathbb{S}^{d-1}$ is an encoder function parametrised by a tuple of parameters $\boldsymbol{\theta}$ and $\Lgentotal$ is a contrastive loss function defined as in Theorem \ref{thrm:mb_main_theorem_proof}. Suppose that the conditions (1)-(3) set in Theorem \ref{thrm:mb_main_theorem_proof} hold. Further, suppose that (5) $\exists \, \boldsymbol{\theta}^* \in \Theta$ such that $f_{\boldsymbol{\theta}}$ achieves \textbf{simultaneously perfect alignment and perfect uniformity}, \ie\ that:
\begin{equation}\label{eq:mb_minima_theta}
\funstar(\bX) = \funstar(\bY) \quad \text{ and } \quad \funstar(\bX)\  \text{form a regular $M-1$ simplex}.
\end{equation}
Then, the optimisation problem of Eq. \eqref{eq: mb_optimisation_theta} obtains its optimal value for all $\boldsymbol{\theta}^*$ that satisfy Eq. \eqref{eq:mb_minima_theta}. Additionally, if the condition (4) set in Theorem \ref{thrm:mb_kernels_theorem_proof} holds, then all the $\boldsymbol{\theta}^*$ that satisfy Eq. \eqref{eq:mb_minima_theta} are the \textbf{unique} optima.
\end{corollary}

\begin{corollary}  The following mini-batch CL loss functions: $\Linfoncea(\cdot, \cdot)$, 
$\Lours(\cdot, \cdot)$, $\Lsimclr(\cdot, \cdot)$, $\Ldcl(\cdot, \cdot)$ have the \textbf{same unique minima } on the unit sphere when $1<M\leq d+1$, \ie\ all the optimal solutions of Eq. \eqref{eq: mb_optimisation_proof} will satisfy the properties of Eq. \eqref{eq:mb_minima_proof}.
\end{corollary}

\begin{proof}
Recall from Eq. \eqref{eq:special_losses} that the above losses are special cases of $\Lgeninfoncea, \Lgensimclr, \Lgeninfonceb$ with $\phi(x) = \exp (x/\tau)$, and $\psi(x) = \log(x)$ or  $\psi(x) = \log(1 + x)$. We know that for $\tau >0$, $\exp (x/\tau)$ is strictly increasing and strictly convex, while $\log(1 + \alpha\exp(x/\tau))$ and $\log(\alpha\exp(x/\tau)) = \log \alpha + x/\tau$ are strictly increasing and convex for $\alpha, \tau >0$. Therefore, all Conditions (1)-(4) of Theorem \ref{thrm:mb_main_theorem} are satisfied.
\end{proof}

\subsection{Expected (True) Contrastive Losses and Asymptotic Behaviour}\label{sec:expectation_proofs}
\begin{lemma}\label{lemma:exp_cl_losses}
The expectations of the following mini-batch contrastive loss functions: $\Linfoncea, \Lsimclr, \Ldcl, \Lours$ are the ones given in Tables \ref{tab:cl_variants} and \ref{tab:our_variants}.
\end{lemma}
\begin{proof}
It is straightforward to see that the expectations of the mini-batch losses are as follows:
\begin{align}
\underset{(\bX, \bY)\overset{\text{i.i.d}}{\sim} \ppos^M}{\E}\left[\Lgeninfoncea(f_{\boldsymbol{\theta}}(\bX), f_{\boldsymbol{\theta}}(\bY))\right]
&= \underset{(\bX, \bY)\overset{\text{i.i.d}}{\sim} \ppos^M}{\E}\left[\frac{1}{M} \sum_{i=1}^M \psi \left(\sum_{j=1, j \neq i}^M \phi\left(\left(f_{\boldsymbol{\theta}}\left(\by_j\right)-f_{\boldsymbol{\theta}}\left(\by_i\right)\right)^{\top} f_{\boldsymbol{\theta}}\left(\bx_i\right)\right)\right)\right]\notag\\
&=\frac{1}{M} \sum_{i=1}^M  \underset{(\bX, \bY)\overset{\text{i.i.d}}{\sim} \ppos^M}{\E}\left[\psi \left(\sum_{j=1, j \neq i}^M \phi\left(\left(f_{\boldsymbol{\theta}}\left(\by_j\right)-f_{\boldsymbol{\theta}}\left(\by_i\right)\right)^{\top} f_{\boldsymbol{\theta}}\left(\bx_i\right)\right)\right)\right]\notag\\
&=\frac{1}{M} \sum_{i=1}^M  \underset{(\bx_i, \by_i)\sim \ppos \atop \{\by_j\}^{M-1}_{j=1}\overset{\text{i.i.d}}{\sim} \pdata}{\E}\left[\psi \left(\sum_{j=1, j \neq i}^M \phi\left(\left(f_{\boldsymbol{\theta}}\left(\by_j\right)-f_{\boldsymbol{\theta}}\left(\by_i\right)\right)^{\top} f_{\boldsymbol{\theta}}\left(\bx_i\right)\right)\right)\right]\notag\\
&=\underset{(\bx, \by)\sim \ppos \atop \{\by_j\}^{M-1}_{j=1}\overset{\text{i.i.d}}{\sim} \pdata}{\E}\left[\psi \left(\sum_{j=1}^{M-1} \phi\left(\left(f_{\boldsymbol{\theta}}\left(\by_j\right)-f_{\boldsymbol{\theta}}\left(\by\right)\right)^{\top} f_{\boldsymbol{\theta}}\left(\bx\right)\right)\right)\right]\label{eq:emb-nce-equality}\\
\underset{(\bX, \bY)\overset{\text{i.i.d}}{\sim} \ppos^M}{\E}\left[\Lgensimclr(f_{\boldsymbol{\theta}}(\bX), f_{\boldsymbol{\theta}}(\bY))\right]&=\underset{(\bx, \by)\sim \ppos \atop \{(\bx_j, \by_j)\}^{M-1}_{j=1}\overset{\text{i.i.d}}{\sim} \ppos}{\E}\Biggl[\psi \Biggl(\sum_{j=1}^{M-1} \phi\left(\left(\fun\left(\by_j\right)-\fun\left(\by\right)\right)^{\top} \fun\left(\bx\right)\right) \notag\\ 
& \qquad \qquad \qquad \qquad \qquad \qquad \qquad +\phi\left(\left(\fun\left(\bx_j\right)-\fun\left(\by\right)\right)^{\top} \fun\left(\bx\right)\right)\Biggl)\Biggl]\label{eq:emb-simclr-equality}\\
\underset{(\bX, \bY)\overset{\text{i.i.d}}{\sim} \ppos^M}{\E}\left[\Lgeninfonceb(f_{\boldsymbol{\theta}}(\bX), f_{\boldsymbol{\theta}}(\bY))\right]
&=\underset{(\bx, \by)\sim \ppos \atop \{\bx_j\}^{M-1}_{j=1}\overset{\text{i.i.d}}{\sim} \pdata}{\E}\left[\psi \left(\sum_{j=1}^{M-1} \phi\left(\left(f_{\boldsymbol{\theta}}\left(\bx_j\right)-f_{\boldsymbol{\theta}}\left(\by\right)\right)^{\top} f_{\boldsymbol{\theta}}\left(\bx\right)\right)\right)\right]\label{eq:emb-dhel-equality}
\end{align}

Expanding Eq. \eqref{eq:emb-nce-equality}, \eqref{eq:emb-simclr-equality} (twice) and \eqref{eq:emb-dhel-equality}  for $\Linfoncea, \Lsimclr, \Ldcl, \Lours$  respectively, we obtain:
\begin{equation}
    \begin{split}
   &\underset{(\bX, \bY)\overset{\text{i.i.d}}{\sim} \ppos^M}{\E}\left[\Linfoncea(\fun(\bX), f_{\boldsymbol{\theta}}(\bY))\right] =
   \underset{(\bu, \bv)\sim \pushfun\ppos}{\E}\left[-\bv^\top \bu/\tau\right] +
  \underset{(\bu, \bv)\sim \pushfun\ppos \atop \bV^{\prime}\overset{\text{i.i.d}}{\sim} \pushfun\pdata^{M-1}}{\E}\left[\log \left(e^{\bv^{\top} \bu/\tau} + \sum_{j=1}^{M-1} e^{\bu^{\top} \bv^{\prime}_j/\tau}\right)\right]\\
    &\underset{(\bX, \bY)\overset{\text{i.i.d}}{\sim} \ppos^M}{\E}\left[\Lsimclr(f_{\boldsymbol{\theta}}(\bX), f_{\boldsymbol{\theta}}(\bY))\right]
   =\underset{(\bu, \bv)\sim \pushfun\ppos}{\E}\left[-\bv^\top \bu/\tau\right] +
   \underset{(\bu, \bv) \sim \pushfun\ppos \atop \hat{\bU}\overset{\text{i.i.d}}{\sim} \pushfun\ppos^{M-1}}{\E}\left[\log \left(e^{\bv^{\top} \bu/\tau} + \sum\limits_{j=1}^{2M-2} e^{\hat{\bu}_j^{\top} \bu}\right)\right]\\
   &\underset{(\bX, \bY)\overset{\text{i.i.d}}{\sim} \ppos^M}{\E}\left[\Ldcl(f_{\boldsymbol{\theta}}(\bX), f_{\boldsymbol{\theta}}(\bY))\right]
   =\underset{(\bu, \bv)\sim \pushfun\ppos}{\E}\left[-\bv^\top \bu/\tau\right] +
   \underset{\bu \sim \pushfun\pdata \atop \hat{\bU}\overset{\text{i.i.d}}{\sim} \pushfun\ppos^{M-1}}{\E}\left[\log \left( \sum\limits_{j=1}^{2M-2} e^{\hat{\bu}_j^{\top} \bu/\tau}\right)\right]\\
&\underset{(\bX, \bY)\overset{\text{i.i.d}}{\sim} \ppos^M}{\E}\left[\Lours(f_{\boldsymbol{\theta}}(\bX), f_{\boldsymbol{\theta}}(\bY))\right]
   = \underset{(\bu, \bv)\sim \pushfun\ppos}{\E}\left[-\bv^\top \bu/\tau\right] +
    \underset{\bu \sim \pushfun\pdata \atop \bU^{\prime}\overset{\text{i.i.d}}{\sim} \pushfun\pdata^{M-1}}{\E}\left[\log \left(\sum_{j=1}^{M-1} e^{\bu^{\top} \bu^{\prime}_j/\tau}\right)\right]
\end{split}
\end{equation}
    
\end{proof}

\begin{proposition}
The expectations of the following batch-level contrastive loss functions: $\Linfoncea(\cdot, \cdot)$, $\Lours(\cdot, \cdot)$, $\Lsimclr(\cdot, \cdot)$, $\Ldcl(\cdot, \cdot)$ have the \textbf{same asymptotic behaviour} when normalised by appropriate normalising constants, \ie\ when $M \to \infty$ we have that:
\begin{equation}
\begin{split}
      \underset{M \to \infty}{\lim}\underset{(\bX, \bY)\overset{\text{i.i.d}}{\sim} \ppos^M}{\E}\left[\Linfoncea(f_{\boldsymbol{\theta}}(\bX), f_{\boldsymbol{\theta}}(\bY))\right] -\log(M-1) 
      &= \underset{M \to \infty}{\lim}\underset{(\bX, \bY)\overset{\text{i.i.d}}{\sim} \ppos^M}{\E}\left[\Lsimclr(f_{\boldsymbol{\theta}}(\bX), f_{\boldsymbol{\theta}}(\bY))\right]-\log(2M-2) \\
      &= \underset{M \to \infty}{\lim}\underset{(\bX, \bY)\overset{\text{i.i.d}}{\sim} \ppos^M}{\E}\left[\Ldcl(f_{\boldsymbol{\theta}}(\bX), f_{\boldsymbol{\theta}}(\bY))\right] - \log(2M-2) \\  
      &=\underset{(\bX, \bY)\overset{\text{i.i.d}}{\sim} \ppos^M}{\E}\left[\Lours(f_{\boldsymbol{\theta}}(\bX), f_{\boldsymbol{\theta}}(\bY))\right] - \log(M-1)\\  
   &= \underset{(\bx', \by')\sim \ppos}{\E}\left[-\fun\left(\by'\right)^\top \fun\left(\bx'\right)\right]\\ & \qquad \qquad \qquad \qquad \qquad+ \underset{\bx'\sim \pdata}{\E}\left[\log \underset{ \bx{\sim} \pdata}{\E}\left(e^{\fun\left(\bx\right)^\top \fun\left(\bx'\right)} 
   \right)\right] 
\end{split}
\end{equation}
\end{proposition}

\begin{proof}
For a fixed $\bx'$, since $\underset{M \to \infty}{\lim}\frac{1}{M-1}e^{\fun\left(\by'\right)^\top \fun\left(\bx'\right)} \underset{M \to \infty}{\lim}\frac{1}{2M-2}e^{\fun\left(\by'\right)^\top \fun\left(\bx'\right)} = 0$, because $e^{\fun\left(\by'\right)^\top \fun\left(\bx'\right)}$ is bounded and since
\begin{align*}
\underset{\bX\sim \pdata^{M-1}}{\E}\left[\frac{1}{M-1}\sum_{j=1}^{M-1} e^{\fun\left(\by_j\right)^{\top} \fun\left(\bx'\right)}\right] &=\underset{\bX\sim \pdata^{M-1}}{\E}\left[\frac{1}{M-1}\sum_{j=1}^{M-1} e^{\fun\left(\bx_j\right)^{\top} \fun\left(\bx'\right)}\right] \\
=  &\underset{\hat{\bX}\sim \ppos^{M-1}}{\E}\left[\frac{1}{2M-2}\sum_{j=1}^{2M-2} e^{\fun\left(\hat{\bx}_j\right)^{\top} \fun\left(\bx'\right)}\right] =  
\underset{\bx\sim \pdata}{\E}\left[ e^{\fun\left(\bx\right)^{\top} \fun\left(\bx'\right)}\right],
\end{align*}
then with probability 1 over the respective sample spaces, due to the strong law of large numbers, it holds that:
\begin{align*}
\underset{M \to \infty}{\lim}\frac{1}{M}e^{\fun\left(\by'\right)^\top \fun\left(\bx'\right)} + \frac{1}{M}\sum_{j=1}^{M-1} e^{\fun\left(\by_j\right)^{\top} \fun\left(\bx'\right)}&=
\underset{M \to \infty}{\lim}\frac{1}{2M-2}e^{\fun\left(\by'\right)^\top \fun\left(\bx'\right)} + \frac{1}{2M-2}\sum_{j=1}^{2M-2} e^{\fun\left(\hat{\bx}_j\right)^{\top} \fun\left(\bx'\right)} \\
&=
\underset{M \to \infty}{\lim} \frac{1}{2M-2}\sum_{j=1}^{2M-2} e^{\fun\left(\hat{\bx}_j\right)^{\top} \fun\left(\bx'\right)}\\
&=
\underset{M \to \infty}{\lim} \frac{1}{M-1}\sum_{j=1}^{M-1} e^{\fun\left(\bx_j\right)^{\top} \fun\left(\bx'\right)} =
\underset{\bx\sim \pdata}{\E}\left[ e^{\fun\left(\bx\right)^{\top} \fun\left(\bx'\right)}\right]. 
\end{align*}
Now the desideratum follows directly using the same steps as in the proof of Theorem 1 in \cite{wang2020understanding}. Briefly,  the same limit holds for the $\log$ (continuous function) of the above quantities due to the Continuous Mapping Theorem, and therefore when taking the limit of each loss variant (after first subtracting the right normalisation constant $M-1$ or $2M-2$), since the quantities inside the expectation are bounded, we can invoke the Dominated Convergence Theorem and switch the limit with the expectation, thus arriving at the desideratum.
\end{proof}

\subsection{Minima of Mini-Batch Kernel Contrastive Losses}\label{sec:mb_kcl_proof}

\begin{theorem}\label{thrm:mb_kernels_theorem_proof}
Consider the following optimisation problem:
\begin{equation}\label{eq: mb_optimisation_kernels}
    \underset{\bU, \bV \in (\mathbb{S}^{d-1})^M}{\argmin} \Lgenkertotal(\bU, \bV),
\end{equation}
with $\Lgenkertotal(\bU, \bV) = \frac{1}{2} \left(\Lgenker(\bU, \bV) + \Lgenker(\bV, \bU) \right)$,
where $\mathbb{S}^{d-1} = \{\bu \in \mathbb{R}^{d}\  | \ \|\bu\|_2=1\}$ a unit sphere of $d$ dimensions, $\bU, \bV$ are tuples of $M$ vectors on the unit sphere and $\Lgenker(\cdot, \cdot)$ is a kernel loss function of the form:
\begin{equation}
    \Lgenker(\bU, \bV) = -\frac{1}{M}\sum_{i=1}^M K_A(\bu_i, \bv_i) + \gamma \frac{1}{M(M-1)}\sum_{i,j=1 \atop i \neq j}^M K_U(\bu_i, \bu_j),
\end{equation}
 where ${K_A(\bx, \by) = \kappa_A(\|\bx - \by\|^2)}$ and ${K_U(\bx, \by) = \kappa_U(\|\bx - \by\|^2)}$ with $\kappa_A, \kappa_U: (0,4] \to \mathbb{R}$, the limits $\underset{x \to 0^+}{\lim}\kappa_A(x)$,  $\underset{x \to 0^+}{\lim}\kappa_U(x)$ exist and are bounded in both cases and $\gamma>0$. 
Further, for case (a) $1<M \leq d + 1$, suppose the following conditions: (1) $k_A$ is \textbf{decreasing} and (2) $k_U$ is \textbf{decreasing and convex}.
Then, the optimisation problem of Eq. \eqref{eq: mb_optimisation} obtains its optimal value $(\bU^*, \bV^*)$ when:
\begin{equation}\label{eq:mb_minima_kernels}
\begin{split}
    \bU^* = \bV^* \quad \text{ and } \quad \bU^*= [\bu^*_1, \dots, \bu^*_M] \  \text{ form a regular $M-1$ simplex centered at the origin.}
\end{split}
\end{equation}
Additionally, (3) if $\kappa_A$ is \textbf{strictly decreasing} and $\kappa_U$ is  \textbf{strictly decreasing and strictly convex} then all the $(\bU^*, \bV^*)$ that satisfy Eq. \eqref{eq:mb_minima} are the \textbf{unique} optima. For case (b) $M = 2d$, suppose again that $k_A$ is \textbf{decreasing} and that (4) $k_U$ is \textbf{completely monotone}. Then Eq. \eqref{eq: mb_optimisation} obtains its optimal value when:
\begin{equation}\label{eq:mb_minima_kernels_cp}
\begin{split}
    \bU^* = \bV^* \quad \text{ and } \quad \bU^*= [\bu^*_1, \dots, \bu^*_M] \  \text{ form a cross-polytope}.
\end{split}
\end{equation}
\end{theorem}

\begin{proof}
Our strategy in this proof will be to analyse the two terms independently and show that it is possible to simultaneously attain their minima, \ie\ for the same input arguments. Let us start with the first term.

\noindent\textbf{Part I: Alignment term $-\frac{1}{M}\sum_{i=1}^M K_A(\bu_i, \bv_i)$.}

Observe that in this term, the summands are independent of each other, so we can optimise each of them independently. Let $(\bu_i^*, \bv_i^*)$ be a minimiser for $-K_A(\bu_i, \bv_i) = - \kappa_A(\|\bu_i - \bv_i\|^2)$. Therefore,
\begin{align*}
    - \kappa_A(\|\bu_i - \bv_i\|^2) \geq - \kappa_A(\|\bu_i^* - \bv_i^*\|^2) \Leftrightarrow
    \kappa_A(\|\bu_i - \bv_i\|^2) {\leq} \kappa_A(\|\bu_i^* - \bv_i^*\|^2)
    \stackrel{(a)}{\Leftrightarrow}
    \|\bu_i - \bv_i\|^2 \geq \|\bu_i^* - \bv_i^*\|^2,
\end{align*}
where \textit{(a) follows from the fact that $\kappa_A$ is decreasing}, and thus:
\begin{equation*}
        \underset{\bu_i, \bv_i \in \mathbb{S}^{d-1}}{\argmin}-K_A(\bu_i, \bv_i) 
    \supseteq
  \underset{\bu_i, \bv_i \in \mathbb{S}^{d-1}}{\argmin}\|\bu_i - \bv_i\|^2 =\{(\bu_i^*, \bv_i^*) \mid \bu^*_i = \bv^*_i\}.
\end{equation*}
So, since the above holds $\forall i \in \{1,\dots, M\}$ we obtain  
$-K_A(\bu_i, \bv_i) \geq -K_A(\bu_i^*, \bv_i^*) \Leftrightarrow -\frac{1}{M}\sum_{i=1}^M K_A(\bu_i, \bv_i) \geq -\frac{1}{M}\sum_{i=1}^M K_A(\bu_i^*,  \bv_i^*)$, $\forall \, \bU, \bV \in (\mathbb{S}^{d-1})^M$ and thus:
\begin{equation*}
\underset{\bU, \bV \in (\mathbb{S}^{d-1})^M}{\argmin} -\frac{1}{M}\sum_{i=1}^M K_A(\bu_i, \bv_i)\supseteq
  \{(\bU^*, \bV^*) \mid \bu^*_i = \bv^*_i, \forall i \in \{1,\dots, M\} \}.
\end{equation*}

\noindent\textbf{Part II: Uniformity term $ \frac{\gamma}{M(M-1)}\sum_{i,j=1 \atop i \neq j}^M K_U(\bu_i, \bu_j)$.}

First, observe that optimising the second term depends only on $\bU$. Further, note that finding its minimiser is a classical hyperspherical energy minimisation problem, which is straightforward when $\kappa_U$ is convex and decreasing and $1<M\leq d+1$. We can invoke Theorem 1 from \cite{LiuYWS23}, which asserts that if the aforementioned conditions hold, then:
\begin{equation*}
\underset{\bU, \bV \in (\mathbb{S}^{d-1})^M}{\argmin}\frac{\gamma}{M(M-1)}\sum_{i,j=1 \atop i \neq j}^M K_U(\bu_i, \bu_j) \supseteq
  \{(\bU^*, \bV^*) \mid \bU^* \text{: regular $M-1$ simplex on $\mathbb{S}^{d-1}$ centered at the origin}\}.
\end{equation*}

Similarly, for the case $M=2d$, we can invoke Theorem 5.7.2 \cite{borodachov2019discrete} or Theorem 2 \cite{LiuYWS23}, which imply that:
\begin{equation*}
\underset{\bU, \bV \in (\mathbb{S}^{d-1})^M}{\argmin}\frac{\gamma}{M(M-1)}\sum_{i,j=1 \atop i \neq j}^M K_U(\bu_i, \bu_j) \supseteq
  \{(\bU^*, \bV^*) \mid \bU^* \text{: cross-polytope}\},
\end{equation*}
when the function $\tilde{k}_U(x) = k_U(2-2x)$, \ie\ the corresponding function expressing the kernel w.r.t. the inner product, is\textit{ absolutely monotone} in $[-1,1)$. But \cite{borodachov2019discrete} show that this equivalent to $k_U$ being completely monotone in $(0,4]$ (Condition (4)).

In the intersection of the two sets of minimisers, both terms will be minimised, \ie\ $-\frac{1}{M}\sum_{i=1}^M K_A(\bu_i, \bv_i) \geq -\frac{1}{M}\sum_{i=1}^M K_A(\bu^*_i, \bv^*_i)$ and  $\frac{\gamma}{M(M-1)}\sum_{i,j=1 \atop i \neq j}^M K_U(\bu_i, \bu_j) \geq \frac{\gamma}{M(M-1)}\sum_{i,j=1 \atop i \neq j}^M K_U(\bu^*_i, \bu^*_j)$,  $\forall \, \bU, \bV \in (\mathbb{S}^{d-1})^M$. Therefore, the intersection is a minimiser of the total objective:
\begin{equation*}
 \{(\bU^*, \bV^*) \mid \bU^*=\bV^* \text{: regular $M-1$ simplex on $\mathbb{S}^{d-1}$ centered at the origin}\} \subseteq \underset{\bU, \bV \in (\mathbb{S}^{d-1})^M}{\argmin} \Lgenker(\bU, \bV),   
\end{equation*}
and similarly for the cross-polytope.
It is easy to see that the same set will be a minimiser of $\Lgenker(\bV, \bU)$ and thus it is also a minimiser of $\Lgenkertotal(\bV, \bU)$.

Finally, if $\kappa_A$ is strictly decreasing, then (a) holds with equality only when $\bu_i = \bv_i$, while if $\kappa_U$ is strictly convex and strictly decreasing then, again by Theorem 1 in \cite{liu2022generalizing}, we know that the regular $M-1$ simplex is the only minimiser of the second term, and thus Eq. \eqref{eq:mb_minima} is the unique minimiser of the kernel contrastive loss for $1<M \leq d+1$.

\end{proof}

\subsection{Expected (True) Kernel Contrastive Losses}

\begin{proposition}
The expectation of the mini-batch kernel contrastive loss functions $\Lgenker(\cdot, \cdot)$ is \textbf{independent of the size of the batch} and therefore equal to the asymptotic expected loss. In other words, mini-batch kernel loss is an \textbf{unbiased estimator} of the asymptotic expected loss and in particular, we have that:
\begin{equation}\label{eq:kernel_expected_loss_proof}
\underset{(\bX, \bY) \sim \ppos^M}{\E} \left[\Lgenkertotal\left(\fun\left(\bX\right), \fun\left(\bY\right)\right)\right] =  -\underset{(\bx, \by) \sim \ppos}{\E} \left[K_A\left(\fun\left(\bx\right), \fun\left(\by\right)\right)\right] + \gamma\underset{\bx,  \sim \pdata \atop \bx' \sim \pdata}{\E} \left[K_U\left(\fun\left(\bx\right), \fun\left(\bx'\right)\right)\right].
\end{equation}
If (1) $\kappa_A$ is (strictly) decreasing and if (2) $\exists \, \boldsymbol{\theta}^*$ such that $\mathbb{P}_{(\bx, \by)\sim\ppos}\left[\fun(\bx) = \fun(\by)\right] = 1$, then the set of $\boldsymbol{\theta}^*$ for which (2) holds are (unique) minimisers of the first term of Eq. \eqref{eq:kernel_expected_loss}. Additionally, if (3) -$\kappa^{(1)}_U$ (first derivative) is \textbf{strictly completely monotone} in $(0,4]$, (4) the expectation defined in the l.h.s. of Eq. \eqref{eq:kernel_expected_loss} is finite and (5) $\exists \, \boldsymbol{\theta}^*$ such that the pushforward measure $\pushfun\pdata = U(\mathbb{S}^{d-1})$, then $\boldsymbol{\theta}^*$ is a unique minimiser of the second term of Eq. \eqref{eq:kernel_expected_loss}. Finally, if (6) $\exists \, \boldsymbol{\theta}^*$ such that conditions (2) and (3) can be satisfied simultaneously, then $\boldsymbol{\theta}^*$ is a unique minimiser of Eq. \eqref{eq:kernel_expected_loss}.
\end{proposition}
\begin{proof}
The first part of the proposition is obvious since:
\begin{align*}
\underset{(\bX, \bY) \sim \ppos^M}{\E} \left[\Lgenker\left(\fun\left(\bX\right), \fun\left(\bY\right)\right)\right] 
&=\underset{(\bX, \bY) \sim \ppos^M}{\E} \Biggl[-\frac{1}{M}\sum_{i=1}^M K_A(\fun(\bx_i), \fun(\by_i))\\ 
& \qquad \qquad \qquad \qquad \qquad \qquad \qquad + \frac{\gamma}{M(M-1)}\sum_{i,j=1 \atop i \neq j}^M K_U(\fun(\bx_i), \fun(\by_j))\Biggl]\\
& =  \underset{(\bX, \bY) \sim \ppos^M}{\E}\left[-\frac{1}{M}\sum_{i=1}^M K_A(\fun(\bx_i), \fun(\by_i))\right]\\
& \qquad \qquad \qquad \qquad \qquad + \underset{(\bX, \bY) \sim \ppos^M}{\E} \left[\frac{\gamma}{M(M-1)}\sum_{i,j=1 \atop i \neq j}^M K_U(\fun(\bx_i), \fun(\by_j))\right]\\
& =  -\frac{1}{M} \sum_{i=1}^M \underset{(\bx_i, \by_i) \sim \ppos}{\E}\left[ K_A(\fun(\bx), \fun(\by))\right]\\
& \qquad \qquad \qquad \qquad \qquad + \frac{\gamma}{M(M-1)}\sum_{i,j=1 \atop i \neq j}^M\underset{\bx_i \sim \pdata \atop \bx_j \sim \pdata}{\E} \left[K_U(\fun(\bx_i), \fun(\bx_j))\right]\\
& =  - \underset{(\bx, \by) \sim \ppos}{\E}\left[ K_A(\fun(\bx), \fun(\by))\right] + \gamma\underset{\bx \sim \pdata \atop \bx' \sim \pdata}{\E} \left[K_U(\fun(\bx), \fun(\bx'))\right]\\
\end{align*}
Regarding the minimiser of the alignment term, we have that for every $(\bx, \by)$:
\begin{align*}
\|\fun(\bx) - \fun(\by)\|^2 \geq 0 = \Leftrightarrow
-\kappa_A\left(\|\fun(\bx) - \fun(\by)\|^2\right) \geq  
-\kappa_A(0),
\end{align*}
 since $\kappa_A$ is decreasing.
It follows that $-\kappa_A(0) \leq  \underset{(\bx, \by) \sim \ppos}{\E}\left[K_A\left(\fun\left(\bx\right), \fun\left(\by\right)\right)\right]$. But, we know that $\|\funstar(\bx) - \funstar(\by)\|^2 = 0$ almost surely and thus also $K_A\left(\funstar\left(\bx\right), \funstar\left(\by\right)\right) = \kappa_A\left(\|\funstar(\bx) - \funstar(\by)\|^2\right) = \kappa_A(0)$ almost surely. Therefore, $ -\underset{(\bx, \by) \sim \ppos}{\E}\left[K_A\left(\funstar\left(\bx\right), \funstar\left(\by\right)\right)\right] \leq  \underset{(\bx, \by) \sim \ppos}{\E}\left[K_A\left(\fun\left(\bx\right), \fun\left(\by\right)\right)\right]$, which proves that $\boldsymbol{\theta}^*$ is a minimiser. If $\kappa_A$ is strictly decreasing (Condition (1)), then equality holds only if the set $(\bx, \by)$ for which $\|\fun(\bx) - \fun(\by)\|^2 > 0$ has measure zero under $\ppos$, \ie\ for $\boldsymbol{\theta} = \boldsymbol{\theta}^*$ (unique minimiser).

Regarding the minimiser of the uniformity term, we can invoke Theorem 6.2.1 from \cite{borodachov2019discrete} (restated as Lemma 2 in \cite{wang2020understanding}), that under conditions (3) and (4) asserts that the unique measure $\mu \in \mathcal{M}(\mathbb{S}^{d-1})$ minimising $\underset{\bu  \sim \mu \atop \bv \sim \mu}{\E} \left[K_U\left(\bu, \bv\right)\right]$ is the uniform measure on the sphere $U(\mathbb{S}^{d-1})$. Then, condition (5) simply guarantees the existence of a parameter $\boldsymbol{\theta}$ so that this measure is attainable. 
\end{proof}

\section{Experimental Details}\label{sec:app_experiments}
In the following section, we provide a detailed description of the experimental setup.

\subsection{Detailed Sampling Process.}\label{sec:sampling}
The distributions of interest $\pdata, \ppos$ throughout all the experiments are formed from the following two-step sampling process. First, we sample a datapoint $\bx_{\text{init}} \in \cX$ from an (unknown) \textit{initial distribution} $p_{\text{init}}$ on $\cX$ (\ie\ the one from which we sample the datapoints in our dataset) and subsequently we independently sample a \textit{transformation operator} ${T: \cX \to \cX}$ from a (usually known) distribution $p_{T}$ on a space of available transformations $\cT$. Then, $\pdata$ is the distribution of the datapoint $T(x_{\text{init}})$ and the p.d.f. is given by $\pdata(\bx) = \int_{T \in \cT}p_{\text{init}}(x) p_T(T) dT$.\footnote{The transformations are usually parameterised by parameters residing in a measurable space. Here we slightly abuse notation and the integration over $T$ implies an integration over the transformation parameters.} Additionally, we sample positive pairs by first sampling a datapoint $\bx_{\text{init}}$ and then transforming it by two independently sampled operators $T_1, T_2$. We define the distribution of positive pairs as the distribution of the tuples $\left(T_1\left(\bx_{\text{init}}\right), T_2\left(\bx_{\text{init}}\right)\right)$ and the p.d.f. is given by $\ppos(\bx, \by) = \int_{T_1, T_2 \in \cT, x_{\text{init}}\in \cX \atop y = T_2(x_{\text{init}}), x = T_1(x_{\text{init}})}p_{\text{init}}(x_{\text{init}}) p_T(T_1)p_T(T_2) dx_{\text{init}}dT_1 dT_2 $. The transformation operators encode the symmetries of the data, \ie\ it is expected that the downstream tasks will be invariant to them.

In practice, batches of data are sampled from a fixed finite dataset of $N>M$ samples $\cD \sim p_{\text{init}}^N$ as follows: $M$ samples are obtained by sampling uniformly at random from the dataset, \ie\ $(\bx_{\text{init},i})_{i=1}^M \sim U(\cD) = \tilde{p}_\text{init}$ and $2M$ transformations $T_{1,i}, T_{2,i}$ are independently sampled from $p_T$, resulting in a batch of $M$ positive pairs $\big((\bx_i, \by_i)\big)_{i=1}^M \sim \ppos^M $, where $\bx_i = T_{1,i}(\bx_{\text{init},i})$, $\by_i = T_{2,i}(\bx_{\text{init},i})$. 

\subsection{Performance Metrics.} \label{sec:metrics}
Below we provide more details on the metrics used in \Cref{fig:properties_temperature}, \Cref{sec:ablations} of the main paper.
\begin{itemize}
    \item \textbf{Alignment}. It estimates the expected L2 distance between a pair of positive samples:
    \begin{equation}
        L_{\text{alignment}}(f_\#\ppos) = \underset{(\bu, \bv) \sim f_\#\ppos}{\E}\left[\|\bu -\bv\|_2^2\right] \approx
        \frac{1}{M}{\sum_{i=1}^M}\|f(\bx_i) -f(\by_i)\|_2^2 = \hat{L}_{\text{alignment}}(f, \bX, \bY)
    \end{equation}
    \item \textbf{Uniformity}. The logarithm of an estimation of the expected pairwise Gaussian potential as in \cite{wang2020understanding}:
    \begin{equation}
    \begin{split}
          E_{\text{uniformity}}(f_\#\pdata;t) &=  \underset{\bu \sim f_\#\pdata \atop \bu' \sim f_\#\pdata}{\E}\left[e^{-t\|\bu -\bu'\|_2^2}\right] \approx
        \frac{1}{M(M-1)}{\sum_{i,j=1 \atop j \neq i }^M}e^{-t\|f(\bx_i) -f(\bx_j)\|_2^2} = \hat{E}_{\text{uniformity}}(f, \bX;t)  \\
        L_{\text{uniformity}}(f_\#\pdata;t) &= \log E_{\text{uniformity}}(f_\#\pdata;t) \approx \hat{L}_{\text{uniformity}}(f, \bX;t) = \log \hat{E}_{\text{uniformity}}(f, \bX;t), 
    \end{split}
    \end{equation}
where the last approximation holds for large $M$ (strong law of 
large numbers and continuous mapping theorem) and $t=2$ as in \cite{wang2020understanding}.
    \item \textbf{Wasserstein distance between similarity distributions}. 
    Our novel metric estimates the \textit{1-Wasserstein distance} $W_1(\qsim, \psim)$,  where
    $\psim$ is the p.d.f of the inner products when $\bu, \bu' \sim U(\mathbb{S}^{d-1})$ and $\qsim$ is the corresponding one when  $\bu, \bu' \sim f_\#\pdata$.
    According to \citet{cho2009inner}, $\psim$ is equal to
$\frac{\Gamma(\frac{d}{2})}{\Gamma(\frac{d-1}{2}) \sqrt{\pi}}\left(\sqrt{1-s^{2}}\right)^{d-3}$ for $s \in (-1,1)$ and $0$ elsewhere.
We chose 1-Wasserstein distance because it can be easily calculated by the following formula
\begin{equation}
    W_1(\qsim, \psim) = \int_{-\infty}^{+\infty}|F_{\qsim}(s) - F_{\psim}(s)|ds,
\end{equation} where $F_{\psim}, F_{\qsim}$ are the c.d.fs corresponding to $\psim, \qsim$ respectively. Therefore, one can estimate the latter from samples and approximate the integral numerically. In our implementation, we use the method
\texttt{scipy.stats.wasserstein\_distance} from the SciPy library \cite{2020SciPy}, which implements precisely the aforementioned process.

To understand the connection between this metric and the uniformity metric used in \cite{wang2020understanding}, we will use an equivalent definition of 1-Wasserstein, using the Kantorovich-Rubinstein dual (see \cite{peyre2019computational} for more details): $W_1(\qsim, \psim) = \frac{1}{K}\underset{\|g\|_{\text{Lip}}\leq K}{\sup} \underset{s \sim \qsim}{\E}\left[g\left(s\right)\right] - \underset{s \sim \psim}{\E}\left[g\left(s\right)\right]$, where $g$ continuous, $g: \cS \to \mathbb{R}$ and $\|g\|_{\text{Lip}}\leq K$ means that the Lipschitz constant of $g$ is at most $K$. The domain $\cS$ in our case is the interval $[-1,1]$ Now revisiting the definition of the uniformity metric (ignoring the logarithm) we get:

\begin{align*}
    E_{\text{uniformity}}(f_\#\pdata;t) =
\underset{\bu \sim f_\#\pdata \atop \bu' \sim f_\#\pdata}{\E}\left[e^{-t(2 -2\bu^\top \bu')}\right] 
= 
\underset{\bu \sim f_\#\pdata \atop \bu' \sim f_\#\pdata}{\E}\left[e^{-2t + 2t\bu^\top \bu'}\right] 
=
\underset{s \sim \qsim}{\E}\left[e^{-2t + 2ts}\right] 
=
\underset{s \sim \qsim}{\E}\left[g_1(s;t)\right], 
\end{align*}
where $g_1(s;t) = e^{-2t + 2ts}$ and its Lipschitz constant in $[-1,1]$ is at most equal to the maximum value of its derivative, \ie\ $\text{Lip}(g_1(\cdot;t)) = 2te^{-2t + 2t} = 2t$. Therefore:
\begin{align}
E_{\text{uniformity}}(f_\#\pdata;t) - E_{\text{uniformity}}(U(\mathbb{S}^{d-1});t) 
&= 
\underset{s \sim \qsim}{\E}\left[g_1(s;t)\right]  - \underset{s \sim \psim}{\E}\left[g_1(s;t)\right] \notag\\&\leq
\underset{\|g\|_{\text{L}}\leq 2t}{\sup} \underset{s \sim \qsim}{\E}\left[g\left(s\right)\right] - \underset{s \sim \qsim}{\E}\left[g\left(s\right)\right]
=2t W_1(\qsim, \psim) \Leftrightarrow \notag\\
L_{\text{uniformity}}(f_\#\pdata;t)
& \leq
\log\left(2t W_1(\qsim, \psim) + E_{\text{uniformity}}\left(U\left(\mathbb{S}^{d-1}\right);t\right)\right).
\end{align}
Given that $E_{\text{uniformity}}\left(U\left(\mathbb{S}^{d-1}\right);t\right)$ is fixed for a given $t$, the above
implies that $L_{\text{uniformity}}$ underestimates the closeness of $f_\# \pdata$ to a uniform distribution.
 \item \textbf{Rank}. The rank of a given matrix of representations $\text{rank}(\bU) \leq \min(M,d)$, where $\bU \in \mathbb{R}^{M \times d}$. This gives a measurement of the dimensions that are utilised. To account for numerical errors it is computed as follows:
 \begin{equation}
     \widehat{\text{rank}}(\bU) = |\{\sigma_i(\bU) > \epsilon \mid \sigma_i(\bU): i-\text{th singular value of } \bU\}|,
 \end{equation}
 where $\epsilon$ was chosen to $1e-5$.
 \item \textbf{Effective rank}. A smooth approximation of the rank \cite{roy2007effective}, that is less prone to numerical errors and has been found in practice to correlate well with downstream performance \cite{garrido2023rankme}. It is equal to the entropy of the normalised singular values:
  \begin{equation}
  \begin{split}
   \hat{\sigma}_i(\bU)& = \frac{\sigma_i(\bU)}{\sum_{i=1}^{\min(M,d)}|\sigma_i(\bU)|} + \epsilon\\
{\text{eff-rank}}(\bU)& =  -\sum_{i=1}^{\min(M,d)} 
 \hat{\sigma}_i(\bU)\log\hat{\sigma}_i(\bU),    
  \end{split}
 \end{equation}
 where we chose $\epsilon = 1e-7$ as in \cite{garrido2023rankme}.
 
\end{itemize}

\begin{figure*}[!b]
    \centering
   \begin{minipage}{0.24\textwidth}
    \resizebox{\textwidth}{!}{\includegraphics{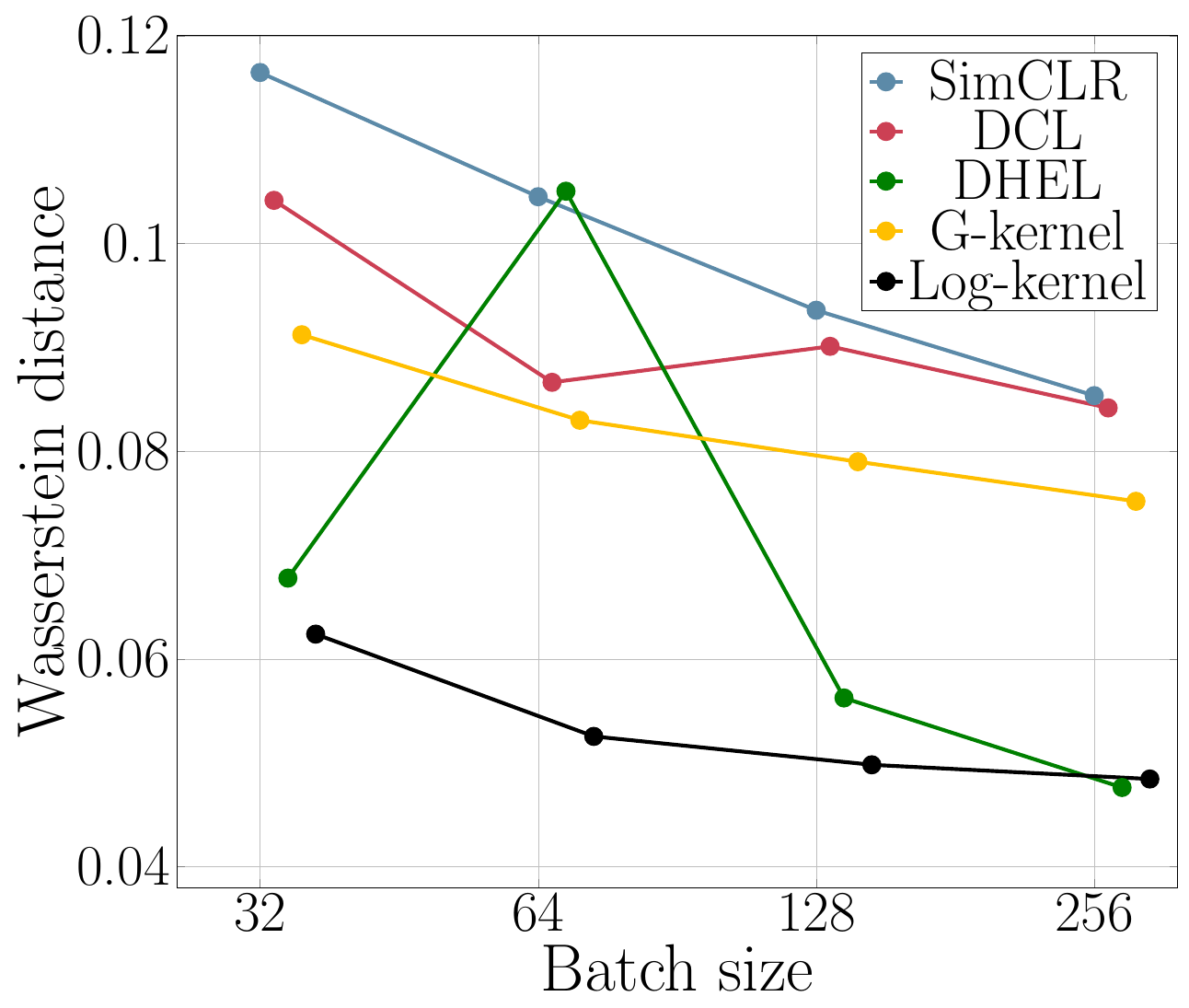}}
    \end{minipage}
    \begin{minipage}{0.24\textwidth}
    \resizebox{\textwidth}{!}{\includegraphics{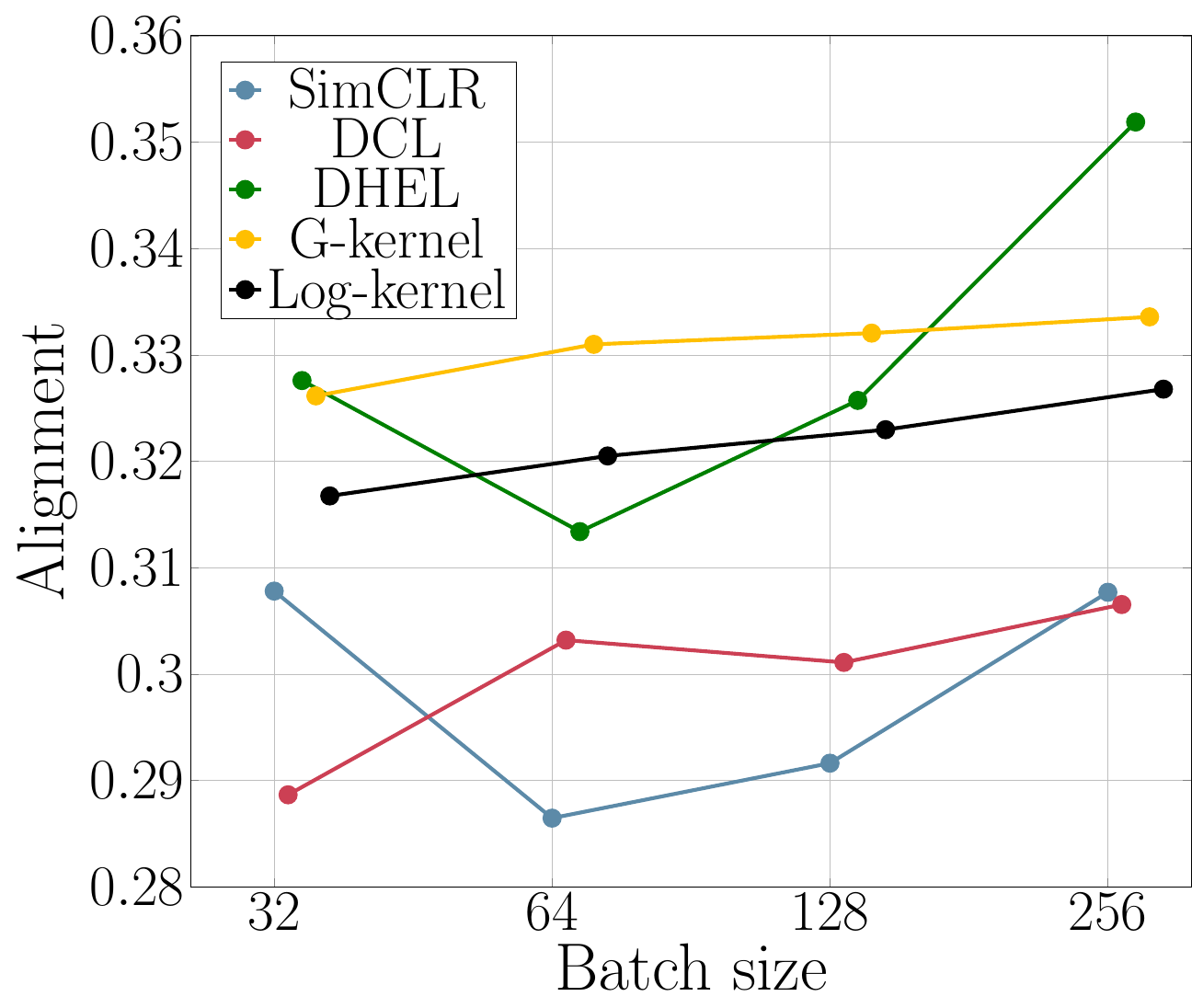}}
    \end{minipage}
    \begin{minipage}{0.24\textwidth}
    \resizebox{\textwidth}{!}{\includegraphics{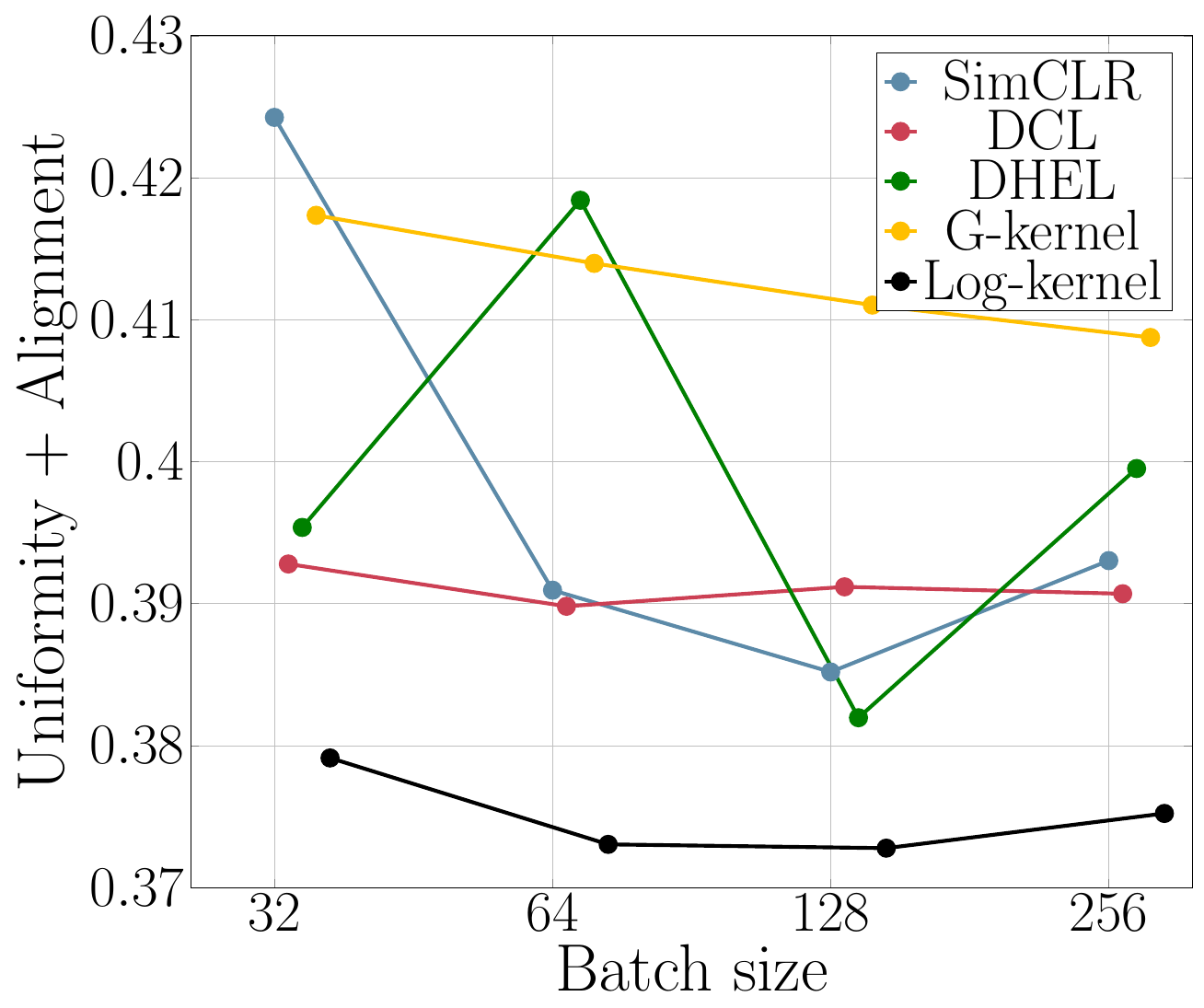}}
    \end{minipage}
    \begin{minipage}{0.24\textwidth}
    \resizebox{\textwidth}{!}{\includegraphics{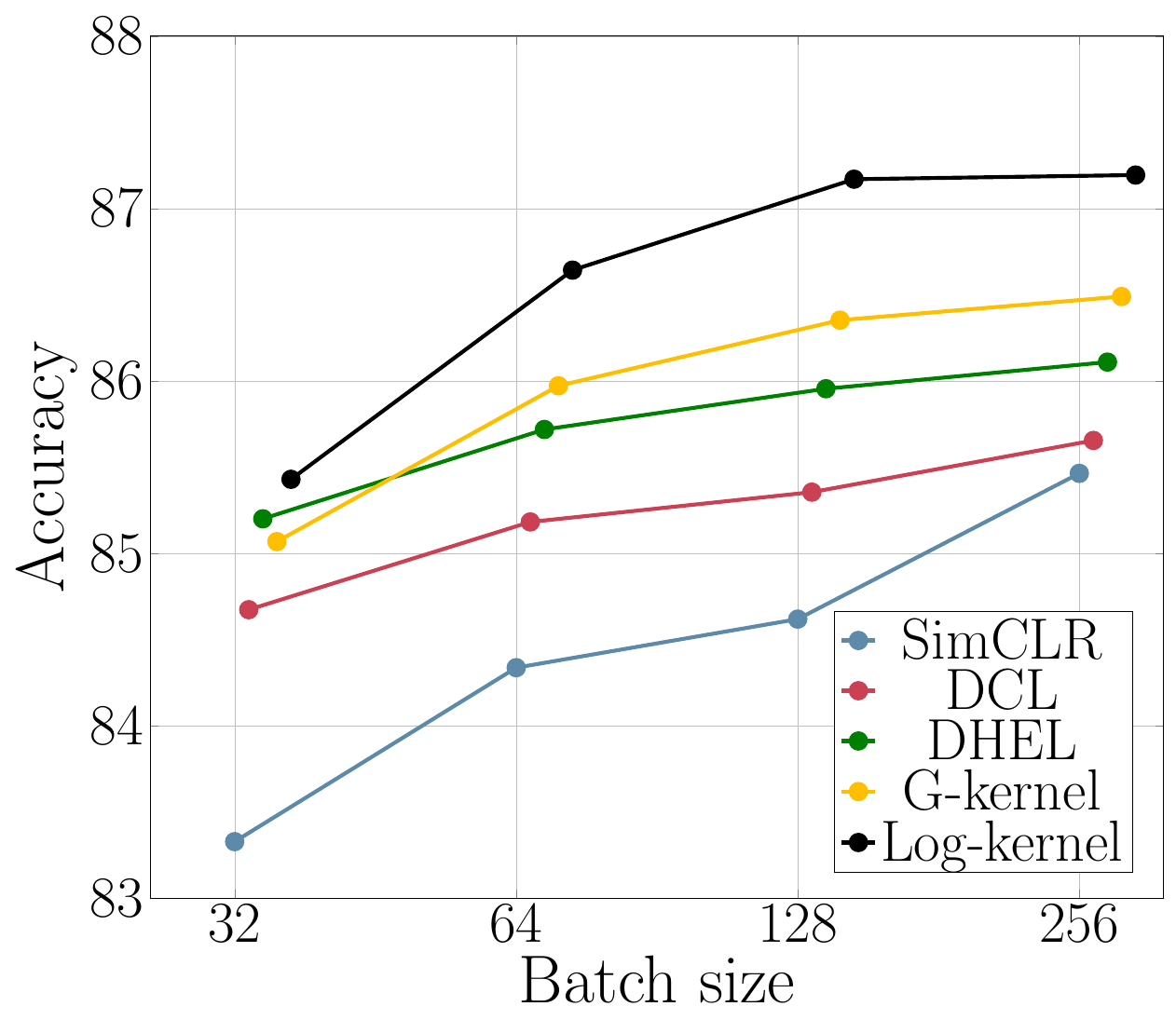}} 
    \end{minipage}
    \begin{minipage}{0.24\textwidth}
    \resizebox{\textwidth}{!}{\includegraphics{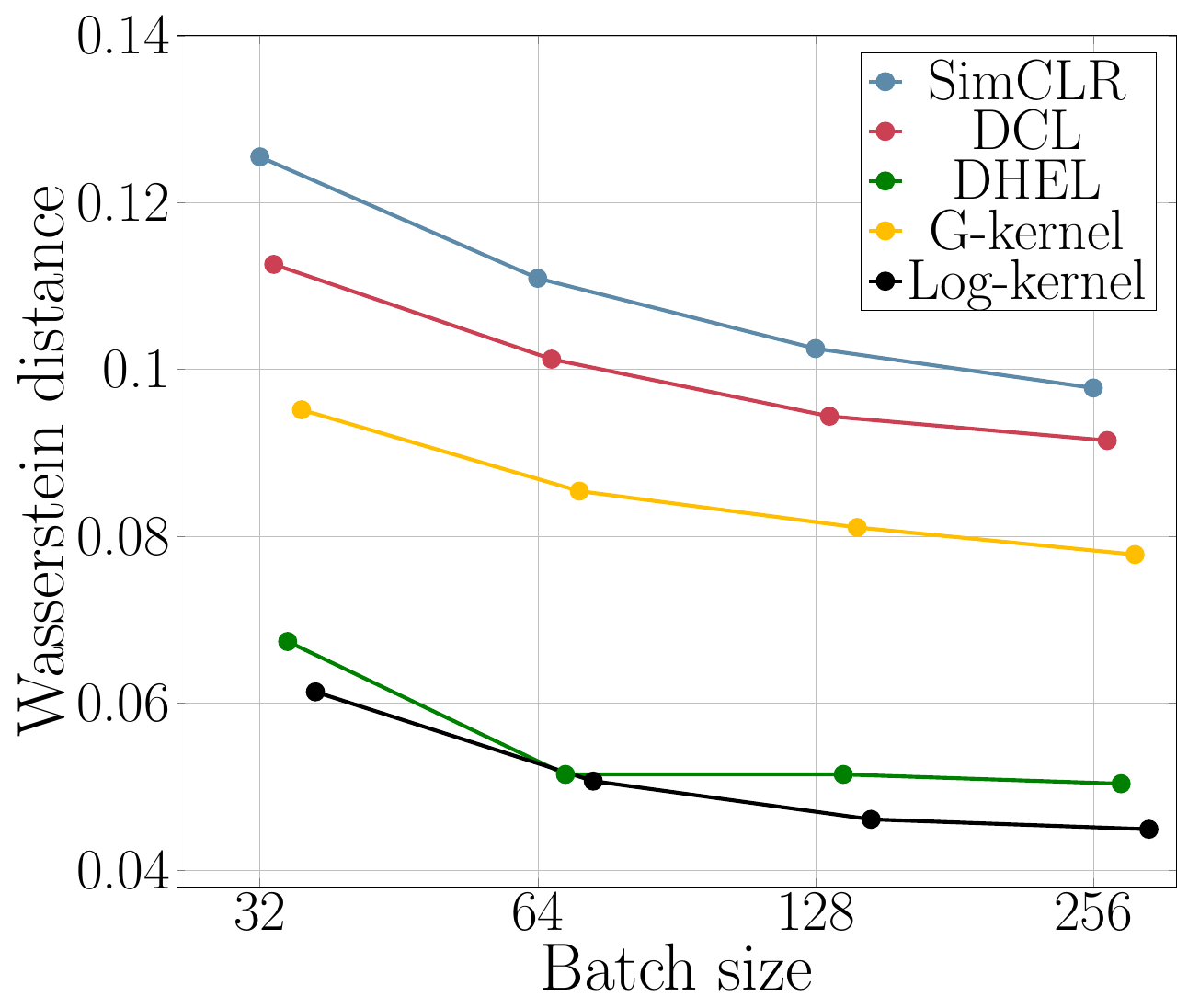}}
    \caption*{(a) Alignment}
    \end{minipage}
    \begin{minipage}{0.24\textwidth}
    \resizebox{\textwidth}{!}{\includegraphics{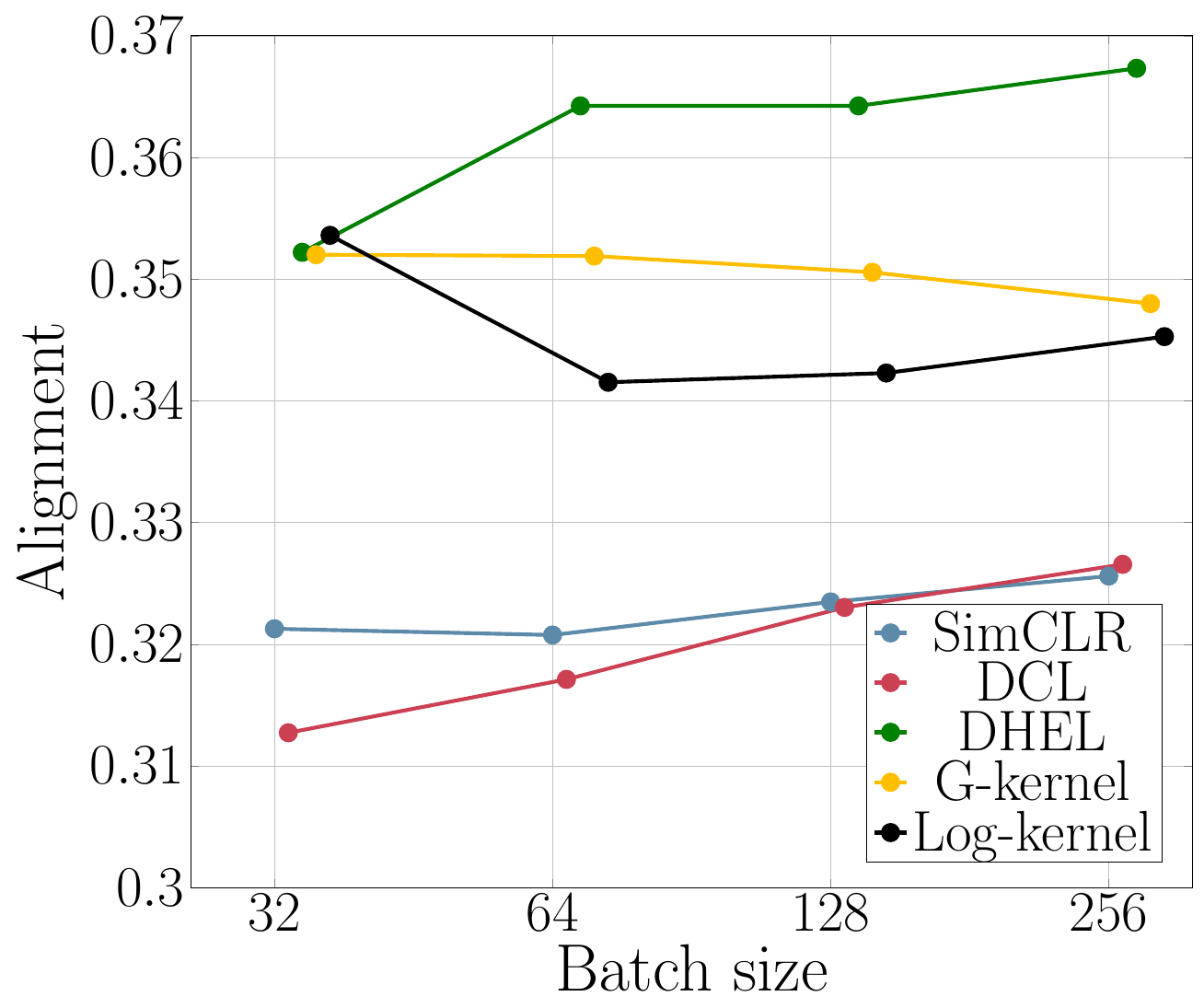}}
    \caption*{(b) Wasserstein distance}
    \end{minipage}
    \begin{minipage}{0.24\textwidth}
    \resizebox{\textwidth}{!}{\includegraphics{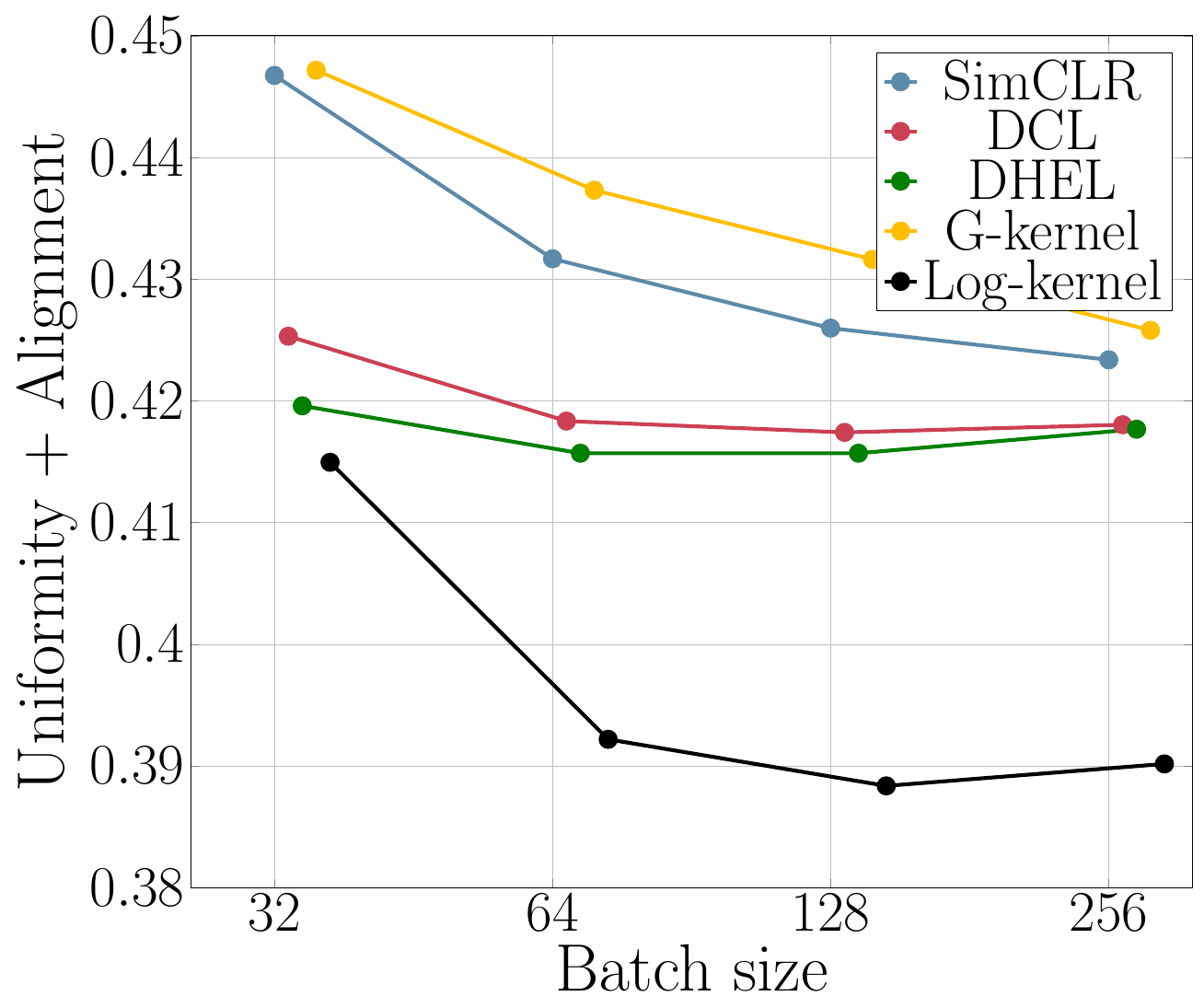}}
    \caption*{(c) Alignment + Wasserstein}
    \end{minipage}
    \begin{minipage}{0.24\textwidth}
    \resizebox{\textwidth}{!}{\includegraphics{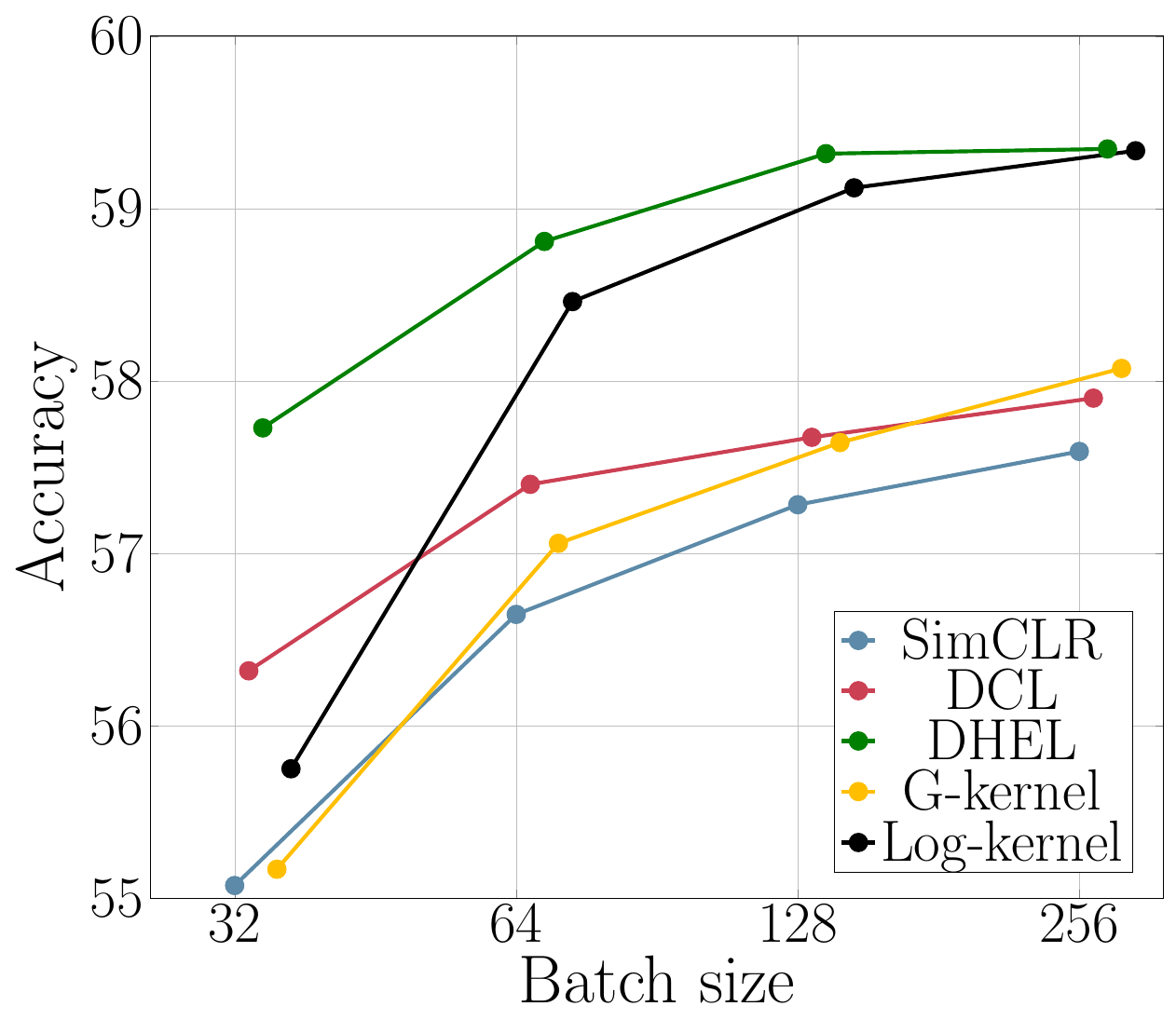}}
    \caption*{(d) Performance}
    \end{minipage}
    \caption{Mean value of properties vs batch size calculated on CIFAR10 (top) \& CIFAR100 (bottom) datasets.}\label{fig:properties_batch}
    \centering
   \begin{minipage}{0.24\textwidth}
    \resizebox{\textwidth}{!}{\includegraphics{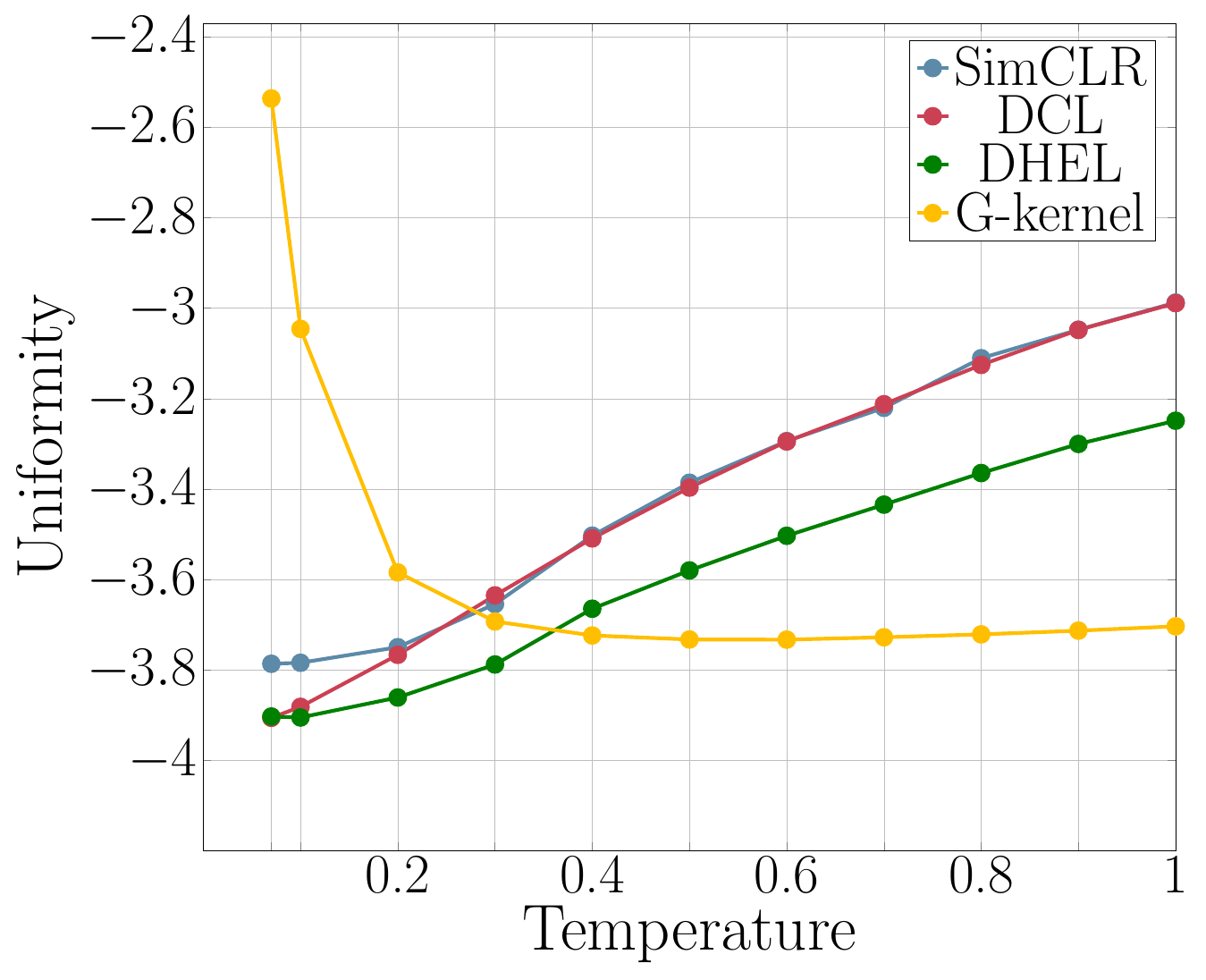}}
    \end{minipage}
    \begin{minipage}{0.24\textwidth}
    \resizebox{\textwidth}{!}{\includegraphics{figures/temperature_properties/cifar10/wasserstein_cifar10.pdf}}
    \end{minipage}
    \begin{minipage}{0.24\textwidth}
    \resizebox{\textwidth}{!}{\includegraphics{figures/temperature_properties/cifar10/rank_cifar10.pdf}}
    \end{minipage}
    \begin{minipage}{0.24\textwidth}
    \resizebox{\textwidth}{!}{\includegraphics{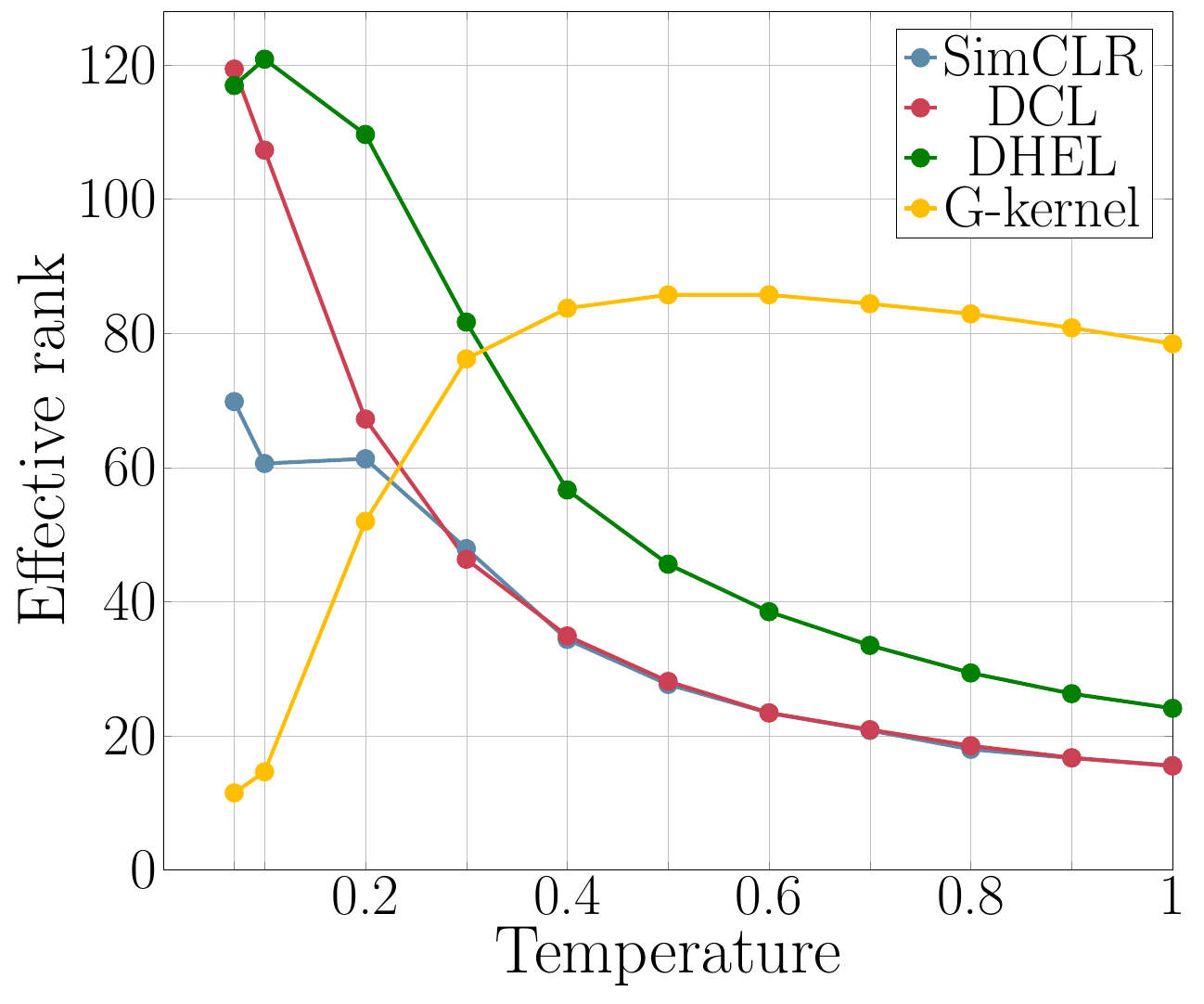}} 
    \end{minipage}
    \begin{minipage}{0.24\textwidth}
    \resizebox{\textwidth}{!}{\includegraphics{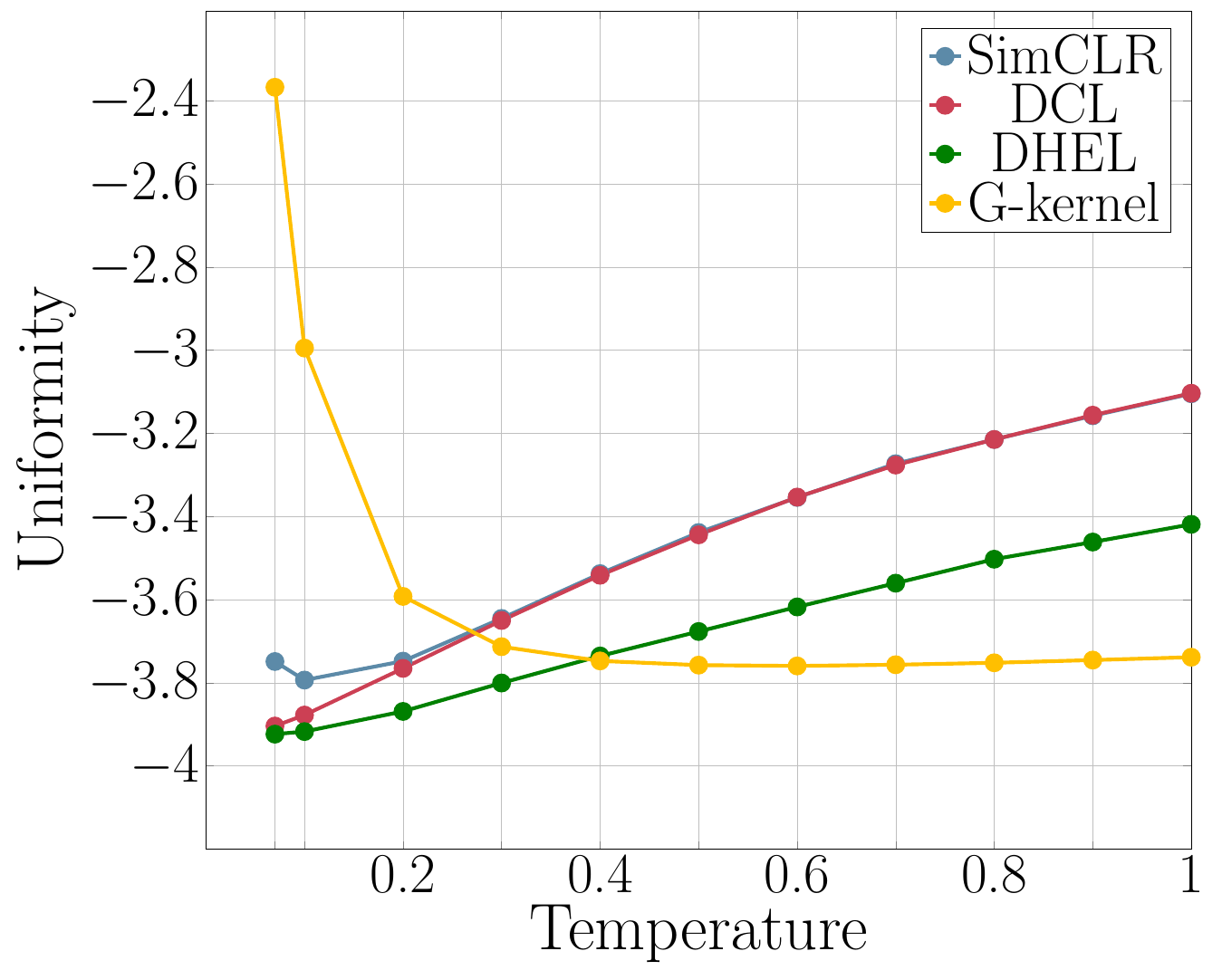}}
    \caption*{(a) Uniformity}
    \end{minipage}
    \begin{minipage}{0.24\textwidth}
    \resizebox{\textwidth}{!}{\includegraphics{figures/temperature_properties/cifar100/wasserstein_cifar100.pdf}}
    \caption*{(b) Wasserstein distance}
    \end{minipage}
    \begin{minipage}{0.24\textwidth}
    \resizebox{\textwidth}{!}{\includegraphics{figures/temperature_properties/cifar100/rank_cifar100.pdf}}
    \caption*{(c) Rank}
    \end{minipage}
    \begin{minipage}{0.24\textwidth}
    \resizebox{\textwidth}{!}{\includegraphics{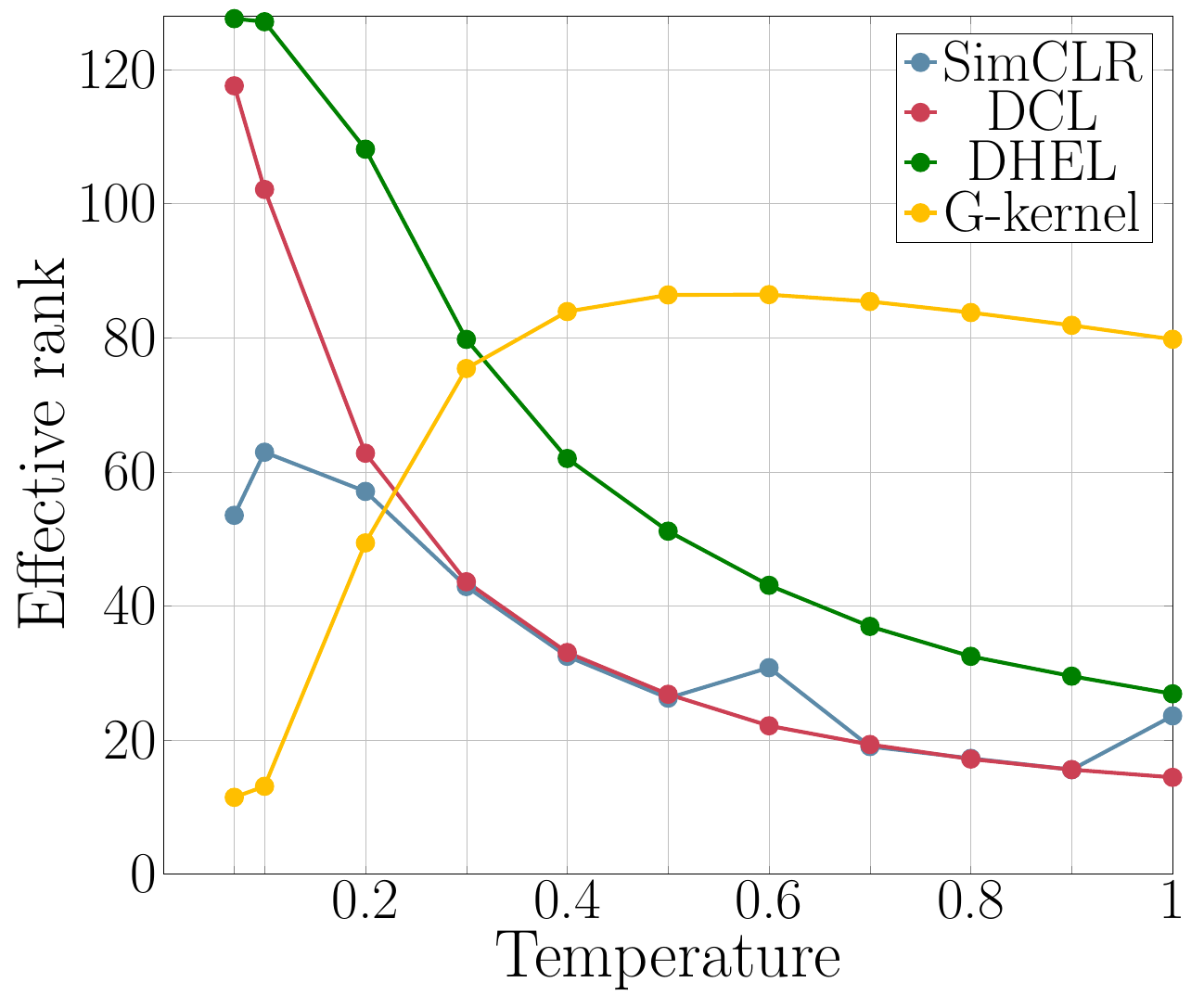}}
    \caption*{(d) Effective rank}
    \end{minipage}
    \caption{Comparison of two uniformity and rank metrics calculated on CIFAR-10 (top) \& CIFAR-100 (bottom) dataset}\label{fig:metrics_comparison}
\end{figure*}

\subsection{Implementation Details}\label{sec:impl_details}
\paragraph{Code} The implementation of the experimental pipeline (networks, augmentations, training, evaluation functions etc) were based on \url{https://github.com/AndrewAtanov/simclr-pytorch.git}, while our implementation of the proposed loss functions and metrics can be found at \url{https://github.com/pakoromilas/DHEL-KCL.git}
\paragraph{CIFAR10, CIFAR100 and STL10} ResNet-18 is employed as the encoder architecture for CIFAR10, CIFAR100, and STL10 datasets. Training spans 200 epochs with the SGD optimizer and the cosine annealing learning rate schedule, using a base learning rate of (batch size) / 256. It's worth mentioning that STL10 includes both the train and unlabeled sets for pre-training the model. Augmentations include resizing, cropping, horizontal flipping, color jittering, and random grayscale conversion.  Linear evaluation is conducted by training a single linear layer on the learned embeddings, with an additional 200 epochs using SGD and a learning rate of 0.1.

\paragraph{ImageNet-100} ResNet-50 is employed as the encoder architecture for ImageNet-100. Training spans 200 epochs with the SGD optimizer and the cosine annealing learning rate schedule, using a base learning rate of 1.4 * (batch size) / 256. We use the same augmentations as in the above datasets and extend them to include gaussian blur. Linear evaluation is conducted by training a single linear layer on the learned embeddings, with an additional 200 epochs using SGD and a learning rate of 0.5.

\paragraph{Hyperparameters}
For the InfoNCE methods we run experiments for temperatures $[0.07, 0.1, 0.2, 0.3, 0.4, 0.5, 0.6, 0.7, 0.8, 0.9, 1]$. For the G-kernel we use the same temperature parameters along  $w = [8, 16, 32]$. For the Log-kernel we use $s = [1, 1.5, 2, 3]$ and  $w = [16, 32, 64]$.

\subsection{Further Results}\label{sec:further_res}

In this section we present further experimental results on (i) the optimisation capabilities of the examined loss functions, as well as (ii) the comparison between different measures of uniformity and rank.

\paragraph{Optimisation capabilities} As demonstrated in \Cref{sec:expectation_proofs}, CL objectives are optimised for both perfect alignment and uniformity. In \Cref{fig:properties_batch} we compare the mean value of such properties for all the examined methods across different batch sizes for the CIFAR10 and CIFAR100 datasets. The methods that decouple uniformity from alignment (DHEL and KCL) achieve superior optimization of uniformity, although they are inferior to baseline methods in alignment optimization. This shortfall may not be problematic, as recent studies \cite{gupta2023structuring, xie2022should} suggest that perfect alignment might not be ideal for downstream performance, since many downstream tasks may not be invariant to the augmentations used to generate positive samples, implying that perfect alignment could be less critical in these scenarios.

Despite alignment and Wasserstein distance observed values being on different scales, both metrics have the same range [0, 1] and monotonicity.Therefore, we examine the balance of the overall metric (alignment + uniformity) to gain insights into how these properties interact. Methods that optimize this combination tend to perform better in downstream tasks. For example, the Log-kernel method performs well for both CIFAR10 and CIFAR100, and the DHEL method excels in CIFAR100. However, the G-kernel achieves the second-best performance in CIFAR10 despite having the poorest optimization. This discrepancy can be attributed to the possibly undesired optimal alignment \cite{gupta2023structuring, xie2022should}, and the tolerance-uniformity dilemma \cite{wang2021understanding}. The pretraining stage can be benefited by designing a proper weighting function that correlates these two properties to downstream task performance.

\paragraph{Uniformity metric vs Wasserstein distance}

In \Cref{fig:metrics_comparison}, we elucidate the difference between the conventional uniformity metric and the Wasserstein distance introduced in our study. Although both metrics are generally consistent, discernible differences emerge at temperatures 0.07, 0.1, and 0.2. The Wasserstein distance more effectively highlights the superior uniformity of DCL and the compromised uniformity of SimCLR. This distinction is also observed in downstream task performance. Despite DCL and SimCLR having the same alignment and uniformity for all temperatures, DCL achieves superior performance at these specific temperatures ([0.07, 0.1, 0.2]) and matches SimCLR's performance at other temperatures, as shown in \Cref{fig:properties_temperature}
. Our metric captures the actual uniformity difference, which reflects to different downstream task performance. Notably, this behavior is observed at optimal uniformity levels, underscoring the discerning power of the introduced metric.

\paragraph{Rank vs Effective rank}
In \Cref{fig:metrics_comparison} we can see that both the rank and the effective rank metrics demonstrate the same trend. In most cases, the rank does correlate with downstream performance as \cite{garrido2023rankme} mention, but note that for small temperatures, this correlation seems to die out, which probably comes from the fact that alignment is quite poor, since the latter is not captured by either rank metrics.


\end{document}